\documentclass{article} % For LaTeX2e
\usepackage{iclr2026_conference,times}
\iclrfinalcopy
\usepackage{amsmath,amsfonts,bm}

\def\eqref#1{equation~\ref{#1}}
\def\1{\bm{1}}

\DeclareMathAlphabet{\mathsfit}{\encodingdefault}{\sfdefault}{m}{sl}
\SetMathAlphabet{\mathsfit}{bold}{\encodingdefault}{\sfdefault}{bx}{n}

\newcommand{\E}{\mathbb{E}}

\usepackage{hyperref}
\usepackage{url}
\usepackage{marvosym}
\usepackage[utf8]{inputenc} % allow utf-8 input
\usepackage[T1]{fontenc}    % use 8-bit T1 fonts
\usepackage{url}            % simple URL typesetting
\usepackage{booktabs}       % professional-quality tables
\usepackage{amsfonts}       % blackboard math symbols
\usepackage{nicefrac}       % compact symbols for 1/2, etc.
\usepackage{microtype}      % microtypography
\usepackage{mathtools}      % for \coloneqq
\usepackage{xcolor}         % colors
\usepackage{multirow}
\usepackage{adjustbox}
\usepackage{algorithm}
\usepackage{algorithmic}
\usepackage{amsmath}
\usepackage{amsthm}
\usepackage{hyperref}       % hyperlinks
\usepackage{enumitem}
\usepackage{amsmath}
\usepackage{graphicx}
\usepackage{xspace}
\usepackage{subcaption} 
\usepackage{listings}
\usepackage{amsmath, amssymb, amsthm}
\usepackage{amsfonts} % For \mathbb
\usepackage{bm} % For bold math symbols like \bm{\theta}
\theoremstyle{definition}

\newtheorem{theorem}{Theorem}
\newtheorem{assumption}{Assumption}
\usepackage{amsthm}
\usepackage{mathtools}
\newcommand{\defeq}{\coloneqq}
\newtheorem{lemma}{Lemma}
\newtheorem{remark}{Remark}
\newtheorem{corollary}{Corollary}
\newcommand{\Loss}{\mathcal{L}}
\usepackage{authblk}
\usepackage{marvosym} % 提供 \Letter 符号
\usepackage{fancyhdr}  % 用于自定义页眉页脚
\usepackage{lastpage}
\newcommand{\ie}{\textit{i.e.}}
\newcommand{\eg}{\textit{e.g.}}

\newcommand{\partitle}[1]{\medskip \noindent \textbf{#1.}}
\newcommand{\Name}{\texttt{XTF}\xspace}

\title{Explainable Token-level Noise Filtering for LLM Fine-tuning Datasets}

\author[*,1,2]{\textbf{Yuchen Yang}}
\author[*,3]{\textbf{Wenze Lin}}
\author[1,2]{\textbf{Enhao Huang}}
\author[\Letter,1,2]{\textbf{Zhixuan Chu}} 
\author[4]{\textbf{Hongbin Zhou}}
\author[4]{\textbf{Lan Tao}}
\author[5]{\textbf{Yiming Li}}
\author[1,2]{\textbf{Zhan Qin}}
\author[1,2]{\textbf{Kui Ren}}
\vspace{0.3em}

\affil[1]{The State Key Laboratory of Blockchain and Data Security, Zhejiang University}
\affil[2]{Hangzhou HighTech Zone (Binjiang) Blockchain and Data Security Research Institute, China}

\affil[ ]{\textsuperscript{3}Qiuzhen College, Tsinghua University \quad \textsuperscript{4}Alibaba Group \quad \textsuperscript{5}Nanyang Technological University}

\date{\vspace{-5ex}} % 隐藏日期
\date{\vspace{-5ex} % 隐藏日期并缩减上方间距
    \vbox{\vspace{2ex}
    \small \textsuperscript{*}Equal contribution \quad \textsuperscript{\Letter}Corresponding author: zhixuanchu@zju.edu.cn}
}
\iclrfinalcopy % Uncomment for camera-ready version, but NOT for submission.

\fancypagestyle{plain}{\pagestyle{fancy}}
\begin{document}

% \author{
%   Yuchen Yang\textsuperscript{1}, 
%   Wenze Lin\textsuperscript{2}, 
%   Enhao Huang\textsuperscript{1}, 
%   Zhixuan Chu\textsuperscript{1, $\dagger$}, % 注意这里加了 $ 符号
%   Hongbin Zhou\textsuperscript{3}, 
%   Lan Tao\textsuperscript{3}, 
%   Yiming Li\textsuperscript{1}, 
%   Zhan Qin\textsuperscript{1}, 
%   Kui Ren\textsuperscript{1} \\
%   \\
%   \textsuperscript{1}Zhejiang University \quad
%   \textsuperscript{2}Tsinghua University \quad
%   \textsuperscript{3}Alibaba Group \\
%   \\
%   \texttt{\textsuperscript{$\dagger$}zhixuanchu@zju.edu.cn} % 这里也加了 $ 符号
% }

% 临时关闭脚注超链接
% \blfootnote{%
% \begin{NoHyper}
% \textsuperscript{*}Equal contribution\quad 
% \textsuperscript{\Letter} Corresponding Author  \quad 
% \textit{Correspond to: \url{zhixuanchu@zju.edu.cn}}
% \end{NoHyper}
% }

\maketitle
\vspace{-35pt}
\begin{center}
\small % 小字号更美观
\textsuperscript{*} Equal contribution \quad \textsuperscript{\Letter} Corresponding to: \texttt{zhixuanchu@zju.edu.cn}
\end{center}
\vspace{20pt}
\begin{abstract}
Large Language Models (LLMs) have seen remarkable advancements, achieving state-of-the-art results in diverse applications. Fine-tuning, an important step for adapting LLMs to specific downstream tasks, typically involves further training on corresponding datasets. However, a fundamental discrepancy exists between current fine-tuning datasets and the token-level optimization mechanism of LLMs: most datasets are designed at the sentence-level, which introduces token-level noise, causing negative influence to final performance. In this paper, we propose \Name, an \emph{explainable token-level noise filtering} framework. \Name decomposes the complex and subtle contributions of token-level data to the fine-tuning process into three distinct and explicit attributes (\emph{reasoning importance}, \emph{knowledge novelty}, and \emph{task relevance}), which can be assessed using scoring methods, and then masks the gradients of selected noisy tokens accordingly to optimize the performance of fine-tuned LLMs. We conduct extensive experiments on three representative downstream tasks (math, code and medicine) across 7 mainstream LLMs. The results demonstrate that \Name can significantly improve downstream performance by up to 13.7\% compared to regular fine-tuning. Our work highlights the importance of token-level dataset optimization, and demonstrates the potential of strategies based on attribute decomposition for explaining complex training mechanisms.
\end{abstract}

\section{Introduction}
\label{sec:introduction}
% %大模型技术在近期发展迅速，具有强大推理能力，能够广泛地应用于各种下游任务。同时，为了让大模型在具体的下游任务中具有更好的表现，开发者们会使用相应的数据集对模型进行训练，小幅度的训练改变模型参数。这种技术被称为微调，是现实应用中非常常用的技术。
LLM technology \citep{an2024makeLLM,nam2024usingLLM,kambhampati2024positionLLM} has developed rapidly in recent years, possessing powerful reasoning capabilities and enabling widespread application across various downstream tasks \citep{thirunavukarasu2023ai4medicine,wen2024ai4edu,guo2024aiagent}. Meanwhile, to enhance the performance of LLMs in specific downstream tasks, developers usually train the base models, \ie, general purpose LLM such as Llama \citep{touvron2023llama} and Deepseek \citep{guo2025deepseek}, using relevant datasets, making adjustments to the model parameters. This technique, known as fine-tuning, is widely used in practical applications \citep{wang2022no-code-llm-finetuning,lin2024data-efficient-fine-tuning-recommendation}.

% %然而，现有微调数据集无法完全匹配LLM微调优化过程。与LLM逐个token计算loss值并进行参数优化不同，现有大部分微调数据集都是基于完整的sentence去设计目标输出。这样的不同会导致训练收敛的方向出现偏差。具体来说，输出label中部分token的预测与提升模型的微调效果并没有关系，但是同样在微调过程中被重视，成为优化的噪音。
However, existing fine-tuning datasets do not align fully with the token-by-token optimization process of LLMs. While LLM fine-tuning involves computing a loss at each token and updating model parameters accordingly, most fine-tuning datasets are designed at the sentence level, providing label sentences as the target output. Since not all tokens (in the label sentence) are valuable for performance improvement \citep{lin2024not-all-tokens,peng2023token-self-evolution}, training in the entire label sentence can possibly introduce token-level noise and misguide the direction of convergence, ultimately reducing performance of fine-tuned LLMs in the target downstream task. 

% %现有研究无法针对LLM微调任务对数据集进行token-level的优化。主流数据优化方法可分为两类：一类是数据过滤；一类是数据扩充。这两类方法都是在样本级别上进行的，无法去除token_level的噪声。一些现有的工作分别从预训练、人类偏好优化、知识蒸馏等角度探究token-level数据与sentence-level数据的差异。然而，这些工作受限于特殊的应用场景（如预训练过程、DPO框架）或不探究token之间的价值差异（信息升维、知识蒸馏），无法适配微调数据集优化任务。
Current research lacks the capability to optimize datasets at the token level for LLM fine-tuning tasks. Mainstream data optimization methods fall into two categories: data filtering \citep{li2024superfiltering,goyal2024scalingDF} and data augmentation \citep{dai2025auggpt,ding2024dataDA}. All of these approaches operate at the sample level and thus fail to further eliminate token-level noise. Some existing studies have explored the differences between token-level and sentence-level data from various perspectives, such as pretraining \citep{lin2024not-all-tokens}, human preference optimization \citep{zeng2024tokenTDPO,yoon2024tlcrtoken-level}, and knowledge distillation \citep{wei2024sentence-level-or,cui2025multitoken-level-knowledge}. However, these works are often limited to specific scenarios (\eg, pretraining with lower data quality requirements or direct preference optimization (DPO) \citep{rafailov2023directDPO} that relies on prior knowledge of labeled text pairs) or do not sufficiently investigate the value differences between tokens (\eg, as in some knowledge distillation approaches \citep{peng2023token-self-evolution,cui2025multitoken-level-knowledge}), rendering them unsuitable for fine-tuning dataset optimization.

%想要实现token-level的数据集优化，就需要过滤掉对微调效果没有帮助的tokens，which is a non-trivial task。首先，并没有任何研究揭示单个token和微调效果之间的关系。尽管某些可解释性研究能够指出推理过程中token in input content与生成正确输出ground truth之间的关系，它们无法应用于微调过程也无法解释label sentence中的token价值。其次，微调任务的效果取决于预训练模型已有的知识和目标任务。因此，在筛选有帮助token时，必须同时考虑预训练模型自身对数据的主观认知和数据与任务之间的客观相关性。这种差异要求数据集优化方法能够综合地考虑微调任务的需求，而不是仅局限于单一的评估标准。

Achieving token-level dataset optimization for fine-tuning necessitates filtering out tokens in output labels that do not contribute to final performance, which is a non-trivial task. Firstly, no existing research clearly elucidates the relationship between individual tokens within these labels and fine-tuning effectiveness. Although some explainability studies can identify connections between tokens in the input content and the correct generation of label sentences during the reasoning process \citep{chu2024causalXAI,zhao2024explainabilityXAIsurvey}, they cannot explain the value of specific tokens in label sentences for fine-tuning tasks. Secondly, fine-tuning performance depends on both the base model's pre-existing knowledge and the specifics of the target task. When filtering noisy tokens, it is essential to consider both the base model's understanding of the data and that data's relevance to the downstream task \citep{liu2022few,han2024peftsurvey}. Therefore, filtering noisy tokens from fine-tuning datasets requires a comprehensive consideration of fine-tuning task requirements, rather than reliance on a single assessment criterion.

% %motivated by discussion above，我们提出了一种评估token-level数据在相应任务下的微调价值，并筛选其中无用噪声数据的方法（\Name）。我们的方法可以分为三个phases。第一个phase中，我们提出了三个对于微调性能必要的foctors：推理重要性、认知偏移性和任务相关性，并解释其原理。第二个phase中，我们针对提出的三个factors，基于预训练模型和微调任务内容设计了token-level数据价值评估方案。具体来说，对于1）推理重要性：我们拼接input和label output并计算attention score，越高的attention score代表了越高的推理重要性；2）认知偏移性：我们用PPL的值来度量LLM在微调数据集上的认知偏移，越高的PPL代表了越高的认知偏移性；3）任务相关性：我们基于无上下文输入在base model上的嵌入层特征来近似评估任务的相关性。我们通过几个专家词的平均向量表示来近似任务domain，然后计算数据集中的token向量与该平均向量的距离。之后，基于距离数值的统计学分布过滤掉无关的tokens。第三个phase中，我们利用三种factors采用不同的策略来筛选对微调数据集有帮助的tokens，并将其余tokens在label中屏蔽。
Motivated by the preceding discussion, we propose an \textbf{e}xplainable \textbf{t}oken-level data \textbf{f}iltering method \Name. This method aims to assess the value of token-level data within fine-tuning datasets and filter out noisy tokens by considering the specific characteristics of LLM fine-tuning. \Name consists of three phases. In the first phase, we decompose the contribution of data to the fine-tuning effect into three attributes (reasoning importance, knowledge novelty, and task relevance) to reduce the complexity of token value analysis. Concurrently, we define the criteria for identifying noisy tokens based on these attributes and provide theoretical justification in Appendix~\ref{sec:theory}. In the second phase, while considering the two key factors essential to fine-tuning: base model and task dataset, we design scoring mechanisms for the three attributes with controllable computational costs: \textbf{1) Reasoning Importance:} We concatenate the input and label output, then compute attention scores for each token using the base model. A lower attention score indicates a lower reasoning importance. \textbf{2) Knowledge Novelty:} We introduce the \textbf{p}robability of \textbf{c}orrect token \textbf{p}rediction (PCP) to quantify the novelty of knowledge learned from the fine-tuning dataset. A higher PCP indicates lower knowledge novelty. \textbf{3) Task Relevance:} We assess task relevance using embedding vectors generated by the base model for context-free inputs. The task domain is approximated by the average embedding of data samples, and token relevance scores are determined by their distance from the domain center. A larger distance implies lower task relevance. We present detailed scoring procedures in Section~\ref{sec:how to assess tokens}. In the third phase, we identify noisy tokens based on the statistical results and mask the gradients associated with these tokens during training to enhance the performance of the fine-tuned LLM. We adopt a conservative strategy to ensure the filtered tokens align with the criteria established in the first phase.

\begin{figure}[t]
  \begin{subfigure}{0.48\linewidth}
    \includegraphics[width=\linewidth]{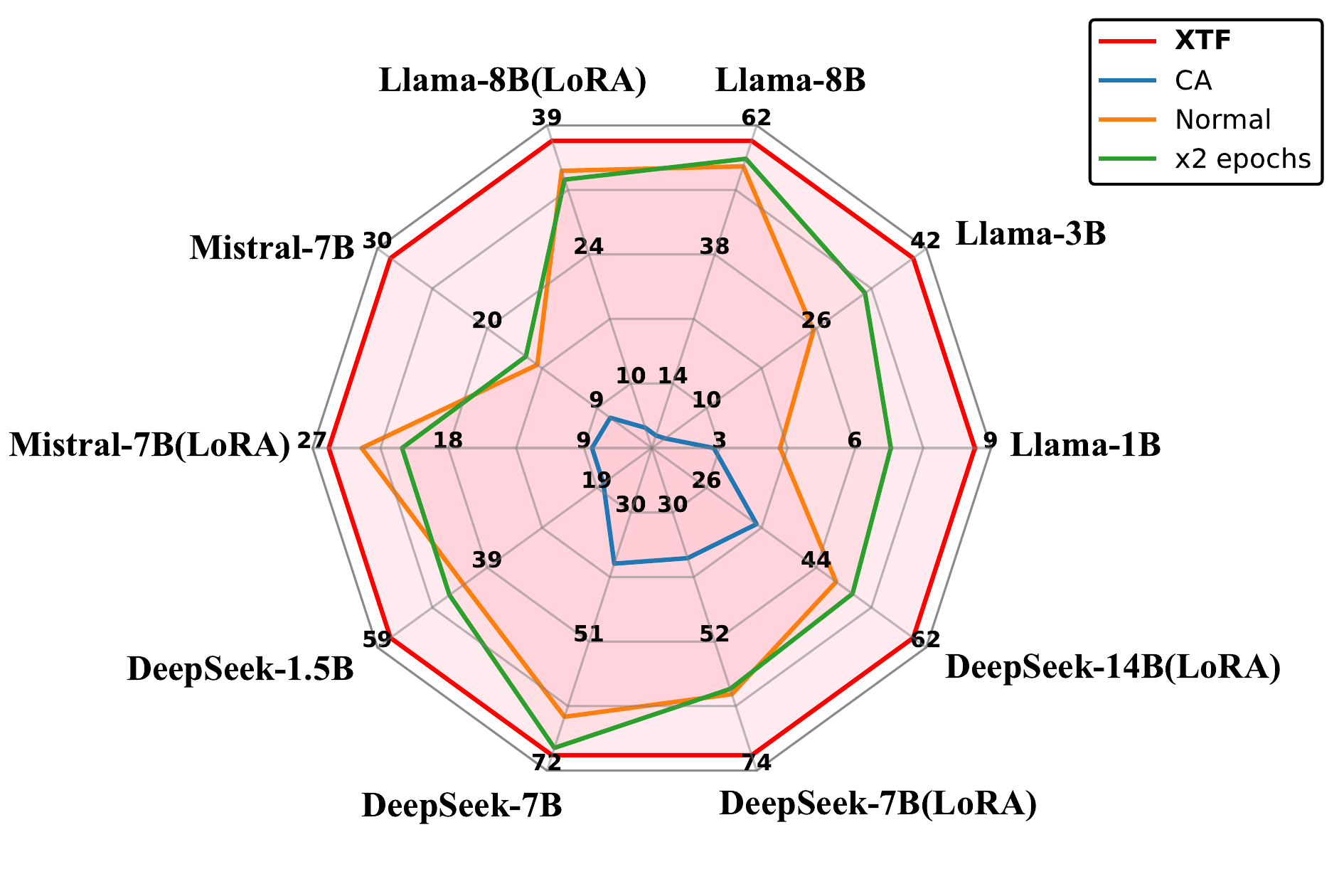}
    \caption{Math task}
  \end{subfigure}
  \hfill
  \begin{subfigure}{0.48\linewidth}
    \includegraphics[width=\linewidth]{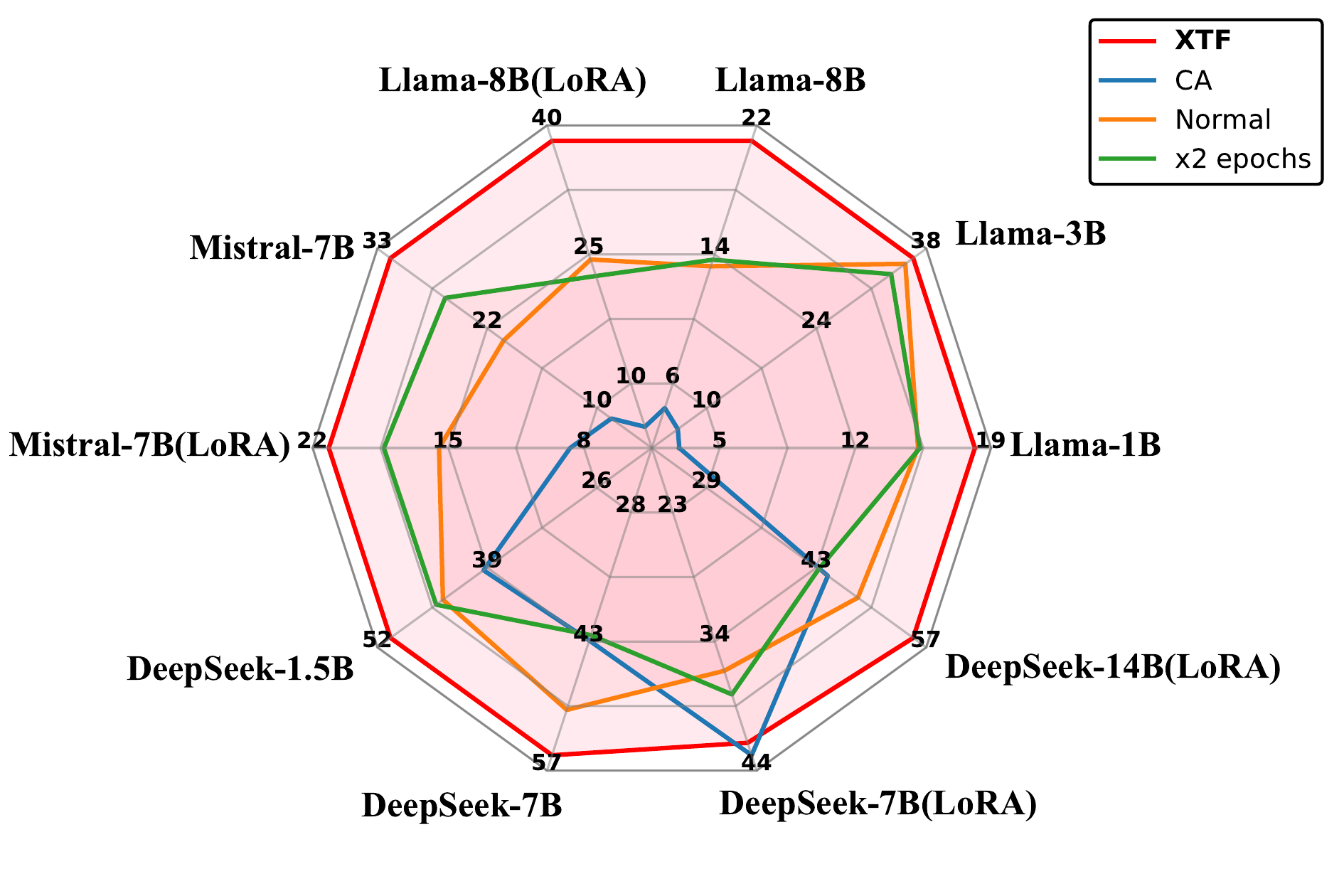}
    \caption{Medicine task}
  \end{subfigure}
  \caption{The accuracy performance of our method on different LLMs. The results show that our method can significantly improve the final performance of fine-tuned LLMs across almost every case.}
  \label{fig:fine-tuning optimization on different LLMs.}
\end{figure}

% %我们方法首次在token-level级别上探究了训练数据对微调效果的影响。我们从推理重要性、认知偏移性和任务相关性三个维度探究了过滤数据中token-level噪声的方案，并提出了\Name。我们在x种数据集以及y种大模型上进行了广泛的实验，证明\Name能够有效提高微调的效果（最高达10%的精度提升）。在code generation生成中，我们的方法分别在pass@1，pass@5和pass@10上取得了最高5.6\%，5.6%和6.3%的性能优化，证明\Name在涉及多轮生成的复杂任务中同样有效。
We conduct extensive experiments on datasets across 3 different downstream tasks and 7 representative LLMs. As shown in Figure~\ref{fig:fine-tuning optimization on different LLMs.}, \Name can achieve up to 13.3\% and 13.7\% accuracy optimization on math task and medicine task, respectively, demonstrating its effectiveness on noise filtering and fine-tuning enhancement. In code generation, our method achieves performance optimization of up to 5.6\%, 5.6\%, and 6.3\% on pass@1, pass@5, and pass@10, respectively, demonstrating that \Name is also effective in complex tasks involving multi-chance generation.

Our main contributions are three-fold:
\begin{itemize}[leftmargin=*,noitemsep,topsep=0pt,parsep=0pt,partopsep=0pt]

\item We reveal the research gap of token-level data optimization for LLM fine-tuning. 

\item We explore solutions for filtering token-level noise in fine-tuning datasets via three decomposed attributes: \emph{reasoning importance}, \emph{knowledge novelty}, and \emph{task relevance}, proposing \Name.

\item We conduct extensive experiments of \Name across multiple representative LLMs and downstream tasks, verifying its superior performance on fine-tuning optimization, demonstrates the potential of strategies based on attribute decomposition for explaining complex training mechanisms.
\end{itemize}

\section{Background and Related Work}
\label{sec:background and related work}

\subsection{Large Language Models Fine-tuning}
%大模型微调是被广泛用于提升大模型下游任务性能的重要技术。存在许多基于高性能base model（如：llama, deepseek）微调的task-oriented model，例如llama-math \citep{}、llama-finance \citep{}。微调数据的质量是决定微调效果好坏的重要因素\red{citep}，现有研究表明少量高质量数据的重要性要远高于大量低质量数据\red{citep}。因此，提高数据质量对于增强微调效果非常重要。
Fine-tuning large models is widely recognized as a key technique to enhance their performance on downstream tasks \citep{wang2022no-code-llm-finetuning,han2024peftsurvey,lin2024data-efficient-fine-tuning-recommendation}. Numerous task-oriented models have been developed by fine-tuning high-performance base models, such as Llama-Math \citep{llama-math} and Llama-Finance \citep{llama-finance}. The quality of the fine-tuning dataset is a critical factor that determines the effectiveness of this process \citep{zhou2023lima,kuramoto2025predicting-finetune-performance}. Therefore, optimizing data is essential for enhancing fine-tuning outcomes.

\subsection{Explainability for LLMs}
%可解释性方法试图解释人工智能模型的决策机制和行为逻辑。现有大模型领域的可解释方法可以分为局部可解释\red{citep}，全局可解释\red{citep}以及xxxx几类。然而，现有的可解释技术关注推理过程，缺少解释数据集label输出与微调训练效果的研究。
Explainability methods for LLMs aim to explain the decision-making mechanisms and behavioral logic of these models \citep{zhao2024explainabilityXAIsurvey,chu2024causalXAI}. However, current techniques predominantly focus on the reasoning process \citep{huang2023canXAI0,kang2024quantitativeXAI1,singh2024rethinkingXAI2}, while research of explaining the relationship between tokens and fine-tuning outcomes remains limited.

\subsection{Token-level Training}
%已经有一些先驱工作研究了精细到token-level数据的训练。这些研究可分为三个领域：数据蒸馏学习、模型直接训练与人类偏好训练。数据蒸馏学习关注学生模型学习教师模型的token-level输出（\ie,logits分布）与sentence-level输出时的性能差异，但不关注作为训练数据的tokens之间的差异；模型直接训练的研究,包括较小transformer模型训练与大模型预训练,利用损失函数来指导有价值token的选择。我们发现这些方法有潜在的假设：存在高质量数据集能够正确引导token选择。这样的假设在微调任务中是难以成立的，细分垂域（\eg, 医疗、教育）往往缺乏数量充足的高质量训练数据。除此以外，没有任何研究表明GSM8K、Humaneval等高质量数据集上不存在token-level噪声；人类偏好训练的研究关注了DPO这一优化框架并构建了一套基于token-level的特殊奖励模型，然而该方法很难迁移到更广泛的应用场景中。
Several pioneering works have explored token-level data refinement in training, broadly categorized into three domains: data distillation, direct model training, and human preference training. \textbf{Data distillation} approaches \citep{wei2024sentence-level-or,cui2025multitoken-level-knowledge,liu2022multiknowledge} primarily focus on the performance differences observed when a student model learns from token-level outputs (i.e., logits) versus sentence-level outputs, rather than specifically differentiating the value among individual tokens. \textbf{Direct model training}, including studies on training small-scale transformer models \citep{peng2023token-self-evolution} and LLM pretraining \citep{lin2024not-all-tokens}, utilizes changes in loss values to guide the selection of valuable tokens. We observe that these methods operate under an implicit assumption: that high-quality datasets are entirely noise-free and thus capable of correctly guiding token selection. This assumption rarely holds true in fine-tuning tasks, as base models often already exhibit strong performance, making it challenging to optimize dataset quality sufficiently to satisfy this noise-free criterion. Furthermore, existing works \citep{peng2023token-self-evolution,lin2024not-all-tokens} do not demonstrate that high-quality datasets are, in fact, approximately devoid of noisy tokens, which is a critical premise for these methods. \textbf{Human preference training} often involves specific optimization frameworks that rely on prior knowledge of labeled text pairs \citep{zeng2024tokenTDPO,xia2024inversetoken-level-prefer,yoon2024tlcrtoken-level} and the construction of specialized token-level reward models; however, this approach can be challenging to generalize to broader application scenarios.

\begin{figure*}[t]
    \centering % 
    \includegraphics[width=\textwidth]{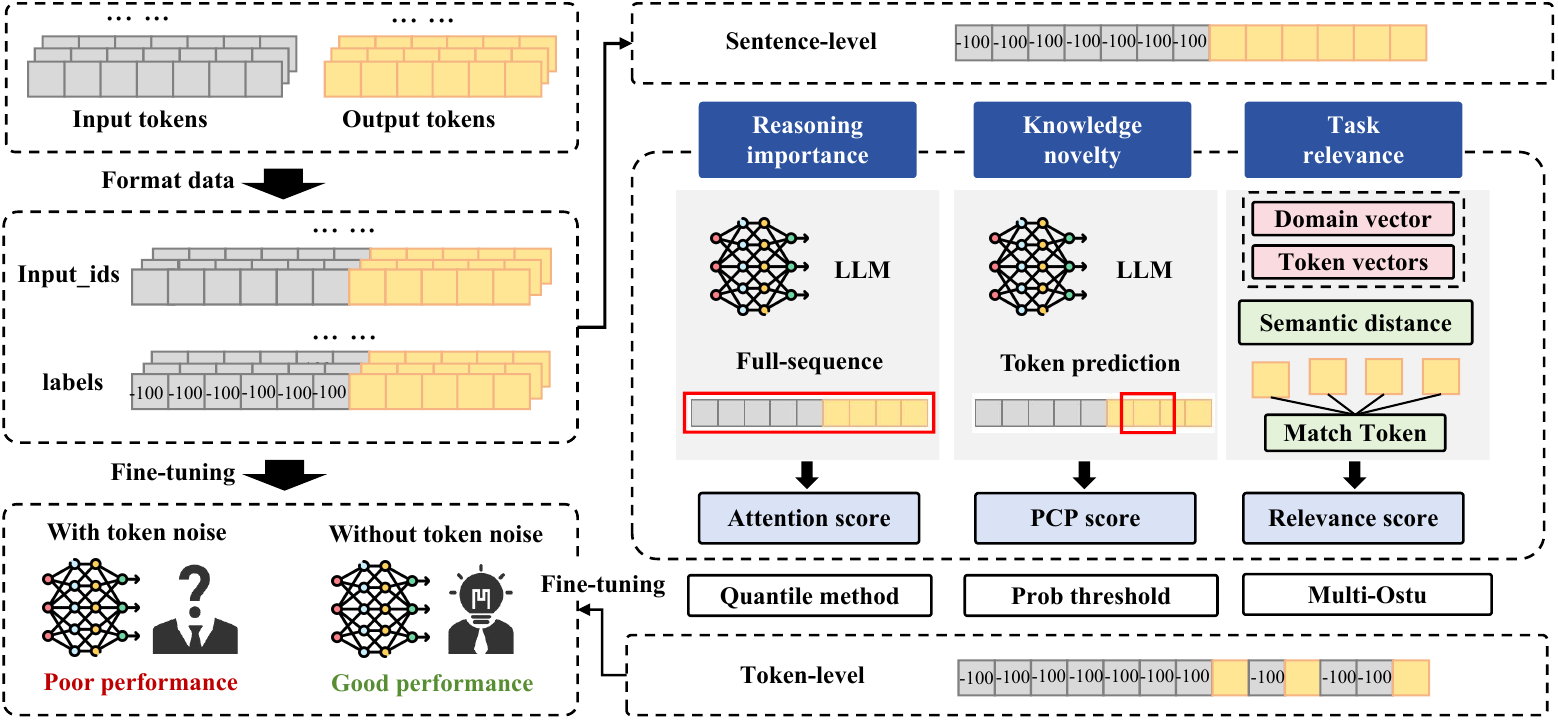} % 
    \caption{The pipeline of token-level data filtering comprises three steps. In the first step, we preprocess the dataset based on a regular format function. In the second phase, we get the sentence-level data item and assess three scores, \ie, attention score, PCP score and relevance score, for the tokens of output label, suggesting the selection of noisy tokens. In the third phase, we mask the noisy tokens and fine-tune the target LLM.} % 
    \label{fig:pipeline} % 
    \vspace{-1em}
\end{figure*}

\section{Methodology}
\label{sec:methodology}
%这一节我们将介绍\Name方法的理论基础以及实现方法。
%In this section we will introduce the theoretical basis and implementation method of our \Name.

\subsection{Which Data Acts as Noise for Fine-tuning?}
\label{sec:which data act as noise for fine-tuning?}
%In this phase, 我们从微调任务本身特性出发，来定义噪声信息的特点。微调过程可以看作高性能基座模型和任务数据集的对齐。因此，我们需要让基座模型在保留推理能力的基础上，学习数据集提供的差异化信息。同时，由数据集应当尽可能避免引入和任务无关的信息，干扰训练过程。基于上述基本逻辑，我们定义了三个评估数据集中信息价值的维度。
%--推理重要性：评估信息对于基座模型而言，在构成上下文逻辑上是否重要。这些信息对于保留基座模型推理能力具有重要价值。
%--知识偏移性：评估信息对于基座模型而言，是否包含新的知识。已有的常识性知识不必要重复学习。
%--任务相关性：评估信息对于任务目标而言，是否相关。相关性差的信息对于微调任务而言价值较低。
%对微调任务有价值的信息应当同时满足以上三种属性。例如，如果信息具备较高的推理重要性和认知偏移性，但不具备任务相关性，那么该信息显然对于微调是无益的。可以说，不具备以上三种属性任意一条的信息，都是微调任务中的噪声。其形式化表示如下：Inf_{noise} = （Inf_(RI↓)）并集（Inf_(KN↓)）并集（Inf_(TR↓)）,其中Inf_(RI↓),（Inf_(KN↓)）和（Inf_(TR↓)）分别代表不具备推理重要性、知识偏移性和任务相关性的信息。我们在附录中通过形式化定义和数学推导证明了在大模型微调过程中利用这三条性质筛选噪声信息的逻辑。
Due to scale differences, it is inherently difficult to intuitively assess the impact of token-level data on the final fine-tuning outcome. Therefore, we attempt to explain the contribution of token data to fine-tuning from three attributions through theoretical analysis. 

A fine-tuning process can be conceptualized as an alignment between a high-performance base model and a task-specific dataset. Consequently, the performance of the fine-tuned model should be influenced by three factors: the cognition of the base model, the knowledge in the task dataset, and the contradiction between the base model and the task dataset. When we aim to mask a token from the label sentence, we can assess its potential impact on the fine-tuning result from these three perspectives. Specifically, we propose three attributes that positively influence the fine-tuning process: for the cognition of the base model, we extract \textbf{reasoning importance}; for the discrepancy between the base model and the task dataset, we extract \textbf{knowledge novelty}; and for the knowledge in the task dataset, we extract \textbf{task relevance}. These attributes represent the following meanings:

\vspace{0.2em}
\textbf{Reasoning Importance (RI)}: whether the presence or absence of this token significantly affects the base model's inference results;

\vspace{0.2em}
\textbf{Knowledge Novelty (KN)}: whether the presence of this token is novel to the base model;

\vspace{0.2em}
\textbf{Task Relevance (TR)}: whether the presence of this token is related to the objective of the task dataset.

It is not feasible to consider all three properties simultaneously and assign a composite score, as there is no clear basis to determine their interrelationship or any hierarchical order. However, we can still use these three properties to identify which tokens are noise. Specifically, we find that if a token completely lacks any of the three attributes, it can be considered as noise. This is intuitive; for example, if a token is entirely unrelated to the task objective (lacking the TR attribute), then even if it may influence the base model's subsequent generation results (RI) or represent a prediction that the base model has not yet learned (KN), it does not contribute to the fine-tuning task. We elaborate on the analysis of the three derived attributes and formally prove the correctness of this judgment through logical reasoning in the Appendix~\ref{sec:theory}.

Formally, the token-level noise in the fine-tuning tasks can be represented as:
\begin{equation}
D_{\text{noise}} = (D_{RI\downarrow}) \cup (D_{KN\downarrow}) \cup (D_{TR}\downarrow)
\end{equation}
where $(D_{RI\downarrow})$, $(D_{KN\downarrow})$ and $(D_{TR\downarrow})$ respectively represent data lacking reasoning importance, knowledge novelty, and task relevance.

\subsection{How to assess Tokens?}
\label{sec:how to assess tokens}
%In this phase, 我们基于对信息价值的三个定义，提出了一套可解释的token评分体系。之后，我们基于三种评分的分布特征来确定其阈值并判断噪声tokens。
After obtaining the three attributes and defining noise data, we employ a parameterized scoring mechanism to separately assess the three attributes and to identify noise tokens in the dataset. In alignment with real-world requirements, the choice of scoring method is expected to meet two requirements: \textbf{1) Controllable computational cost:} Since the dataset contains a vast number of tokens, the assessment method should not impose excessive computational overhead. \textbf{2) Joint consideration of the base model and the data}: The assessment of the three attributes must not be separated from the essential elements of the fine-tuning task itself. Consequently, we adopt three reasoning-level explainability methods to assess the three attributes separately.

\partitle{Attention Score for RI}
%我们使用三种可解释的方法来对tokens进行评分。具体来说，我们使用基座模型的attention score来评估推理重要性。我们将整个数据（包括input和output）输入模型，并计算注意力分数。这种方案考虑了基座模型在整个答案生成逻辑中对数据的认知\red{citep}；对于知识新颖性，我们采用了基座模型在预测token时的ppl来评估。ppl越高，说明模型对于预测下一个token的置信度越小\red{citep}，因此该token更有可能蕴含新颖的知识；我们用基于模型嵌入层的语义距离来计算相关性分数。我们首先用一些专家词汇定义了下游任务domain(\eg, ["math","number","logic"] for GSM8K)，并在无上下文的情况下计算这些词在base model上的平均嵌入层向量。然后，我们统计整个数据集中出现的tokens，并分别计算其无上下文时的嵌入层向量。最后我们计算了这些token向量到domain向量之间的距离，作为相关性分数。三种分数的具体计算公式如下：
We employ the base model's attention scores to assess reasoning importance. Attention is a mechanism that adaptively learns the contextual importance of tokens during the pretraining of the base model \citep{vaswani2017attention}. Existing work has demonstrated that masking low-attention tokens during the LLM reasoning process can even enhance generation quality \citep{gupta2021memorytop-k}. Therefore, using attention scores to assess reasoning importance aligns with our needs. We input the entire text (including both the input sentence and the output label sentence) into the model and compute the attention scores. We formulate the reasoning importance score $\mathcal{S}_{\mathrm{RI}}$ for the $k$-th token in the output label sentence as:
\begin{equation}
    \mathcal{S}_{\mathrm{RI}}(O_k) = \mathcal{A}\left(\theta, I + O\right)\left[l_I + k\right],
\end{equation}
where $\theta$ denotes the parameters of base model, $\mathcal{A}$ is the function to compute attention value, $I$, $O$ represents the input tokens and output label respectively, and $l_I$ is the length of input tokens. This approach can be applied to explain the tokens in the output label and considers the base model's understanding of the data throughout the answer generation logic. 

\partitle{PCP Score for KN} For knowledge novelty, we adopt the base model's \textbf{p}robability of \textbf{c}orrect token \textbf{p}rediction (PCP). Intuitively, a lower probability indicates reduced confidence in predicting the token, suggesting that it is more likely to contain knowledge that the model has not yet acquired. We formulate the knowledge novelty score $\mathcal{S}_{\mathrm{KN}}$ as:
\begin{equation}
    \mathcal{S}_{\mathrm{KN}}(O_k) = 1 - P\left(O_k \mid I + \left[O_0, O_1, \ldots, O_{k-1}\right]\right).
\end{equation}

\partitle{Distance Score for TR}
For relevance scoring, we calculate semantic distance based on the base model's embedding layer. First, we feed each data item from the dataset into the base model in its entirety and obtain the average value of the embedding layer vectors as the domain vector for the fine-tuning task. Then, we collect all tokens appearing in the dataset and compute their context-free embedding vectors. Finally, we measure the distance between these token vectors and the domain vector, using this distance as the relevance score. The specific formulas for the task relevance scores are as follows:
\begin{equation}
\begin{aligned}
    \mathcal{V}(\mathrm{Domain}&) = \frac{\sum \left(\mathcal{E}(\theta, exp_w)\right)}{n_{w}}, \\
    \mathcal{S}_{\mathrm{TR}}(O_k) = 1 - &\mathrm{Normalize}\left(\mathcal{D}\left(\mathcal{E}(O_k), \mathcal{V}(\mathrm{Domain})\right)\right),
\end{aligned}
\end{equation}
where $\mathcal{E}$ denotes the function to get the embedding vector of inputs, $exp_w$ and $n_w$ represents the expert words and its number, respectively. $\mathcal{D}$ is the function to compute the distance of two vector.

%S_RI(O_k) = A(\theta, I + O)[l_I + k]
%S_KN(O_k) = 1 - p(O_k | I + [O_0, O_1, ..., O_k-1])
%E(Domain) = （求和(E(\theta, words))) / words num
%S_TR(O_k) = 1 - Normalize(D(O_k, Domain))

\subsection{How to Fine-tune based on Scores?}
\label{sec:how to fine-tuning}
%在获得每个维度上的评分之后，我们根据评分数值的分布特征采取不同的筛选方案。如图所示，reasoning importance的分布较为极端，许多分数具有相同的大小，并且在正则化之后差别极大。我们直接采用分位数方法筛选掉极端小分数的tokens。knowledge novelty分数的分布较为均匀。Intuitively，我们认为probs超过某一阈值gamma的token预测是不包含新知识的，并将其视为噪声。最后是task relevance分数，其分布呈现集群特征。我们采用Muti-Ostu\red{citep}方法来划分task relevance分数，并将较小的两个集群过滤掉（最小的集群往往是空格替代符号）。Muti-Ostu方法的数学表示如下：
In this section, we will describe how to correctly filter noisy tokens using a conservative strategy and ignore them during the fine-tuning process through gradient masking.

\partitle{Token Filtering}
After obtaining scores for each dimension, we need to filter the noisy tokens based on the scores. As shown in Figure~\ref{fig:score distribution}, the distributions of scores across different tasks share the same features, requiring us to design adaptive mechanism for each attributes. Specifically, the reasoning importance scores exhibit an extreme distribution where many scores share identical values and demonstrate significant differences after normalization. We directly apply the quantile method (Interquartile Range) to filter out tokens with extremely low scores. In mathematical expression:
\begin{equation}
\begin{aligned}
    Q_1, Q3 = &\text{ Quantile}(S_{RI}(O),[25,75]),\\
    IQR &= Q_1 - (Q_3 - Q_1),\\
    O_k \in (D_{RI\downarrow})& \text{ if } S_{RI}(O_k) < Q_1 - IQR,
\end{aligned}
\end{equation}
where $S_{RI}(O)$ means all the reasoning importance score of tokens in output label $O$. The knowledge novelty scores display a uniform distribution, which makes it difficult to distinguish low scores. Therefore, we adopt a heuristic threshold. We consider tokens with a PCP higher than 95\% as only containing knowledge without novelty and treat them as noise. In mathematical expression:
\begin{equation}
O_k \in (D_{KN\downarrow}) \text{ if } S_{KN}(O_k) < 0.05.
\end{equation}
Finally, the task relevance scores exhibit cluster characteristics, for which we employ the Multi-Otsu method \citep{liu2009otsu} to partition the scores. Since the cluster with the smallest mean value typically consists of space replacement symbols, we filter out the tokens in the cluster with the second smallest mean value. In mathematical expression:
\begin{equation}
O_k \in (D_{TR\downarrow}) \text{ if } S_{TR}(O_k) \in \mathcal{M}(S_{TR})^{2nd},
\end{equation}
where$ \mathcal{M}$ is the Multi-Otsu method and $\mathcal{M}(S_{TR})^{2nd}$ denotes the cluster with the second smallest mean value. The detail of Multi-Otsu method is shown in Appendix~\ref{sec:more details}.

\begin{figure}[t]
  \begin{subfigure}{0.32\linewidth}
    \includegraphics[width=\linewidth]{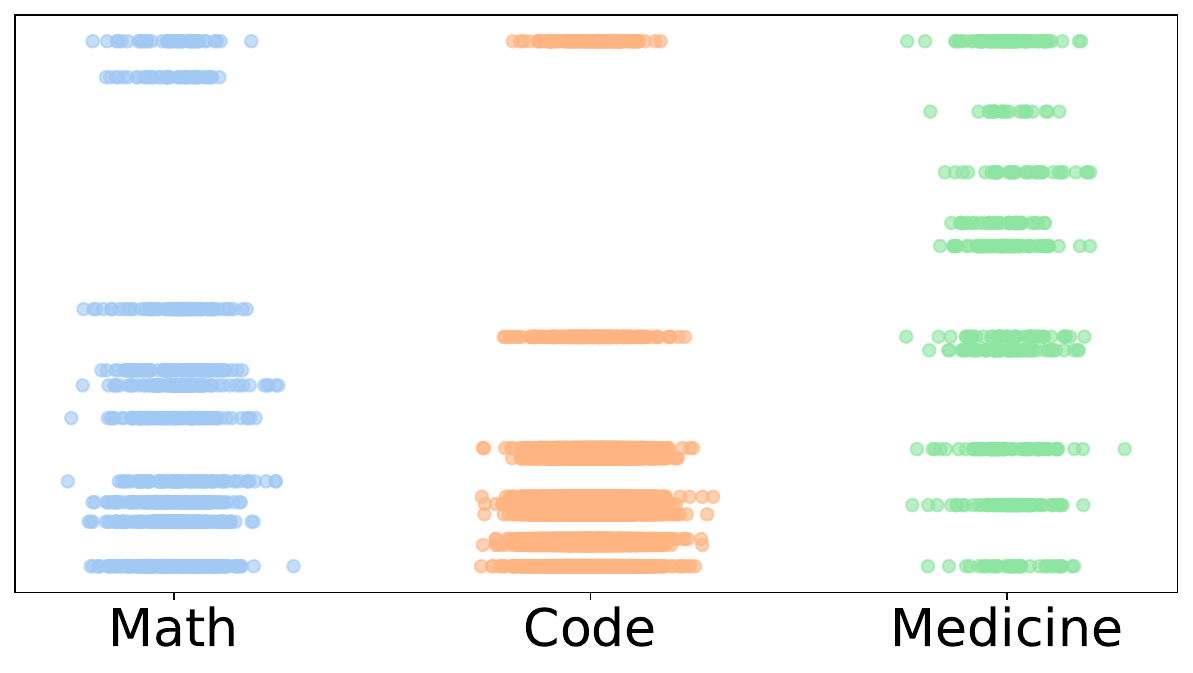}
    \caption{$\mathcal{S}_{\mathrm{RI}}$ on Deepseek}
  \end{subfigure}
  \hfill
  \begin{subfigure}{0.32\linewidth}
    \includegraphics[width=\linewidth]{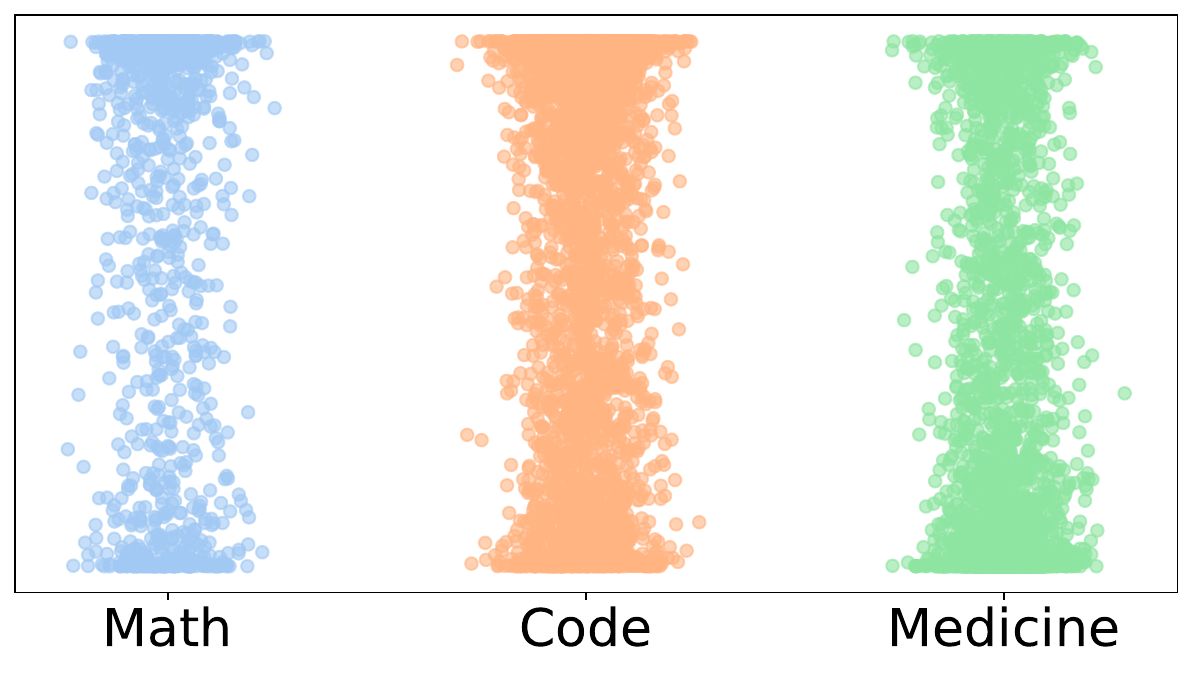}
    \caption{$\mathcal{S}_{\mathrm{KN}}$ on Deepseek}
  \end{subfigure}
  \hfill
  \begin{subfigure}{0.32\linewidth}
    \includegraphics[width=\linewidth]{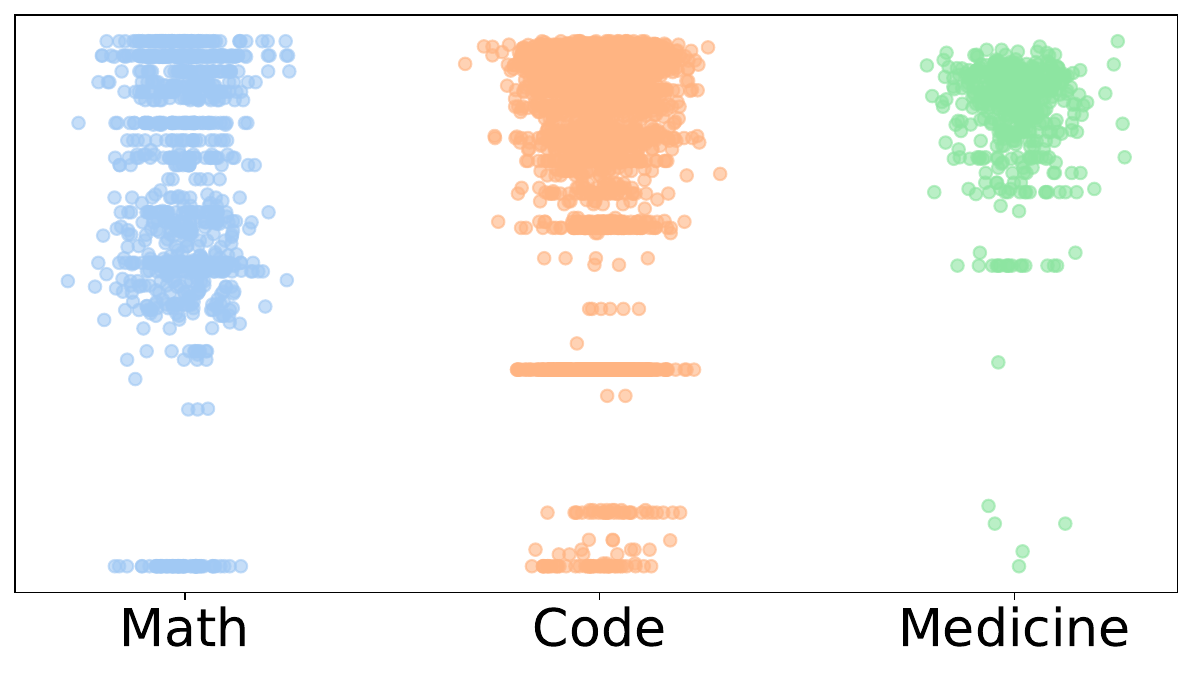}
    \caption{$\mathcal{S}_{\mathrm{TR}}$ on Deepseek}
  \end{subfigure}

  \caption{Distribution of the three scores across different datasets on Deepseek-1.5B. The reasoning importance score is distributed across some fixed values, the knowledge novelty score has a somewhat uniform distribution, and the task relevance score's distribution exhibits clustering features. More distribution figures of different LLMs are shown Appendix~\ref{sec:distribution figures}.}
  \label{fig:score distribution}
\end{figure}

\begin{figure}[t]
  \begin{subfigure}{0.32\linewidth}
    \includegraphics[width=\linewidth]{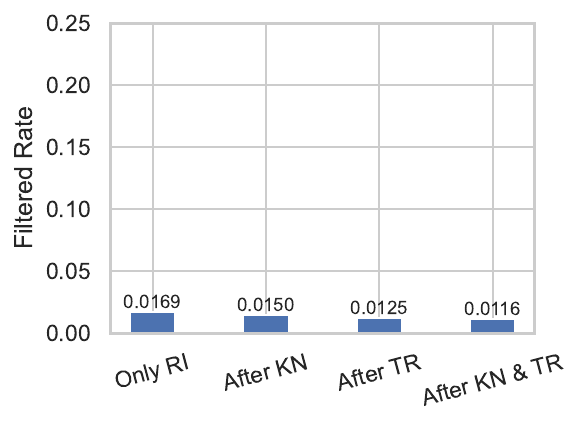}
    \caption{RI Filtering}
  \end{subfigure}
  \hfill
  \begin{subfigure}{0.32\linewidth}
    \includegraphics[width=\linewidth]{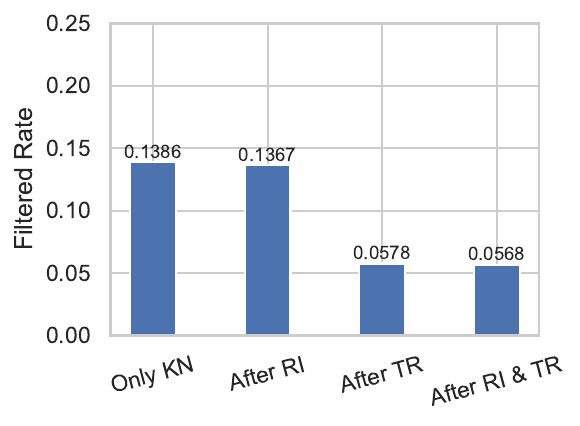}
    \caption{KN Filtering}
  \end{subfigure}
  \hfill
  \begin{subfigure}{0.32\linewidth}
    \includegraphics[width=\linewidth]{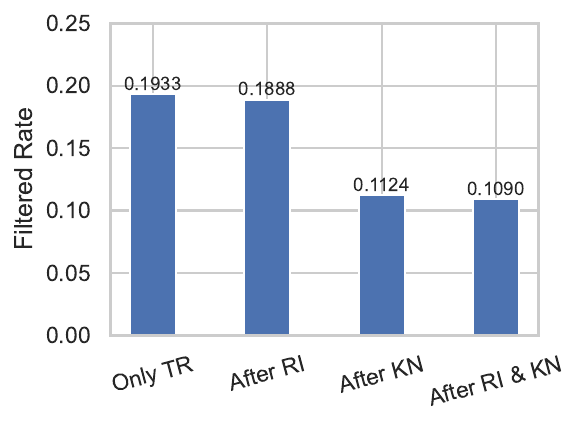}
    \caption{TR Filtering}
  \end{subfigure}

  \caption{Complementarity among attributes using Deepseek-1.5B and GSM8k. Using RI as an example, Only RI represents the percentage of tokens that can be filtered using only RI. After KN represents the percentage of tokens that RI can filter after KN is applied. The reduction in tokens when comparing After KN with Only RI indicates the tokens on which KN and RI overlap. }
  \label{fig:complementarity among attributes}
  \vspace{-1em}
\end{figure}

%XTF采用了一种激进的过滤策略，\ie，取三种过滤tokens的并集。支持这一策略的前提是，被过滤的token应当被认为“完全”不具备相应的attributes。因此，我们应当选择score显著低的tokens过滤，这也是为什么我们选择了保守的阈值。同时，XTF框架下，多个维度的过滤属性又能够完成互补，弥补了由宽松阈值带来的噪声疏漏。图xxx体现了属性间的互补性：不同属性过滤的tokens重复度非常低。这种策略基于一个简单但有用的启发式思维：一个视角下难以判断的模糊噪声，在另一个视角能够清晰分辨。当从多个视角共同判断噪声时，就能将一个复杂的问题转化为多个简单的问题。
\partitle{Threshold Analysis}
\Name adopts an aggressive filtering strategy, namely, it takes the union of three types of filtered tokens. The premise that supports this strategy is that the filtered tokens are assumed to be “completely” devoid of the corresponding attributes. Therefore, tokens with significantly low scores should be selected for filtering, which is also why a conservative threshold is chosen. Meanwhile, within the \Name framework, filtering attributes across multiple dimensions provide complementary effects, compensating for noise omissions introduced by the loose threshold. Figure~\ref{fig:complementarity among attributes} illustrates the complementarity among attributes: the overlap of tokens filtered by different attributes is not higher than 58.3\%. This strategy is based on a simple yet useful heuristic: ambiguous noise that is difficult to distinguish from one perspective can be clearly identified from another perspective. When noise is jointly evaluated from multiple perspectives, a complex problem is transformed into several simpler problems. To further explore the results of token filtering, we have provided specific examples of token filtering in Appendix~\ref{sec:examples}.

%本节我们将介绍如何通过tokens筛选的结果优化微调。As shown in Figure 1，在处理数据时，由于微调任务仅关心输出部分的正确性，因此需要通过一个默认的数值(往往是[-100])来标记输入部分的token并在传递梯度时将其屏蔽。在筛选出output label中的noisy tokens之后，我们将同样标记其为[-100]，并作为最终用来微调的数据。最终的损失函数可以表示为：
\partitle{Training the Model}
Here we describe how to optimize fine-tuning based on token filtering results. As shown in Figure~\ref{fig:pipeline}, when processing the data, because the fine-tuning task focuses solely on the correctness of the output, we mask the input tokens by assigning them a default value (often \textit{[-100]}) and thereby exclude them from gradient computation. After identifying noisy tokens in the output labels, we mark them with the this default value and use the resulting data for fine-tuning. Given the noisy token list $N$, the loss function $\mathcal{L}_F$ for learning a data item can be expressed as:
\begin{equation}
    \mathcal{L}_F = -\sum_{\substack{ O_k \notin N}} \log P\bigl(O_k \mid I + [O_0, O_1, \dots, O_{k-1}]\bigr).
\end{equation}

\section{Experiments}
\label{sec:experiments}

\subsection{Experiment Settings}
\partitle{Dataset}
We select three representative downstream tasks to evaluate the fine-tuning performance, including two mainstream tasks: math, and code (widely used to evaluate LLMs \citep{guo2025deepseek, openai2024gpt4technicalreport,team2024gemma}), as well as an important specialized tasks: medicine \citep{Arun2023naturemedicine,alberts2023largemedicine}. For the math task, we employ the GSM8K \citep{cobbe2021gsm8k} for fine-tuning and evaluation. For the code task, we fine-tune on the CodeExercise \citep{CodeExercise2025code} and evaluate using HumanEval \citep{chen2021humaneval}. For medicine tasks, we employ the PubMedQA \citep{jin2019pubmedqa} for fine-tuning and evaluation. We also conduct additional experiments based on the NuminaMath-CoT \citep{li2024numinamath}, MATH-500 \citep{hendrycks2021measuringmath500}, and FIQA \citep{maia201818fiqa} datasets, but due to space limitations, these results are presented in the Appendix~\ref{sec:extension of main experiment}.

%我们选择了5个具有代表性的垂域任务来评估微调的效果，包括了三个主流的任务：question-answering， math and code，以及两个常见的专业型任务：金融、医疗。在question answering 任务中，我们采用了xxx数据集；在math任务中，我们采用GSM8K\red{citep}数据集来微调和评估；在code任务中，我们采用CodeExercise\red{citep}进行微调并使用humaneval\red{citep}进行评估；在金融和医疗任务上，我们分别使用xxx数据集\red{citep}和xxx数据集\red{citep}来微调和评估。

\partitle{LLMs}
We select 7 different base LLMs of varying scales from three outstanding model families: DeepSeek \citep{guo2025deepseek}, Llama \citep{touvron2023llama,grattafiori2024llama} and Mistral \citep{Mistral_7B2023mistral}. Specifically, for the Deepseek family, we choose DeepSeek-R1-distilled-Qwen-1.5B, 7B and 14B \citep{bai2023qwen}. For the Llama family, we select Llama-3.2-1B, 3B and Llama-3.1-8B. For the Mistral family, we select Mistral-v0.1-7B. 
%我们选择了6个不同规模基座大模型来自于4个优秀的模型系列：Deepseek， Llama， Gemma and Mistral。具体来说，最小规模（1B-2B）的模型中，我们选择了Llama-3.2-1B 和 Deepseek-distilled-qwen-1.5B；中等规模（2B-7B）的模型，我们选择了Gemma-2B和Llama-3.2-3B；稍大规模（7B）的模型，我们选择了Deepseek-distilled-qwen-7B和mistral-7B。

% \partitle{Baselines}
% We consider three regular LLM implementations and three data enhancement methods to demonstrate the effectiveness of our \Name. Specifically, for regular LLM implementations, we adopt clean accuracy (CA), i.e., using the original base model, normal fine-tuning (Normal) and more epochs ($\times$2 epochs). Data enhancement methods can filter the data noise from different perspectives. Specifically, data filtering (DF) \citep{li2024superfiltering} filters noisy data at the sample level, data augmentation (DA) \citep{dai2025auggpt} augments more data to enhance the model's robustness, thus resist noise, and selective language model training \citep{lin2024not-all-tokens} trains the token-level data selectively based on changes in the loss value. More details are shown in Appendix~\ref{sec:baselines setting}.
\partitle{Baselines}
We consider 3 regular LLM implementations and 4 data enhancement methods to demonstrate the effectiveness of our \Name. Specifically, for regular LLM implementations, we adopt clean accuracy (CA), i.e., using the original base model, normal fine-tuning (Normal) and more epochs ($\times$2 epochs). Data enhancement methods can filter the data noise from different perspectives. Specifically, data filtering (DF) \citep{li2024superfiltering} filters noisy data at the sample level, data augmentation (DA) \citep{dai2025auggpt} augments more data to enhance the model's robustness, thus resist noise, selective language model training \citep{lin2024not-all-tokens} trains the token-level data selectively based on changes in the loss value, and token cleaning (TC) \citep{pang2025tokenclean} performs fine-grained data selection for LLM supervised fine-tuning. More details are shown in Appendix~\ref{sec:baselines setting}.

\partitle{Metrics}
We primarily use accuracy to evaluate the performance of LLMs on specific tasks. For tasks with a standard answer, we adopt a zero-shot form to pose the questions and use a judge model \citep{claude-3.5-sonnet} to determine the correctness of the responses. This evaluation method is less influenced by the format of the prompt and the model's response style, providing a clear and intuitive reflection of the fine-tuning effect. For the code completion task, we assess performance using pass@1, pass@5, and pass@10 metrics, which are specifically used to evaluate code generation tasks \citep{kulal2019spocpass@k}.
%我们主要采用Accuracy来对大模型在具体任务上的表现进行打分。对于需要判断对错的任务，\eg，question-answering、math、finance和medecine，我们采用judgement model计算精度，对于较为困难的生成类任务code，我们采用pass@1，pass@5，pass@10来判断其性能.

\partitle{Hyperparameter}
We conduct experiments based on existing work and strictly control the fairness of the results through hyperparameters, as detailed in Appendix~\ref{sec:hyperparameters}.

\subsection{Main Experiment Results}
\label{sec:main experiment results}

\begin{table}[t]
\centering
\caption{Result of main experiment. We show the accuracy of LLMs across different fine-tuning methods. Best results are marked in \textbf{bold} and the second best results are marked with underline. }
\label{tab:main}
\resizebox{\linewidth}{!}{\begin{tabular}{lrccccccccc}
\toprule
\multicolumn{11}{c}{\textbf{MATH: Fine-tuning and evaluate models on GSM8K}} \\
\midrule
\textbf{Model} & $\mathbf{|\theta|}$ & \textbf{LoRA} & \textbf{CA} & \textbf{Normal} & $\mathbf{\times}$\textbf{2 Ep} &\textbf{DF} & \textbf{DA} & \textbf{SLM} & \textbf{TC} & \textbf{\Name} \\
\midrule
Llama-3.2 & 1B & $\times$ & 2.8 & 4.3 & 6.8 & \underline{7.6} & 2.4 & 5.9 &6.8 & \textbf{8.7} \\
Llama-3.2 & 3B & $\times$ & 3.9 & 25.8 & 33.4 & 36.9 & 27.1 & \underline{38.8} & 38.4& \textbf{40.5} \\
Llama-3.1 & 8B & $\times$ & 4.6 & 54.0 & 55.4 & 52.7 & 55.4 & \textbf{60.3} & \underline{58.9}& 58.7 \\
Llama-3.1 & 8B & $\checkmark$ & 4.6 & 33.7 & 32.7 & 37.0 & 33.7 & \underline{37.9} & \textbf{38.4}& 37.1 \\
Mistral & 7B & $\times$ & 8.0 & 15.0 & 16.1 & 21.3 & 18.4 & 22.6 & \underline{24.1}& \textbf{29.1} \\
Mistral & 7B & $\checkmark$  & 8.0 & 23.4 & 20.7 & \underline{25.6} & 21.8 & 21.3 & 20.7& \textbf{25.6} \\
Deepseek-distilled-qwen & 1.5B & $\times$ & 17.6 & 42.9 & 45.5 & \underline{47.0} & 37.4 & 37.3 &38.5 & \textbf{56.2} \\
Deepseek-distilled-qwen & 7B & $\times$ & 37.9 & 63.0 & \underline{68.1} & 65.5 & 56.5 & 63.8 &61.9 & \textbf{69.3} \\
Deepseek-distilled-qwen & 7B & $\checkmark$  & 37.9 & 61.3 & 60.4 & 67.9 & 62.0 & 68.2 &\underline{70.0} & \textbf{71.8} \\
Deepseek-distilled-qwen & 14B & $\checkmark$  & 34.5 & 47.6 & 50.3 & \underline{52.4} & 50.5 & 49.3 & 52.1& \textbf{60.3} \\
\midrule
Average & --& -- & 16.0 & 37.1 & 39.0 & \underline{41.4} & 36.5 & 40.5 &41.0 & \textbf{45.7} \\
\midrule
\multicolumn{11}{c}{\textbf{MEDICINE: Fine-tuning and evaluate models on PubMedQA}} \\
\midrule
\textbf{Model} & $\mathbf{|\theta|}$ & \textbf{LoRA} & \textbf{CA} & \textbf{Normal} & $\mathbf{\times}$\textbf{2 Ep} &\textbf{DF} & \textbf{DA} & \textbf{SLM} & \textbf{TC} & \textbf{\Name} \\
\midrule
Llama-3.2 & 1B & $\times$ & 2.9 & 15.5 & 15.6 & 15.4 & \underline{18.3} & 13.1 &12.4 & \textbf{18.5} \\
Llama-3.2 & 3B & $\times$ & 6.6 & 35.5 & 33.7 & 29.8 & \textbf{37.9} & 36.1 &36.4 & \underline{36.5} \\
Llama-3.1 & 8B & $\times$ & 4.9 & 13.6 & 14.0 & 18.4 & 14.7 & \textbf{22.9} &22.9 & \underline{21.3} \\
Llama-3.1 & 8B & $\checkmark$ & 4.9 & 24.2 & 22.3 & \underline{34.3} & 31.4 & 26.3 &27.1 & \textbf{37.9} \\
Mistral & 7B & $\times$ & 8.6 & 20.0 & 26.2 & 18.7 & 23.4  &\underline{26.5} & 26.2 & \textbf{32.0} \\
Mistral & 7B & $\checkmark$  & 8.6 & 15.5 & \underline{18.4}  & 15.6 & 13.1 & 17.5 &16.9 & \textbf{21.3} \\
Deepseek-distilled-qwen & 1.5B & $\times$ & 39.8 & 44.6 & 45.4 & 41.2 & {49.7} & 48.3 &\textbf{51.2} & \underline{50.8} \\
Deepseek-distilled-qwen & 7B & $\times$ & 42.7 & 50.5 & 42.1 & \underline{55.4} & 52.7 & 51.3 & 51.6& \textbf{55.6} \\
Deepseek-distilled-qwen & 7B & $\checkmark$  & \textbf{42.7} & 35.9 & 37.8 & 33.6 & 38.4 & 39.0 &37.4 & \underline{41.7} \\
Deepseek-distilled-qwen & 14B & $\checkmark$  & 44.6 & 48.5 & 43.4 & 51.3 & 47.0 & 53.9 &\underline{54.6} & \textbf{55.7} \\
\midrule
Average & --& -- & 20.6 & 30.4 & 29.9 & 31.4 & 32.7 & 33.5&\underline{33.7} & \textbf{37.1} \\
\bottomrule
\end{tabular}}
\end{table}

\begin{figure}[t]
  \begin{subfigure}{\linewidth}
    \includegraphics[width=\linewidth]{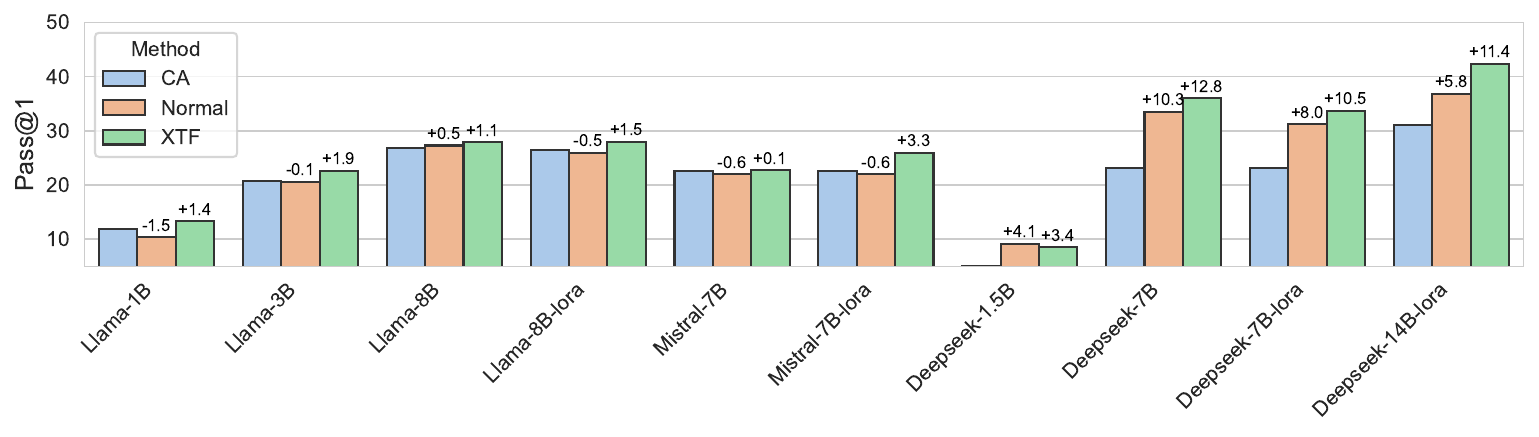}
  \end{subfigure}
  
  \begin{subfigure}{\linewidth}
    \includegraphics[width=\linewidth]{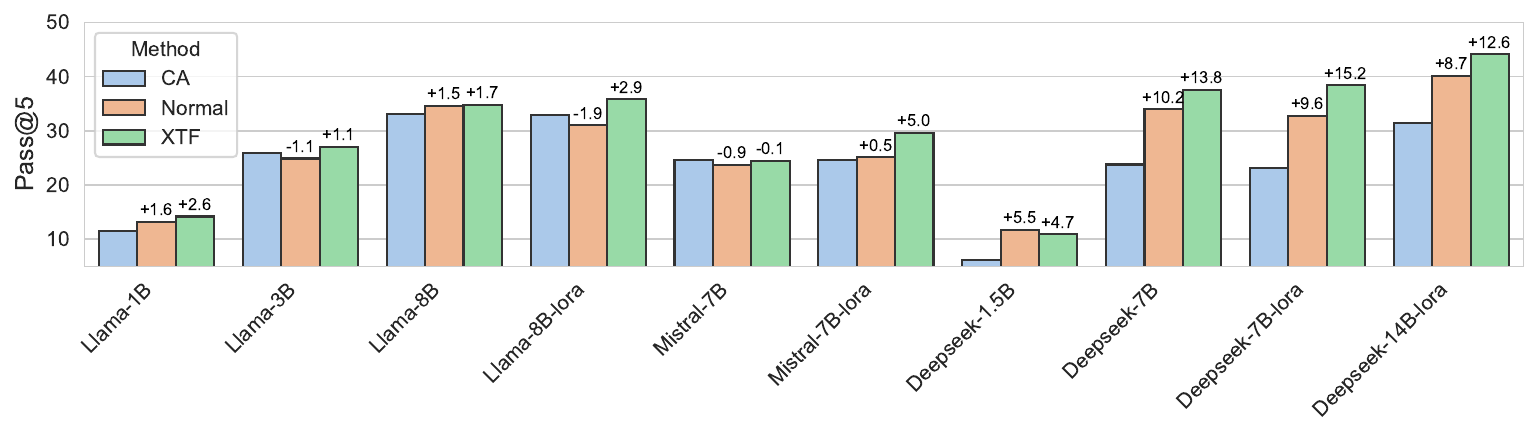}
  \end{subfigure}
  
  \begin{subfigure}{\linewidth}
    \includegraphics[width=\linewidth]{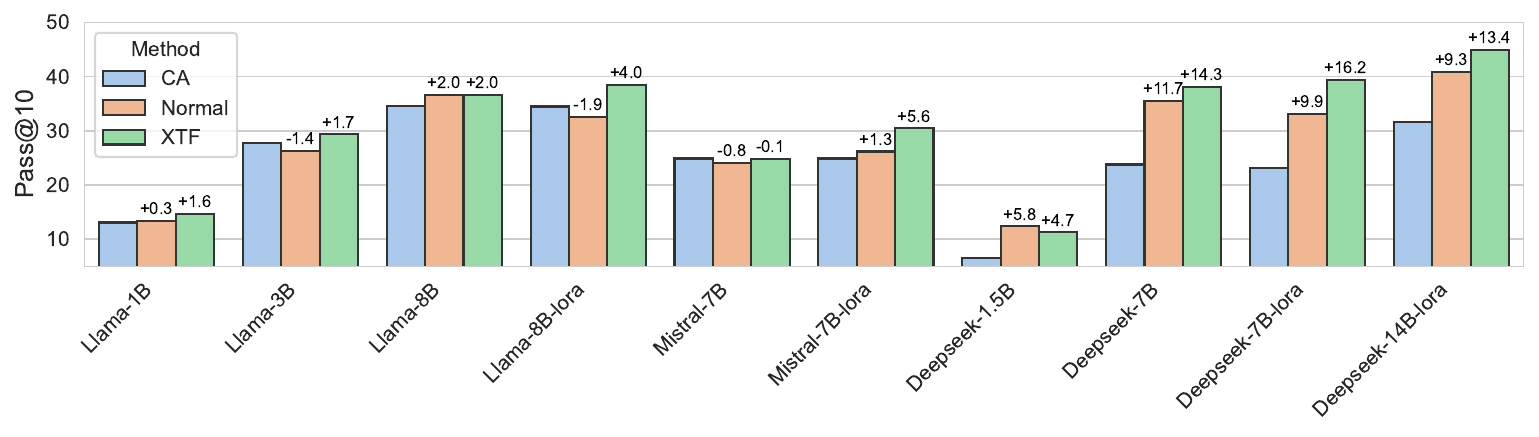}
  \end{subfigure}
  \caption{The results on code task. We show the results of pass@1, pass@5 and pass@10 respectively.}
  \label{fig:code results}
  \vspace{-1em}
\end{figure}

%我们在主实验中采用了大量的微调cases以及三种具有代表性的不同downstream tasks。具体来说，as shown in Table~\ref{tab:main} and Figure~\ref{fig:code results}，主实验一共采用了10种不同的cases，包括了7种不同的模型，以及全参数微调和LoRA微调两种微调方式。LoRA微调只在较大规模（larger than 7B）模型上进行。三种downstream tasks除了领域不同以外，在数据细节上也有所差异：用于math任务微调的GSM8K数据质量较高，而用于code评估的humaneval测试最为复杂且包含多轮生成结果。

We employ a significant number of fine-tuning cases and three representative downstream tasks in our main experiment. By conducting fine-tuning experiments on different models using unified hyperparameters and datasets, and comparing the performance of the fine-tuned models, we can assess the effectiveness of the fine-tuning methods. During the training process, we employ a training set, validation set, and test set to prevent issues such as overfitting caused by varying convergence speeds. Specifically, the model parameters that perform best on the validation set are retained as the final model parameters and tested on the test set. The reported results are test set accuracy.

\partitle{Math}
Math is an important downstream task in LLM research. It is widely adopted as the LLM benchmark, and the performance on math task can reflect the logic ability of LLMs. As shown in Table~\ref{tab:main}, \Name has average 8.6\% higher accuracy than normal fine-tuning and 4.3\% higher accuracy than the best baseline DF. In all 10 cases, \Name obtains 8 best results and 2 second best results. In particular, when fine-tuning the Deepseek-distilled-qwen-1.5B with all parameters, \Name achieves 13.3\% higher accuracy than normal fine-tuning and 9.3\% higher accuracy than the best baseline DF. Due to space limitations, additional experiments on math tasks (NuminaMath-CoT and MATH-500) will be presented in the Appendix~\ref{sec:extension of main experiment}.
%表现最差的baseline是DA，其平均accuracy不如正常微调训练。我们推测这是因为GSM8k种数据噪声较小，引入更多的数据无法增加鲁棒性反而引入了更多噪声。令我们惊讶的是，DF在许多cases下表现出较好的结果，相比于Normal baseline平均提高了4.3\%。这个现象说明即使是粗糙的（sample-level）过滤，消除噪声带来的好处总体而言也大于坏处，进一步证明了噪声过滤任务本身的重要性。

%We first evaluate different methods on a math task and a medicine task. As shown in Table~\ref{tab:main}, \Name outperforms all the baselines in both the math task and the medicine task, with an average of 8.6\% and 6.3\% higher accuracy than normal fine-tuning, respectively. When fine-tuning the Deepseek-distilled-qwen-1.5B on the math task, \Name achieves 13.3\% higher accuracy than normal fine-tuning and 9.3\% higher accuracy than the best baseline. When fine-tuning the Llama-3.1-8B-LoRA on the medicine task, \Name demonstrates 13.7\% higher accuracy than normal fine-tuning and 3.6\% higher accuracy than the best baseline. Despite some baselines, such as DF and SLM, also optimizing the fine-tuning performance in almost every case, they cannot outperform \Name, indicating that \Name has a better noise filtering effect. Besides, the $\times$2 Ep shows a minor difference compared to Normal, demonstrating the convergence of the fine-tuning process under our experimental settings. In summary, \Name exhibits the best performance and can effectively optimize the fine-tuning of LLMs.

\partitle{Medicine}
Medicine is a promising application area for LLMs, and LLM researches on pharmaceutical downstream tasks \citep{Arun2023naturemedicine, alberts2023largemedicine} have already achieved widespread influence.
As shown in Table~\ref{tab:main}, \Name has average 6.7\% higher accuracy than normal fine-tuning and 3.4\% higher accuracy than the best baseline TC. In all the 10 cases, \Name obtains 6 best results and 4 second best results. When fine-tuning the Llama-3.1-8B with LoRA, \Name demonstrates 13.7\% higher accuracy than normal fine-tuning and 3.6\% higher accuracy than the best baseline. 
%DA方法在medicine任务中的表现优于math任务，这很可能与PubMedqa数据质量低于GSM8k有关。SLM在math和medicine任务中都具备较强的竞争力，然而其综合表现和鲁棒性均显著低于我们的方法。具体来说，SLM在两个任务上仅分别超出normal baseline 3.4\% 和 3.1\%(8.6\% and 6.7 of \Name)，并存在3个弱于Normal的cases（\Name中是0个）。

\partitle{Code}
The code task is a more challenging task, and we provide pass@1, pass@5, and pass@10 results. Due to the generally lower accuracy on code tasks, the gap between \Name and the baseline is not as significant as in other experiments. As shown in Figure~\ref{fig:code results}, \Name generally exhibits better results than normal fine-tuning. This difference increases when the LLMs are given more generation chances (from pass@1 to pass@10). In certain cases, normal fine-tuning decreases accuracy, indicating the harmfulness of noisy tokens, while \Name still shows positive performance. We find that when fine-tuning larger-scale models, the effect of noise filtering is more pronounced than when fine-tuning smaller-scale models. This phenomenon aligns with the claim we propose in Section~\ref{sec:which data act as noise for fine-tuning?}: finding noise in data should consider the base model's knowledge. Since larger-scale models generally possess stronger performance, we believe that stronger base model performance can better leverage the \Name method.

%我们发现，在微调规模较大的模型时，噪声过滤所能起到的作用比微调规模较小的模型更加明显。这一现象符合我们在Section~\ref{sec:intro}中提出的声明：判断数据中的噪声应当考虑base model的pre-existing knowledge。考虑到一般情况下规模较大的模型拥有更强的性能，我们认为较强的base model性能能够进一步体现\Name的价值。

\subsection{Ablation Study}
\begin{table}[t]
\centering
\caption{Ablation study of \Name across different settings. DS, LA and MS denotes Deepseek, Llama and Mistral employed in this experiment. Ma and Me denotes math task and medicine task respectively. $\times$ means the corresponding noise has been filtered while $-$ means not. Best results are marked in \textbf{bold} and the second best results are marked with underline.}
\label{tab:ablation study}
\resizebox{\linewidth}{!}{\begin{tabular}{ccccccccccc}
\toprule
\textbf{Case} & $D_{RI\downarrow}$ & \textbf{$D_{KN\downarrow}$} & \textbf{$D_{TR\downarrow}$} & \textbf{DS(Ma)} & \textbf{LA(Ma)} & \textbf{MS(Ma)} & \textbf{DS(Me)} & \textbf{LA(Me)} & \textbf{MS(Me)} & \textbf{Avg} \\ \midrule
Zero      & $-$ & $-$ & $-$ & 42.9 & 25.8 & 15.1 & 44.6 & 35.5 & 20.0 & 30.7 \\ 
I     & $\times$ & $-$ & $-$ & 44.0 & 28.3 & 16.6 & 44.7 & 36.1 & 22.3 & 32.0\\ 
II    & $-$ & $\times$ & $-$ & 48.1 & 28.4 & 20.2 & 46.8 & 36.2 & 27.1 & 34.5\\ 
III     & $-$ & $-$ & $\times$ & 45.3 & 30.1 & 17.7 & 45.7 & 35.5 & 25.4 & 33.3\\ 
IV      & $\times$ & $\times$ & $-$ & 48.3 & 32.2 & 23.9 & 47.8 & \underline{36.4} & 28.1 &36.1\\ 
V     & $\times$ & $-$  & $\times$ & \underline{49.2} & \underline{34.1} & \underline{27.7} & 47.1 & 35.9 & 27.6 & \underline{36.9}\\ 
VI    & $-$ & $\times$ & $\times$ & 47.3 & 32.9 & 22.6 & \underline{48.3} & 36.1 & \underline{30.9} & 36.3\\ \midrule
\Name (Ours)   & $\times$ & $\times$ & $\times$ & \textbf{56.2} & \textbf{40.5} & \textbf{29.1} & \textbf{50.8} & \textbf{36.5} & \textbf{32.0} &\textbf{40.1}\\ \bottomrule
\end{tabular}}
\end{table}

In this section, we conduct an ablation study on three noise filtering attributes. We selectively use the attributes to filter the noisy tokens and train models from different series.
%filter these tokens or not in this experiment to explore the specific influence of these noisy tokens.

As shown in Table~\ref{tab:ablation study}, \Name consistently demonstrates the best performance compared to other settings, suggesting that all the attributes are necessary for better noise filtering. At the same time, we observe that in mathematical tasks, using the combination of RI and KN to filter tokens is consistently superior to other combinations, whereas this is not the case in the Medicine task. Additionally, in the Medicine task, the optimal combination of attributes differs across models. These phenomena suggest that the relative effectiveness of the three attributes can vary depending on the model and task type, which aligns with our understanding of fine-tuning discussed in Section~\ref{sec:which data act as noise for fine-tuning?}.

In addition, we also conduct ablation study against the threshold (for token filtering) selection, which are detailed in the Appendix~\ref{sec:ablation study about threshold}.

%In particular, the second-best results usually exhibit a significant difference (more than 2\% in 4 cases) from the corresponding \Name results, indicating that unfiltered noise can negatively influence the final performance and increase this difference. Furthermore, the influence of the three kinds of noisy tokens varies across different cases, \eg, filtering $D_{RI\downarrow}$ and $D_{TR\downarrow}$ is more effective than filtering $D_{KN\downarrow}$ and $D_{TR\downarrow}$ in math task, while the opposite is true in the medicine task. In summary, each kind of noisy token can lead to a negative influence on fine-tuning, and all the noise filtering modules in \Name are necessary.

\section{Discussion}
\label{sec:discussion}
%我们提出的\Name能够有效地增强LLM微调，但是它仍然存在一些缺陷。具体来说，\Name可能无法过滤全部类型的noisy tokens，例如：部分困惑度极高的异常token label很难学习，并且有可能严重破坏模型原本性能。尽管我们已经试图用统计学方法过滤这部分tokens，但是没能取得较好的效果（低于正常微调）。像SLM一样使用loss值作为指导是一个潜在的解决方案，但是如何正确地构建reference model是一项挑战，将在我们未来的工作被继续探索。
Our proposed \Name effectively enhances LLM fine-tuning, but it still has some limitations. First, regarding computational cost, \Name incurs inference-level computational overhead (detailed in the Appendix~\ref{sec:computational overhead}), which performs better compared to existing methods that train a reference model. However, it still imposes a significant burden when dealing with large models. If a small distilled model could be provided to identify noise for large-scale base models, it would significantly reduce the cost of token scoring. Additionally, we believe that more attributes can be explored for filtering noise, and these attributes can be assessed from multiple perspectives. Such work would provide an effective complement to the application of \Name in real-world scenarios.
%Specifically, \Name might not be able to filter all types of noisy tokens. For example, some anomalous token labels with extremely high perplexity are difficult to learn and could severely degrade the model's original performance. Although we have attempted to filter these tokens using statistical methods, we have not achieved satisfactory results (worse than normal fine-tuning). Using loss values as guidance, similar to SLM \citep{lin2024not-all-tokens}, is a potential solution, but how to correctly construct a reference model is a challenge, which will be further explored in our future work.

\section{Conclusion}

In this paper, we investigated the influence of training data on fine-tuning performance at the token-level. We explored solutions for filtering token-level noise to optimize fine-tuning datasets using three decomposed dimensions: reasoning importance, knowledge novelty, and task relevance, subsequently proposing \Name. We conducted extensive experiments on datasets across 3 different downstream tasks and 7 representative LLMs. \Name achieved up to 13.3\%, 13.7\% and 6.3\% accuracy optimization on the math task, medicine task and code task, respectively, and outperformed all the baselines overall, demonstrating its effectiveness in noise filtering and fine-tuning enhancement.

% 致谢部分（放在参考文献之后）
\subsubsection*{Acknowledgements}
This work was supported in part by the National Natural Science Foundation of China (Grant Nos. 62502435, 625B1032, 62441238, U2441240), the Zhejiang Provincial Natural Science Foundation (No. LQN26F020002), the “Pioneer” and “Leading Goose” R\&D Program of Zhejiang (No. 2024C01169), and the Kunpeng-Ascend Science and Education Innovation Excellence/Incubation Center.

\bibliography{iclr2026_conference}
\bibliographystyle{iclr2026_conference}

\appendix
\newcommand{\dataset}{{\mathcal{D}}}
\newcommand{\params}{{\bm{\theta}}}
\newcommand{\paramsopt}{{\bm{\theta}_{\text{opt}}}}
\newcommand{\paramsbase}{{\bm{\theta}_{0}}}
\newcommand{\losstask}{{\mathcal{L}_{\text{task}}}}
\newcommand{\lossideal}{{L_{\text{ideal}}}}
\newcommand{\lossmod}{{\mathcal{L}'_{\text{\Name}}}}
\newcommand{\probmodel}{{P_{\params}}}
\newcommand{\probopt}{{P_{\paramsopt}}}
\newcommand{\probtask}{{\mathcal{P}_{\text{task}}}}
\newcommand{\probbase}{{P_{\paramsbase}}}
\newcommand{\grad}{{\nabla_{\params}}}
\newcommand{\expect}{{\mathbb{E}}}
\newcommand{\indicator}{{\mathbb{I}}}
\newcommand{\setcomplement}[1]{{#1}^{\mathsf{c}}} % For set complement or y \setminus N(y)

\appendix
\section{Theoretical Foundations of \Name}
\label{sec:theory}

\paragraph{Executive summary.}
We provide a complete, assumption–transparent, and \emph{parameterization–invariant} analysis of how token filtering improves the alignment (in a Riemannian sense) between the training direction and the ``ideal'' gradient direction. 
We (i) clarify the geometry and regularity of the Fisher information (via damping), (ii) decouple the analysis from any specific optimizer by working with an \emph{arbitrary} symmetric positive–definite (SPD) preconditioner $M$ (with $M = F_\lambda$ recovering natural gradient and $M = I$ recovering Euclidean SGD), (iii) make precise the statistical assumption on the \emph{selector} (and quantify the effect of selection bias), (iv) replace global smoothness by \emph{local} smoothness, and (v) tighten the treatment of high–confidence (\textsc{KN}) tokens with a clean and fully rigorous bound. 
Throughout, all random variables are defined on a common probability space; expectations are with respect to the distributions explicitly indicated.

\subsection{Preliminaries and geometry}
\label{app:preliminaries}

\paragraph{Contexts and model.}
A token–level \emph{context} is $c \defeq (x, t_{<i})$ and the associated label (gold token) is $t$.
Let $p_\theta(t \mid c)$ denote the model's conditional distribution at parameters $\theta$, and let
\[
\phi_\theta(c,t) \;\defeq\; \nabla_\theta \log p_\theta(t \mid c)
\]
be the \emph{score} of the pair $(c,t)$.

\paragraph{Teacher–forcing context law.}
Filtering will \emph{remove} some token–level losses but, under teacher forcing, it does not alter the forward contexts $c$. We therefore fix a baseline context distribution
\begin{equation}
\label{eq:Q}
Q(c)\quad\text{(the teacher–forcing context law; e.g.\ the empirical distribution over contexts in the corpus).}
\end{equation}
All expectations below that involve $p_\theta(\cdot\mid c)$ are taken with $c\sim Q$ first.

\paragraph{Damped Fisher information.}
Define the (damped) conditional Fisher information at $\theta$:
\begin{equation}
\label{eq:Fisher}
F_\lambda(\theta)
\;\defeq\;
\mathbb{E}_{c\sim Q}\,\mathbb{E}_{t\sim p_\theta(\cdot\mid c)}\!\big[\phi_\theta(c,t)\,\phi_\theta(c,t)^\top\big]
\;+\;\lambda I,
\qquad \lambda\ge 0.
\end{equation}
We will assume $F_\lambda(\theta)\succ 0$ for all $\theta$ under consideration; see Assumption~\ref{assump:A0}. 
Given any SPD matrix $M(\theta)\succ 0$ (to be specified below), we write
\[
\langle u, v\rangle_{M} \;\defeq\; u^\top M(\theta) v,
\qquad
\|u\|_{M^{-1}}\;\defeq\;\sqrt{u^\top M(\theta)^{-1}u}.
\]
The geometric viewpoint will be stated for a \emph{general} preconditioner $M$; the \emph{natural gradient} corresponds to the special choice $M=F_\lambda$.

\paragraph{Ideal risk.}
Let $p_\star(c,t)$ denote the (unknown) \emph{ideal} task distribution over $(c,t)$, with marginal $Q(c)$ under teacher forcing.
The ideal per–token KL risk is
\begin{equation}
\label{eq:ideal-risk}
\Loss_\star(\theta)
\;\defeq\;
\mathbb{E}_{(c,t)\sim p_\star}\big[-\log p_\theta(t\mid c)\big]
\;=\; \mathrm{KL}\!\big(p_\star \,\Vert\, p_\theta\big) + \text{const.}
\end{equation}
Since $p_\star$ does not depend on $\theta$,
\begin{equation}
\label{eq:ideal-grad}
\nabla_\theta \Loss_\star(\theta) \;=\; -\,\mathbb{E}_{(c,t)\sim p_\star}\big[\phi_\theta(c,t)\big].
\end{equation}
For any SPD $M$, we write the (preconditioned) gradient direction as $\widetilde{\nabla}_M \Loss_\star(\theta) \defeq M(\theta)^{-1}\nabla_\theta \Loss_\star(\theta)$; for $M=F_\lambda$ this is the damped natural gradient.

\subsection{Data model, selector, and assumptions}
\label{app:assumptions}

\begin{assumption}[Damped Fisher and teacher forcing]
\label{assump:A0}
(1) $Q(c)$ is fixed by teacher forcing and independent of filtering. 
(2) There exists $\lambda\ge 0$ such that $F_\lambda(\theta)\succ 0$ for all $\theta$ in a neighborhood of interest (e.g.\ along the optimization trajectory).
\end{assumption}

\begin{assumption}[Mixture model of labels]
\label{assump:A1}
There exists $\varepsilon\in[0,1)$ and distributions $p_{\rm core}, p_{\rm noise}$ on $(c,t)$ such that
\begin{equation}
\label{eq:mixture}
p_{\rm train}(c,t)
\;=\;
(1-\varepsilon)\,p_{\rm core}(c,t) + \varepsilon\,p_{\rm noise}(c,t),
\qquad
\text{with } p_{\rm core}(c,t) \equiv p_\star(c,t).
\end{equation}
\end{assumption}

\begin{assumption}[Selector quality \& independence within components]
\label{assump:A2}
Let $Z(c,t)\in\{0,1\}$ be the indicator that the token is kept by the filter. 
Define the error rates
\begin{equation}
\label{eq:selector_error}
\alpha \;\defeq\; \Pr[Z=0 \mid (c,t)\sim p_{\rm core}], 
\qquad
\beta \;\defeq\; \Pr[Z=1 \mid (c,t)\sim p_{\rm noise}].
\end{equation}
We assume the selector has non–trivial skill: $\alpha+\beta<1$.
We consider two flavors of conditional independence:
\begin{itemize}[leftmargin=1.8em]
\item \emph{Strong (MAR within components).} Given the latent component label $G\in\{\mathrm{core},\mathrm{noise}\}$, $Z$ is independent of $(c,t)$:
\[
Z \perp (c,t) \,\mid\, G.
\]
\item \emph{Weak (bounded selection bias).} There exist $\rho_{\rm c},\rho_{\rm n}\ge 0$ such that
\begin{align}
\big\|\,\mathbb{E}[\phi_\theta(c,t)\mid G\!=\!\mathrm{core}, Z\!=\!1] - \mathbb{E}[\phi_\theta(c,t)\mid G\!=\!\mathrm{core}]\,\big\|_{M^{-1}} &\le \rho_{\rm c}\,\big\|g_{\rm core}(\theta)\big\|_{M^{-1}},
\label{eq:weak-bias-core}
\\
\big\|\,\mathbb{E}[\phi_\theta(c,t)\mid G\!=\!\mathrm{noise}, Z\!=\!1] - \mathbb{E}[\phi_\theta(c,t)\mid G\!=\!\mathrm{noise}]\,\big\|_{M^{-1}} &\le \rho_{\rm n}\,\big\|g_{\rm core}(\theta)\big\|_{M^{-1}}.
\label{eq:weak-bias-noise}
\end{align}
\end{itemize}
\end{assumption}

\begin{assumption}[Teacher forcing invariance]
\label{assump:A3}
Filtering removes token losses but does not change the context law $Q(c)$ in the forward pass. 
Consequently, $F_\lambda(\theta)$ in \eqref{eq:Fisher} is unaffected by filtering.
\end{assumption}

\begin{assumption}[Local smoothness of $\Loss_\star$]
\label{assump:A4}
For each $\theta$ there exist $r>0$ and $L=L(\theta,r)>0$ such that, for all $u$ with $\|u\|_2\le r$,
\begin{equation}
\label{eq:local-smooth}
\Loss_\star(\theta+u)
\;\le\;
\Loss_\star(\theta) + \nabla\Loss_\star(\theta)^\top u + \tfrac{L}{2}\|u\|_2^2.
\end{equation}
\end{assumption}

\begin{assumption}[Score–norm control for high–confidence tokens]
\label{assump:A5}
Assume the logits $z_\theta(c)\in\mathbb{R}^K$ are coordinatewise $L_z$–Lipschitz in $\theta$, i.e.\ $\|\nabla_\theta z_\theta(c)_k\|_2\le L_z$ for all $k$, uniformly in $c$. 
If $F_\lambda(\theta)\succeq \mu I$ for some $\mu>0$, then for all $(c,t)$
\begin{equation}
\label{eq:score-bound}
\big\|\phi_\theta(c,t)\big\|_{F_\lambda^{-1}}
\;\le\; \frac{2L_z}{\sqrt{\mu}}\,(1 - p_\theta(t\mid c)).
\end{equation}
\end{assumption}

\begin{assumption}[Weak incoherence of noise]
\label{assump:A6}
Let
\[
g_{\mathrm{core}}(\theta)\;\defeq\; \mathbb{E}_{(c,t)\sim p_{\rm core}}\![\phi_\theta(c,t)],
\qquad
g_{\mathrm{noise}}(\theta)\;\defeq\; \mathbb{E}_{(c,t)\sim p_{\rm noise}}\![\phi_\theta(c,t)].
\]
There exists $\zeta_M\in[0,1)$ such that for all $\theta$
\begin{equation}
\label{eq:coherenceM}
\langle g_{\mathrm{core}}(\theta),\, g_{\mathrm{noise}}(\theta)\rangle_{M^{-1}}
\;\le\;
\zeta_M \,\big\|g_{\mathrm{core}}(\theta)\big\|_{M^{-1}}^2.
\end{equation}
\end{assumption}

\begin{remark}[Proof of Assumption~\ref{assump:A5}]
\label{rem:A5-proof}
For softmax $p_\theta(t\mid c)=\mathrm{softmax}(z_\theta(c))_t$, the score decomposes as
$\phi_\theta(c,t)=\sum_{k=1}^K(\mathbb{I}\{k=t\}-p_\theta(k\mid c))\,\nabla_\theta z_\theta(c)_k$.
By triangle inequality and $\|\nabla_\theta z_\theta(c)_k\|_2\le L_z$,
\[
\big\|\phi_\theta(c,t)\big\|_2
\;\le\; L_z \sum_{k=1}^K|\mathbb{I}\{k=t\}-p_\theta(k\mid c)|
\;=\; 2L_z \big(1-p_\theta(t\mid c)\big).
\]
If $F_\lambda(\theta)\succeq \mu I$, then $\|v\|_{F_\lambda^{-1}}\le \mu^{-1/2}\|v\|_2$ for all $v$, which yields \eqref{eq:score-bound}.
\end{remark}

\subsection{Distributions used by SGD and gradient identities}

Let $\mathsf{D}$ be a distribution on $(c,t)$; write
\[
g(\theta;\mathsf{D}) \;\defeq\; \mathbb{E}_{(c,t)\sim \mathsf{D}}\![\phi_\theta(c,t)].
\]
We denote by $p_{\rm train}$ the unfiltered training distribution \eqref{eq:mixture}, and by $p_{\rm fil}$ the filtered training distribution induced by keeping only tokens with $Z=1$ and renormalizing.

\begin{lemma}[Filtering–induced mixture and gradient identities]
\label{lem:mixture-grad}
Let $a\!\defeq\!1-\varepsilon$, $b\!\defeq\!\varepsilon$ and 
\begin{equation}
\label{eq:Zfil}
Z_{\rm fil} \;\defeq\; a(1-\alpha) + b\,\beta.
\end{equation}
Then
\begin{align}
g(\theta; p_{\rm train}) &= a\,g_{\rm core}(\theta) + b\,g_{\rm noise}(\theta),
\label{eq:g-train}
\\
\text{\emph{(Strong MAR)}:}\quad
g(\theta; p_{\rm fil}) &= \frac{a(1-\alpha)\,g_{\rm core}(\theta) + b\beta\,g_{\rm noise}(\theta)}{Z_{\rm fil}},
\label{eq:g-fil-strong}
\\
\text{\emph{(Weak bias)}:}\quad
g(\theta; p_{\rm fil}) &= \frac{a(1-\alpha)\,g_{\rm core}^{\rm sel}(\theta) + b\beta\,g_{\rm noise}^{\rm sel}(\theta)}{Z_{\rm fil}},
\label{eq:g-fil-weak}
\end{align}
where $g_{\rm core}^{\rm sel}\!\defeq\! \mathbb{E}[\phi_\theta\mid G\!=\!\mathrm{core},Z\!=\!1]$ and $g_{\rm noise}^{\rm sel}$ is defined analogously. Moreover,
\begin{equation}
\label{eq:ideal-grad-core}
\nabla_\theta \Loss_\star(\theta) \;=\; -\,g_{\rm core}(\theta).
\end{equation}
\end{lemma}

\begin{proof}
\eqref{eq:g-train} is linearity of expectation applied to \eqref{eq:mixture}. 
For \eqref{eq:g-fil-strong}, under $Z\perp (c,t)\mid G$,
\[
\mathbb{E}\big[\phi_\theta(c,t)\,\mathbb{I}\{Z=1\}\mid G=\mathrm{core}\big]
=(1-\alpha)\,\mathbb{E}[\phi_\theta(c,t)\mid G=\mathrm{core}]
=(1-\alpha)\,g_{\rm core},
\]
and similarly for noise; after renormalization by $Z_{\rm fil}$ we obtain \eqref{eq:g-fil-strong}. 
If independence is replaced by the weak bias bounds \eqref{eq:weak-bias-core}–\eqref{eq:weak-bias-noise}, the same computation yields \eqref{eq:g-fil-weak}. 
Finally, \eqref{eq:ideal-grad-core} is the usual score–matching identity since $p_\star$ does not depend on $\theta$.
\end{proof}

\subsection{Alignment and its improvement under filtering (general SPD $M$)}
\label{app:alignment}

Given an SPD $M(\theta)\succ0$, define the \emph{$M$–alignment} of any preconditioned direction $\widetilde g$ with the ideal direction by
\begin{equation}
\label{eq:alignment-def}
\mathcal{A}_M(\theta;\widetilde g) \;\defeq\; \big\langle \widetilde{\nabla}_M \Loss_\star(\theta),\, \widetilde g\big\rangle_M.
\end{equation}
In particular, if we use the $M$–preconditioned gradient of a distribution $\mathsf{D}$, i.e.\ $\widetilde g(\theta;\mathsf{D}) \defeq -M(\theta)^{-1} g(\theta;\mathsf{D})$, then by \eqref{eq:ideal-grad-core} and \eqref{eq:alignment-def}
\begin{equation}
\label{eq:alignment-forms}
\mathcal{A}_M(\theta;\widetilde g(\cdot;\mathsf{D}))
= \big\langle -M^{-1}g_{\rm core},\, -M^{-1}g(\cdot;\mathsf{D})\big\rangle_M
= g_{\rm core}^\top M^{-1} g(\theta;\mathsf{D}).
\end{equation}

\begin{lemma}[Explicit alignments under unfiltered and filtered sampling]
\label{lem:align-explicit}
Under Assumptions~\ref{assump:A1}–\ref{assump:A3}, for $a=1-\varepsilon$, $b=\varepsilon$, and $Z_{\rm fil}$ as in \eqref{eq:Zfil}, we have
\begin{align}
\mathcal{A}^{\rm train}_M(\theta) &\;\defeq\; \mathcal{A}_M\big(\theta;\widetilde g(\cdot;p_{\rm train})\big) 
\,=\, a\,\|g_{\rm core}\|_{M^{-1}}^2 \;+\; b\,\langle g_{\rm core}, g_{\rm noise}\rangle_{M^{-1}},
\label{eq:align-train}
\\
\text{\emph{(Strong MAR)}:}\quad
\mathcal{A}^{\rm fil}_M(\theta) 
&\;\defeq\; \mathcal{A}_M\big(\theta;\widetilde g(\cdot;p_{\rm fil})\big) 
\,=\, \frac{a(1-\alpha)}{Z_{\rm fil}}\,\|g_{\rm core}\|_{M^{-1}}^2 
\;+\; \frac{b\beta}{Z_{\rm fil}}\,\langle g_{\rm core}, g_{\rm noise}\rangle_{M^{-1}}.
\label{eq:align-fil-strong}
\end{align}
Under the weak–bias alternative in Assumption~\ref{assump:A2}, the same formulas hold with the replacements 
$g_{\rm core}\mapsto g_{\rm core}^{\rm sel}$ and $g_{\rm noise}\mapsto g_{\rm noise}^{\rm sel}$ in the numerators.
\end{lemma}

\begin{proof}
Plug \eqref{eq:g-train} and \eqref{eq:g-fil-strong} into \eqref{eq:alignment-forms}.
\end{proof}

\begin{theorem}[Filtering strictly improves $M$–alignment (strong MAR case)]
\label{thm:alignment-gain-M}
Under Assumptions~\ref{assump:A1}–\ref{assump:A3} and \ref{assump:A6}, with $a=1-\varepsilon$, $b=\varepsilon$, $Z_{\rm fil}=a(1-\alpha)+b\beta$, the improvement in alignment is
\begin{equation}
\label{eq:alignment-gain-exact-M}
\mathcal{A}^{\rm fil}_M(\theta) - \mathcal{A}^{\rm train}_M(\theta)
\;=\;
\frac{ab\,(1-\alpha-\beta)}{Z_{\rm fil}}
\Big(\,\|g_{\rm core}(\theta)\|_{M^{-1}}^2 \;-\; \langle g_{\rm core}(\theta), g_{\rm noise}(\theta)\rangle_{M^{-1}}\,\Big).
\end{equation}
In particular, if $\alpha+\beta<1$ and \eqref{eq:coherenceM} holds with $\zeta_M<1$, then
\begin{equation}
\label{eq:alignment-gain-lb-M}
\mathcal{A}^{\rm fil}_M(\theta) - \mathcal{A}^{\rm train}_M(\theta)
\;\ge\;
\frac{ab\,(1-\alpha-\beta)\,(1-\zeta_M)}{Z_{\rm fil}}\;\big\|g_{\rm core}(\theta)\big\|_{M^{-1}}^2
\;>\; 0.
\end{equation}
\end{theorem}

\begin{proof}
Subtract \eqref{eq:align-train} from \eqref{eq:align-fil-strong} and note
\[
\frac{a(1-\alpha)}{Z_{\rm fil}}-a \;=\; \frac{ab(1-\alpha-\beta)}{Z_{\rm fil}} \, - \, \frac{b\beta}{Z_{\rm fil}} + b,
\quad
\frac{b\beta}{Z_{\rm fil}}-b \;=\; -\,\frac{ab(1-\alpha-\beta)}{Z_{\rm fil}}.
\]
This simplifies the difference to \eqref{eq:alignment-gain-exact-M}. The lower bound \eqref{eq:alignment-gain-lb-M} follows from \eqref{eq:coherenceM}.
\end{proof}

\begin{theorem}[Robust alignment gain under bounded selection bias]
\label{thm:alignment-gain-weak}
Under Assumptions~\ref{assump:A1}–\ref{assump:A3}, \ref{assump:A6}, and the weak–bias bounds \eqref{eq:weak-bias-core}–\eqref{eq:weak-bias-noise}, the alignment improvement satisfies
\begin{equation}
\label{eq:alignment-gain-weak}
\mathcal{A}^{\rm fil}_M(\theta) - \mathcal{A}^{\rm train}_M(\theta)
\;\ge\;
\frac{ab\,(1-\alpha-\beta)\,(1-\zeta_M)}{Z_{\rm fil}}\;\big\|g_{\rm core}\big\|_{M^{-1}}^2
\;-\;
\frac{a(1-\alpha)\rho_{\rm c} + b\beta \rho_{\rm n}}{Z_{\rm fil}}\;\big\|g_{\rm core}\big\|_{M^{-1}}^2.
\end{equation}
Hence the net gain remains positive whenever
\[
ab(1-\alpha-\beta)(1-\zeta_M) \;>\; a(1-\alpha)\rho_{\rm c} + b\beta \rho_{\rm n}.
\]
\end{theorem}

\begin{proof}
Repeat the proof of Theorem~\ref{thm:alignment-gain-M} using \eqref{eq:g-fil-weak}. Add and subtract the strong–MAR numerators, then apply \eqref{eq:weak-bias-core}–\eqref{eq:weak-bias-noise} and Cauchy–Schwarz in $\langle\cdot,\cdot\rangle_{M^{-1}}$ to bound the deviation terms by $\rho_{\rm c}\|g_{\rm core}\|_{M^{-1}}^2$ and $\rho_{\rm n}\|g_{\rm core}\|_{M^{-1}}^2$, respectively.
\end{proof}

\begin{remark}[Edge cases]
If $\varepsilon=0$ (no noise) or $\alpha+\beta=1$ (random selector), then \eqref{eq:alignment-gain-exact-M} yields $\mathcal{A}^{\rm fil}_M=\mathcal{A}^{\rm train}_M$, i.e.\ no gain; if $\zeta_M=1$ (perfect coherent noise), the lower bound \eqref{eq:alignment-gain-lb-M} is zero.
\end{remark}

\subsection{One–step decrease in $\Loss_\star$: general preconditioners and Euclidean SGD}
\label{app:one-step}

\begin{lemma}[One–step decrease for $M$–preconditioned steps (local version)]
\label{lem:one-step-M}
Let $\widetilde g$ be any direction and consider the update $\theta^+=\theta - \eta\,\widetilde g$. 
Under Assumption~\ref{assump:A4} with radius $r>0$, if $\eta\|\widetilde g\|_2\le r$, then
\begin{equation}
\label{eq:one-step-M}
\Loss_\star(\theta^+) 
\;\le\; 
\Loss_\star(\theta) \;-\; \eta\,\langle \widetilde{\nabla}_M \Loss_\star(\theta), \,\widetilde g\rangle_M 
\;+\; \tfrac{L}{2}\,\eta^2\,\|\widetilde g\|_2^2.
\end{equation}
In particular, for $\widetilde g(\cdot;\mathsf{D}) \!=\! -M^{-1}g(\cdot;\mathsf{D})$ we have
\[
\Loss_\star(\theta^+) 
\;\le\;
\Loss_\star(\theta) \;-\; \eta\,\mathcal{A}_M\big(\theta;\widetilde g(\cdot;\mathsf{D})\big)
\;+\; \tfrac{L}{2}\,\eta^2\,\|\widetilde g(\cdot;\mathsf{D})\|_2^2.
\]
\end{lemma}

\begin{proof}
By \eqref{eq:local-smooth},
\[
\Loss_\star(\theta^+) \le \Loss_\star(\theta) + \nabla\Loss_\star(\theta)^\top(\theta^+ - \theta) + \tfrac{L}{2}\|\theta^+ - \theta\|_2^2.
\]
With $\theta^+ - \theta = -\eta \widetilde g$ and $\widetilde{\nabla}_M\Loss_\star = M^{-1}\nabla\Loss_\star$, we obtain
$\nabla\Loss_\star^\top(-\eta\widetilde g)=-\eta\,\nabla\Loss_\star^\top M M^{-1}\widetilde g
= -\eta\,\langle \widetilde{\nabla}_M\Loss_\star, \widetilde g\rangle_M$.
\end{proof}

\begin{corollary}[Filtered vs.\ unfiltered one–step improvement]
\label{cor:one-step-diff}
Consider the two updates with common step size $\eta>0$:
\[
\theta^+_{\rm fil}=\theta - \eta\,\widetilde g(\cdot;p_{\rm fil}),
\qquad
\theta^+_{\rm train}=\theta - \eta\,\widetilde g(\cdot;p_{\rm train}).
\]
If $\eta\max\{\|\widetilde g(\cdot;p_{\rm fil})\|_2,\|\widetilde g(\cdot;p_{\rm train})\|_2\}\le r$, then
\begin{equation}
\label{eq:one-step-difference}
\Loss_\star(\theta^+_{\rm fil}) - \Loss_\star(\theta^+_{\rm train})
\;\le\; -\,\eta\Big(\mathcal{A}^{\rm fil}_M(\theta) - \mathcal{A}^{\rm train}_M(\theta)\Big)
\;+\; \frac{L}{2}\,\eta^2\Big(\|\widetilde g(\cdot;p_{\rm fil})\|_2^2 - \|\widetilde g(\cdot;p_{\rm train})\|_2^2\Big).
\end{equation}
In particular, whenever
\begin{equation}
\label{eq:eta-small}
0<\eta \le 
\min\!\left\{
\frac{2\big(\mathcal{A}^{\rm fil}_M(\theta) - \mathcal{A}^{\rm train}_M(\theta)\big)}{L\big(\|\widetilde g(\cdot;p_{\rm fil})\|_2^2 + \|\widetilde g(\cdot;p_{\rm train})\|_2^2\big)}\;,\;
\frac{r}{\max\{\|\widetilde g(\cdot;p_{\rm fil})\|_2,\|\widetilde g(\cdot;p_{\rm train})\|_2\}}
\right\},
\end{equation}
we have $\Loss_\star(\theta^+_{\rm fil}) \le \Loss_\star(\theta^+_{\rm train})$, with strict inequality if the alignment gain in \eqref{eq:alignment-gain-lb-M} (or \eqref{eq:alignment-gain-weak}) is strict.
\end{corollary}

\begin{remark}[Euclidean SGD and Adam]
Lemma~\ref{lem:one-step-M} and Corollary~\ref{cor:one-step-diff} hold for \emph{any} SPD $M$; taking $M=I$ yields the standard Euclidean form
\[
\Loss_\star(\theta-\eta g)\le \Loss_\star(\theta)-\eta\langle\nabla\Loss_\star,g\rangle+\tfrac{L}{2}\eta^2\|g\|_2^2
\]
and the same alignment improvement conclusions with $M^{-1}$–inner products replaced by Euclidean inner products. Taking $M=F_\lambda$ recovers the damped natural gradient geometry.
\end{remark}

\subsection{High–confidence (\textsc{KN}) tokens: negligible alignment contribution}
\label{app:KN}

We show that high–confidence tokens (\textsc{KN}) have vanishing contribution to alignment in the Fisher geometry, with a \emph{quantitative} bound relying on Assumption~\ref{assump:A5}.

\begin{lemma}[\textsc{KN} tokens have vanishing Fisher contribution]
\label{lem:KN-small}
Let $\mathcal{S}_{\rm KN}\subseteq \mathrm{Supp}(p_{\rm train})$ be any measurable set such that $p_\theta(t\mid c)\ge 1-\delta$ for all $(c,t)\in\mathcal{S}_{\rm KN}$, with $\delta\in(0,1)$. 
Then, under Assumption~\ref{assump:A5},
\begin{equation}
\label{eq:KN-signal}
\Big\| \mathbb{E}_{(c,t)\sim p_{\rm train}}\!\big[\phi_\theta(c,t)\,\mathbb{I}\{(c,t)\in\mathcal{S}_{\rm KN}\}\big] \Big\|_{F_\lambda^{-1}}
\;\le\; \frac{2L_z}{\sqrt{\mu}}\,\delta\,p_{\rm train}(\mathcal{S}_{\rm KN}).
\end{equation}
Consequently, using Cauchy–Schwarz in $\langle\cdot,\cdot\rangle_{F_\lambda^{-1}}$,
\begin{equation}
\label{eq:KN-align-impact}
\Big|\,\mathcal{A}_{F_\lambda}\big(\theta;\widetilde g(\cdot;p_{\rm train})\big) - \mathcal{A}_{F_\lambda}\big(\theta;\widetilde g(\cdot;p_{\rm train}\setminus\mathcal{S}_{\rm KN})\big)\,\Big|
\;\le\; \frac{2L_z}{\sqrt{\mu}}\,\delta\,p_{\rm train}(\mathcal{S}_{\rm KN})\; \big\|g_{\rm core}(\theta)\big\|_{F_\lambda^{-1}}.
\end{equation}
\end{lemma}

\begin{proof}
By Jensen and \eqref{eq:score-bound},
\[
\Big\| \mathbb{E}[\phi\,\mathbb{I}_{\mathcal{S}_{\rm KN}}] \Big\|_{F_\lambda^{-1}} 
\;\le\; \mathbb{E}\big[\|\phi\|_{F_\lambda^{-1}}\mathbb{I}_{\mathcal{S}_{\rm KN}}\big]
\;\le\; \frac{2L_z}{\sqrt{\mu}}\,\mathbb{E}\big[(1-p_\theta(t\mid c))\mathbb{I}_{\mathcal{S}_{\rm KN}}\big]
\;\le\; \frac{2L_z}{\sqrt{\mu}}\,\delta\,p_{\rm train}(\mathcal{S}_{\rm KN}).
\]
The alignment bound follows from \eqref{eq:alignment-forms} and Cauchy–Schwarz.
\end{proof}

\begin{remark}[From Fisher to general $M$]
If $M\succ0$ is any SPD matrix, then for all $v$,
$\|v\|_{M^{-1}}\le \|M^{-1/2}F_\lambda^{1/2}\|_{\mathrm{op}}\cdot \|v\|_{F_\lambda^{-1}}$.
Thus Lemma~\ref{lem:KN-small} transfers to $\langle\cdot,\cdot\rangle_{M^{-1}}$ up to the condition–number factor $\|M^{-1/2}F_\lambda^{1/2}\|_{\mathrm{op}}$.
\end{remark}

\subsection{Combining \textsc{TR}/\textsc{RI} with \textsc{KN}}
\label{app:net}

Let $\mathcal{S}_{\mathrm{TR}\downarrow}\cup\mathcal{S}_{\mathrm{RI}\downarrow}$ denote the subset labeled \emph{low relevance/importance} by the selector (mass $\varepsilon$ under $p_{\rm train}$), and let $\mathcal{S}_{\rm KN\downarrow}$ denote the subset of \emph{high–confidence} tokens with $p_\theta(t\mid c)\ge 1-\delta$.

Combining Theorems~\ref{thm:alignment-gain-M}–\ref{thm:alignment-gain-weak} with Lemma~\ref{lem:KN-small} (for $M=F_\lambda$), the net alignment improvement from filtering $\mathcal{S}_{\mathrm{TR}\downarrow}\cup\mathcal{S}_{\mathrm{RI}\downarrow}\cup\mathcal{S}_{\rm KN\downarrow}$ satisfies
\begin{align}
\mathcal{A}^{\rm fil}_{F_\lambda}(\theta) - \mathcal{A}^{\rm train}_{F_\lambda}(\theta)
&\;\ge\; 
\frac{(1-\varepsilon)\varepsilon(1-\alpha-\beta)(1-\zeta_{F_\lambda})}{Z_{\rm fil}}
\|g_{\rm core}(\theta)\|_{F_\lambda^{-1}}^2 \nonumber \\[-1mm]
&\hspace{3.1em}
- \frac{2L_z}{\sqrt{\mu}}\,\delta\,p_{\rm train}(\mathcal{S}_{\rm KN\downarrow})\,
\|g_{\rm core}(\theta)\|_{F_\lambda^{-1}} \nonumber \\[-1mm]
&\hspace{3.1em}
- \frac{(1-\varepsilon)(1-\alpha)\rho_{\rm c} + \varepsilon\beta \rho_{\rm n}}{Z_{\rm fil}}
\|g_{\rm core}(\theta)\|_{F_\lambda^{-1}}^2
\end{align}

where the last (bias) term is absent under strong MAR. 
For sufficiently small $\delta$ and small selection bias $(\rho_{\rm c},\rho_{\rm n})$, the positive \textsc{TR}/\textsc{RI} term dominates and the one–step KL improvement of the filtered update is strictly larger (Corollary~\ref{cor:one-step-diff}).

\subsection{Consequences for the modified loss $\mathcal{L}_{\mathrm{F}}$}
\label{app:LF}

Under teacher forcing (Assumption~\ref{assump:A3}), excluding tokens identified by the selector from the token sum exactly replaces the expectation with respect to $p_{\rm train}$ by the renormalized expectation with respect to $p_{\rm fil}$ in the expression for the stochastic gradient (or its $M$–preconditioned variant). 
Therefore Theorems~\ref{thm:alignment-gain-M}–\ref{thm:alignment-gain-weak} and Corollary~\ref{cor:one-step-diff} apply directly to the SGD (or $M$–preconditioned SGD) dynamics for $\mathcal{L}_{\mathrm{F}}$.

\subsection{Additional remarks and caveats}
\label{app:remarks}

\begin{enumerate}[leftmargin=1.8em]
\item \textbf{Estimator independence \& selection bias.}
Equation~\eqref{eq:g-fil-strong} relies on the strong MAR–within–component assumption in Assumption~\ref{assump:A2}. 
When MAR fails, \eqref{eq:g-fil-weak} and Theorem~\ref{thm:alignment-gain-weak} quantify the degradation through $(\rho_{\rm c},\rho_{\rm n})$.

\item \textbf{Geometry and optimizers.} 
All alignment statements are given for an \emph{arbitrary} SPD preconditioner $M$. 
The Fisher choice $M=F_\lambda$ confers parameterization invariance (up to damping), while $M=I$ matches Euclidean SGD; Adam/K–FAC correspond to other choices of $M$. 
Corollary~\ref{cor:one-step-diff} therefore bridges theory and practice without further assumptions.

\item \textbf{Smoothness is local.}
We use local smoothness (Assumption~\ref{assump:A4}) within a radius $r$ around $\theta$ to control the second–order term. 
The step–size condition \eqref{eq:eta-small} ensures $\theta^+$ remains in this local region.

\item \textbf{On variance claims.}
In general, filtering with \emph{renormalization} does not guarantee a universal reduction in gradient covariance; the effect is data– and selector–dependent.
Our theoretical results deliberately refrain from claiming such a reduction. 
Empirically, we observe reduced gradient dispersion when removing \textsc{KN} tokens; a rigorous sufficient condition (e.g.\ stratified second–moment dominance) leads to provable variance reduction, but this is orthogonal to the alignment results developed here.

\item \textbf{Parameter estimation.}
The selector ROC $(\alpha,\beta)$ and the mixture mass $\varepsilon$ can be estimated on a small annotated split;
$\zeta_M$ can be estimated from held–out data via $\langle g_{\rm core}, g_{\rm noise}\rangle_{M^{-1}}/\|g_{\rm core}\|_{M^{-1}}^2$. 
These estimators make the alignment gain \eqref{eq:alignment-gain-exact-M} empirically checkable.

\item \textbf{Notation.} 
We write $Z_{\rm fil}$ for the filtering normalizer \eqref{eq:Zfil} to avoid confusion with the Fisher matrix $F_\lambda$.
\end{enumerate}

\subsection{Self–contained proofs of key auxiliary facts}

\begin{lemma}[Ideal gradient equals negative core expectation]
\label{lem:ideal-grad-proof}
Under teacher forcing with $p_{\rm core}\equiv p_\star$, we have $\nabla_\theta \Loss_\star(\theta) = -\,g_{\rm core}(\theta)$.
\end{lemma}

\begin{proof}
By definition,
$\Loss_\star(\theta)=\E_{(c,t)\sim p_\star}[-\log p_\theta(t\mid c)]$,
so $\nabla_\theta \Loss_\star(\theta) = -\,\E_{(c,t)\sim p_\star}[\nabla_\theta \log p_\theta(t\mid c)] = -\,g_{\rm core}(\theta)$.
\end{proof}

\begin{lemma}[Score–norm bound in Fisher geometry]
\label{lem:score-Fisher}
Under Assumption~\ref{assump:A5}, for all $(c,t)$, $\|\phi_\theta(c,t)\|_{F_\lambda^{-1}}\le \frac{2L_z}{\sqrt{\mu}}\,(1-p_\theta(t\mid c))$.
\end{lemma}

\begin{proof}
Shown in Remark~\ref{rem:A5-proof}.
\end{proof}

\begin{lemma}[Alignment identity]
\label{lem:alignment-identity}
For any SPD $M\succ0$, $\langle \widetilde{\nabla}_M \Loss_\star, \,-M^{-1}g(\cdot;\mathsf{D})\rangle_M 
= g_{\rm core}^\top M^{-1} g(\cdot;\mathsf{D})$.
\end{lemma}

\begin{proof}
$\langle M^{-1}\nabla\Loss_\star,\, -M^{-1}g\rangle_M = -\,\nabla\Loss_\star^\top M^{-1}g = g_{\rm core}^\top M^{-1} g$ by Lemma~\ref{lem:ideal-grad-proof}.
\end{proof}

\subsection*{Summary of the strengthened guarantee}
Under Assumptions~\ref{assump:A0}–\ref{assump:A6}, token filtering induces a renormalized training distribution $p_{\rm fil}$ such that, for any SPD preconditioner $M$,
\[
\mathcal{A}^{\rm fil}_M(\theta)-\mathcal{A}^{\rm train}_M(\theta)
=
\frac{(1-\varepsilon)\varepsilon\,(1-\alpha-\beta)}{Z_{\rm fil}}
\Big(\|g_{\rm core}\|_{M^{-1}}^2-\langle g_{\rm core},g_{\rm noise}\rangle_{M^{-1}}\Big),
\]
and hence, under the incoherence condition $\zeta_M<1$, the alignment gain is \emph{strictly positive} and explicitly lower–bounded by \eqref{eq:alignment-gain-lb-M}. 
By the local one–step bound (Corollary~\ref{cor:one-step-diff}), for sufficiently small step sizes, a single filtered update yields a strictly larger \emph{decrease} of the ideal risk $\Loss_\star$ than the unfiltered update. 
Finally, \textsc{KN} tokens (high confidence under the base model) have vanishing Fisher contribution with an $O(\delta)$ bound (Lemma~\ref{lem:KN-small}), so removing them incurs negligible alignment cost while empirically stabilizing optimization. 
All statements specialize to the Fisher geometry ($M=F_\lambda$) for parameterization–invariant conclusions and to $M=I$ for Euclidean SGD/Adam–style updates.

\section{More Details}
\label{sec:more details}

\subsection{Otsu Method}
\label{sec:ostu method}

\partitle{Formulation}
For a histogram with gray levels in the range $[0, L-1]$, let threshold $t$ divide the pixels into two classes $C_0$ ($0$ to $t$) and $C_1$ ($t+1$ to $L-1$). The between-class variance $\sigma_b^2(t)$ is defined as:

\[
\sigma_b^2(t) = \omega_0(t)\omega_1(t)\left[\mu_1(t) - \mu_0(t)\right]^2
\]

where:
\begin{itemize}
    \item $\omega_0(t) = \sum\limits_{i=0}^t p(i)$, $\omega_1(t) = \sum\limits_{i=t+1}^{L-1} p(i)$ are the class probability weights ($p(i)$ is the probability of gray level $i$)
    \item $\mu_0(t) = \frac{1}{\omega_0(t)}\sum\limits_{i=0}^t i p(i)$, $\mu_1(t) = \frac{1}{\omega_1(t)}\sum\limits_{i=t+1}^{L-1} i p(i)$ are the class means
\end{itemize}

The \textbf{optimal threshold $t^*$} is determined by maximizing the between-class variance:

\[
t^* = \arg\max_{0 \leq t < L} \sigma_b^2(t)
\]

\partitle{Multi-Otsu Extension}
For partitioning into $k$ classes requiring $k-1$ thresholds $t_1, t_2, \dots, t_{k-1}$, the objective function becomes:

\[
\sigma_b^2(t_1, \dots, t_{k-1}) = \sum_{m=0}^{k-1} \omega_m \left(\mu_m - \mu_T\right)^2
\]

where $\mu_T$ is the global mean. The optimal threshold combination is obtained by maximizing the total between-class variance.

\subsection{Baselines Setting}
\label{sec:baselines setting}
Here we describe all baseline methods in detail and explain their roles in our experiment.

1) \textbf{Regular LLM Implementations}:
   \begin{itemize}
   \item \textbf{CA}: This represents the performance of the original base model, ensuring that our fine-tuning is correct and effective.
   \item \textbf{Normal}: This gives the performance of regular fine-tuning, serving as a direct reference to observe the influence of not applying token-level noise filtering.
   \item \textbf{More Epochs ($\times$2 Epochs)}: This method uses double the training epochs for fine-tuning, eliminating bias caused by recording experimental results before complete convergence.
   \end{itemize}

2) \textbf{Data Enhancement Methods}:
   Since our \Name is essentially a dataset enhancement method that optimizes fine-tuning performance before training the model, we adopt two mainstream data enhancement methods as baselines: data augmentation (DA) and data filtering (DF).
   \begin{itemize}
   \item \textbf{DA}: We use the state-of-the-art data augmentation method, which leverages high-performance LLMs to optimize the dataset \citep{dai2025auggpt}. Specifically, we employ Claude-3.5-sonnet \citep{claude-3.5-sonnet} in our experiment.
   \item \textbf{DF}: Recent studies have shown that removing data items with high perplexity can improve training performance \citep{li2024superfiltering,ankner2024perplexed}. Thus, we filter out the top-5\% perplexity data items in this baseline experiment.
   \item \textbf{SLM}: This is a novel method that can effectively filter noisy tokens in pre-training tasks. Vanilla SLM relies on a high-quality dataset to train a reference model, which is impractical in fine-tuning tasks. Therefore, we train the reference model using a subset of the fine-tuning dataset to simulate the effect of SLM as closely as possible.
   \item \textbf{TC}: This is a fine-grained token-level data selection method for LLM supervised fine-tuning, dedicated to filtering uninformative tokens while preserving task-specific informative ones—effectively addressing token-level noise even in high-quality samples \citep{pang2025tokenclean}. In our experiment, we adopt TC’s Self-Evolving Cleaning strategy, with implementation tailored to the epoch-wise training process: Specifically, in each training epoch, we treat the model trained in the previous epoch as the reference model and the model to be updated in the current epoch as the base model. For every token in the training dataset, we calculate its quality score using the loss disparity between this base model and the reference model (a core scoring logic of Self-Evolving Cleaning). 

   \end{itemize}
   
For a fair experiment, we use the same validation dataset to select the best model for all baselines.

\begin{table}[t]
\centering
\caption{Hyperparameters for different models and tasks. Ma1 denotes math (gsm8k);Ma2 denotes math(NuminaMath-CoT for MATH-500 evaluatio); Fi denotes finance; Co denotes code; Me denotes medicine. Lr denotes learning rate, Ep denotes epoch number of fine-tuning, and MNT denotes max new tokens for model evaluation.}
\label{tab:hyperparams}
\begin{tabular}{lccccccc}
\toprule
Model & Task & Data scale & Lr & Lr\_LoRA & Ep & Ep\_LoRA & MNT \\ \midrule
1B & Ma1 & 500 & 9e-6 & - & 3 & - & 512 \\
1B & Ma2 & 500 & 9e-6 & - & 3 & - & 2048 \\
1B & Fi  & 500 & 9e-6 & - & 3 & - & 1024 \\
1B & Co & 500 & 9e-6 & - & 3 & - & 512 \\
1B & Me & 500 & 1e-5 & - & 3 & - & 1024 \\ \midrule
1.5B & Ma1 & 600 & 9e-6 & - & 3 & - & 512 \\
1.5B & Ma2 & 600 & 9e-6 & - & 3 & - & 2048 \\
1.5B & Fi  & 600 & 9e-6 & - & 3 & - & 1024 \\
1.5B & Co & 600 & 9e-6 & - & 3 & - & 512 \\
1.5B & Me & 600 & 1e-5 & - & 3 & - & 1024 \\ \midrule
3B & Ma1 & 1000 & 8e-6 & - & 3 & - & 512 \\
3B & Ma2& 1000 & 8e-6 & - & 3 & - & 2048 \\
3B & Fi  & 1000 & 8e-6 & - & 3 & - & 1024 \\
3B & Co & 1000 & 8e-6 & - & 3 & - & 512 \\
3B & Me & 1000 & 3e-6 & - & 3 & - & 1024 \\ \midrule
7B\&8B & Ma1 & 3000 & 5e-/6 & 5e-5 & 3 & 5 & 512 \\
7B\&8B & Ma2 & 3000 & 5e-6 & 5e-5 & 3 & 5 & 2048 \\
7B\&8B & Fi  & 3000 & 5e-6 & 5e-5 & 3 & 5 & 1024 \\
7B\&8B & Co & 3000 & 5e-6 & 5e-5 & 3 & 5 & 512 \\
7B\&8B & Me & 3000 & 2e-6 & 5e-5 & 3 & 5 & 1024 \\ \midrule
14B & Ma1 & 3000 & 3e-6 & 3e-5 & -- & 5 & 512 \\
14B & Ma2 & 3000 & 3e-6 & 3e-5 & -- & 5 & 2048 \\
14B & Fi & 3000 & 3e-6 & 3e-5 & -- & 5 & 1024 \\
14B & Co & 3000 & 3e-6 & 3e-5 & -- & 5 & 512 \\
14B & Me & 3000 & 3e-6 & 3e-5 & -- & 5 & 1024 \\ \bottomrule
\end{tabular}
\end{table}

\subsection{Hyperparameters Setting}
\label{sec:hyperparameters}

We carefully adjust the hyperparameter to make sure our fine-tuning process is correct and effective, and we finally select a common setting shown in Table~\ref{tab:hyperparams}. Besides, we set \texttt{torch\_dtype = bfloat16} and adopt default generation setting (\eg, temperature) in all of our experiments. We conduct our experiment on NVIDIA A100-SMX-80G, H20-96G and L20-48G.

\section{More experiment results}
\label{sec:more experiment results}

\begin{table}[t]
\centering
\caption{Extension of Main Experiment. We show the accuracy of LLMs across different fine-tuning methods. Best results are marked in \textbf{bold} and the second best results are marked with underline. }
\label{tab:main-extension}
\resizebox{\linewidth}{!}{\begin{tabular}{lrccccccccc}
\toprule
\multicolumn{11}{c}{\textbf{MATH:Fine-tuning models on NuminaMath-CoT and Evaluate models on MATH-500}} \\
\midrule
\textbf{Model} & $\mathbf{|\theta|}$ & \textbf{LoRA} & \textbf{CA} & \textbf{Normal} & $\mathbf{\times}$\textbf{2 Ep} &\textbf{DF} & \textbf{DA} & \textbf{SLM} & \textbf{TC} & \textbf{\Name} \\
\midrule
Llama-3.2 & 1B & $\times$ & 0.0 & 2.8 &3.0 & 1.5 &2.1  & \underline{3.0} & 2.9& \textbf{3.6} \\
Llama-3.2 & 3B & $\times$ & 0.8& 4.8 & 4.9& 5.2& 4.9 & 8.2 & \underline{8.4}&\textbf{8.6} \\
Llama-3.1 & 8B & $\times$ & 4.9& 19.8 & 19.8 & \textbf{23.6} & 19.2 &22.8 &22.3 & \underline{23.4} \\
Llama-3.1 & 8B & $\checkmark$ &4.9 &20.7 & 21.1&19.5 &18.8 &\underline{24.7} &24.2 & \textbf{28.7}\\
Mistral & 7B & $\times$ & 12.3 & 26.6 & 26.9 & 26.4 & 25.8 & \underline{27.3} & 27.1 & \textbf{30.8} \\
Mistral & 7B & $\checkmark$  & 12.3 & 26.4 & 26.8 & 29.1 & 27.3 & \textbf{33.4} & \underline{33.2} & 32.5\\
Deepseek-distilled-qwen & 1.5B & $\times$ & 33.4 & 50.6 & 50.8 & \underline{52.3} & 49.2 & 51.2 & 50.5 & \textbf{52.6}\\
Deepseek-distilled-qwen & 7B & $\times$ & 44.5 & 65.3 & 65.9 & 55.7 & 61.4 & 69.2 & \underline{69.5} & \textbf{70.3}\\
Deepseek-distilled-qwen & 7B & $\checkmark$  & 44.5 & 49.6 & 49.8 & 51.2 & 52.8 & \underline{55.3} & 53.2 & \textbf{57.4}\\
Deepseek-distilled-qwen & 14B & $\checkmark$  & 52.8 & 72.1 & 72.1 & 63.4 & 68.9 & \underline{75.3} & 74.6 & \textbf{77.2} \\
\midrule
Average & --& -- & 21.0 & 33.9 & 34.1 & 32.8 & 33.0 & \underline{37.0} & 36.6 & \textbf{38.3} \\
\midrule
\multicolumn{11}{c}{\textbf{Fine-tuning and evaluate models on fiqa}} \\
\midrule
\textbf{Model} & $\mathbf{|\theta|}$ & \textbf{LoRA} & \textbf{CA} & \textbf{Normal} & $\mathbf{\times}$\textbf{2 Ep} &\textbf{DF} & \textbf{DA} & \textbf{SLM} & \textbf{TC} & \textbf{\Name} \\
\midrule
Llama-3.2 & 1B & $\times$ & 0.0 & 0.4 & 0.4 & \textbf{0.6} & 0.3 & 0.5 & 0.5 & \underline{0.5}   \\
Llama-3.2 & 3B & $\times$ & 0.5 & 1.1 & 1.0 & 2.2 & 1.6 & 3.2 & \underline{3.4} & \textbf{3.5} \\
Llama-3.1 & 8B & $\times$  & 2.5 & 6.8 & 6.9 & 5.1 & \underline{8.2} & 7.9 & 7.6 & \textbf{8.4}\\
Llama-3.1 & 8B & $\checkmark$ & 2.5 & 5.2 & 5.4 & 4.3 & 4.8 & \textbf{8.9} & \underline{8.6} & 8.1 \\
Mistral & 7B & $\times$ & 8.7 & 13.3 & 13.3 & 12.4 & 14.1 & \underline{15.1} & 14.8 & \textbf{18.5}  \\
Mistral & 7B & $\checkmark$  & 8.7 & 12.8 & 12.9 & 13.6 & 12.9 & 14.2 & \underline{14.9} & \textbf{17.2}  \\ 
Deepseek-distilled-qwen & 1.5B & $\times$ & 11.9 & 18.4 & 18.7 & 15.7 & 19.2 & \underline{21.5} & 20.8 & \textbf{25.9} \\
Deepseek-distilled-qwen & 7B & $\times$ & 18.4 & 28.5 & 28.9 & 28.3 & 30.9 & 31.6 & \underline{31.9} & \textbf{36.1}\\
Deepseek-distilled-qwen & 7B & $\checkmark$ & 18.4 & 22.6 & 22.7 & 20.8 & \underline{23.9} & 22.8 & 22.1 &\textbf{29.3}  \\
Deepseek-distilled-qwen & 14B & $\checkmark$  & 26.3 & 35.8 & 36.2 & 31.5 & 34.1 & 40.7 & \underline{41.2} & \textbf{45.6} \\
\midrule
Average & --& -- & 9.8 & 14.5 & 14.6 & 13.5 & 15 & 16.6 & \underline{16.8} & \textbf{19.3} \\
\bottomrule
\end{tabular}}
\end{table}

\subsection{Extension of Main Experiment}
\label{sec:extension of main experiment}
Two experimental setups were implemented: (1) Models were fine-tuned on NuminaMath-CoT and subsequently evaluated on MATH-500; (2) Models were fine-tuned and evaluated end-to-end on the FIQA dataset. For space considerations, additional experimental results from one economics dataset are included in the appendix, where all configurations are detailed in Table~\ref{tab:hyperparams}.

The result are shown in Table~\ref{tab:main-extension}. In math task, \Name has average 4.6\% higher accuracy than normal fine-tuning and 1.3\% higher accuracy than the best baseline SLM. In finance task, \Name has average 4.8\% higher accuracy than normal fine-tuning and 2.5\% higher accuracy than the best baseline TC. The two tasks are more complex than the gsm8k and medicine tasks used in Section~\ref{sec:main experiment results}, and the final model accuracies are generally lower, especially for small-scale models. Therefore, the average advantage achieved by \Name is not as pronounced as in Section 4.2. However, \Name still achieves the best performance in 16 out of 20 cases and the second-best performance in 2 cases, which is sufficient to demonstrate the reliable advantage of the \Name method.

\subsection{Ablation Study about Threshold}
\label{sec:ablation study about threshold}
In the threshold ablation study, we explore how the three thresholds influence the percentage of filtered tokens and the final training result. This ablation study is conducted using Deepseek-distilled-qwen-1.5B on the GSM8K dataset.

\begin{figure}[t]
  \begin{subfigure}{0.32\linewidth}
    \includegraphics[width=\linewidth]{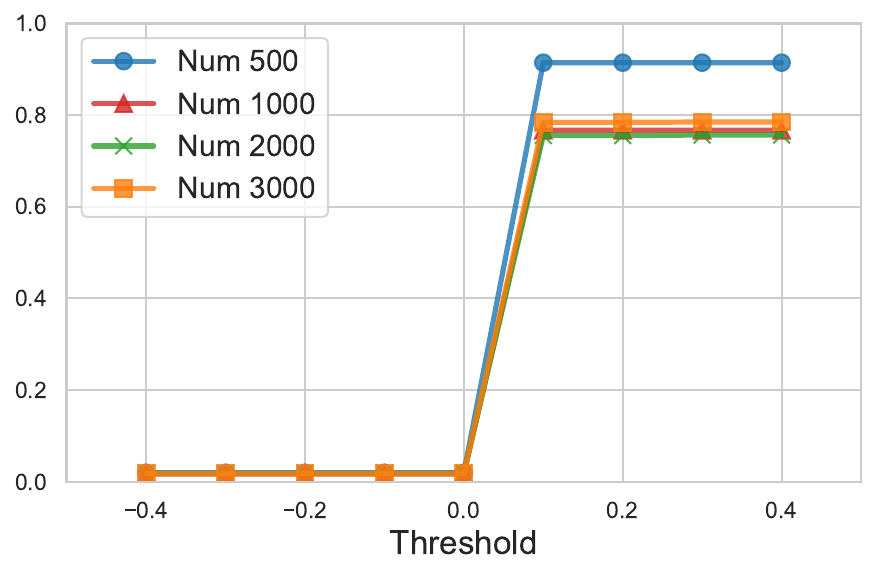}
    \caption{RI Filtering}
  \end{subfigure}
  \hfill
  \begin{subfigure}{0.32\linewidth}
    \includegraphics[width=\linewidth]{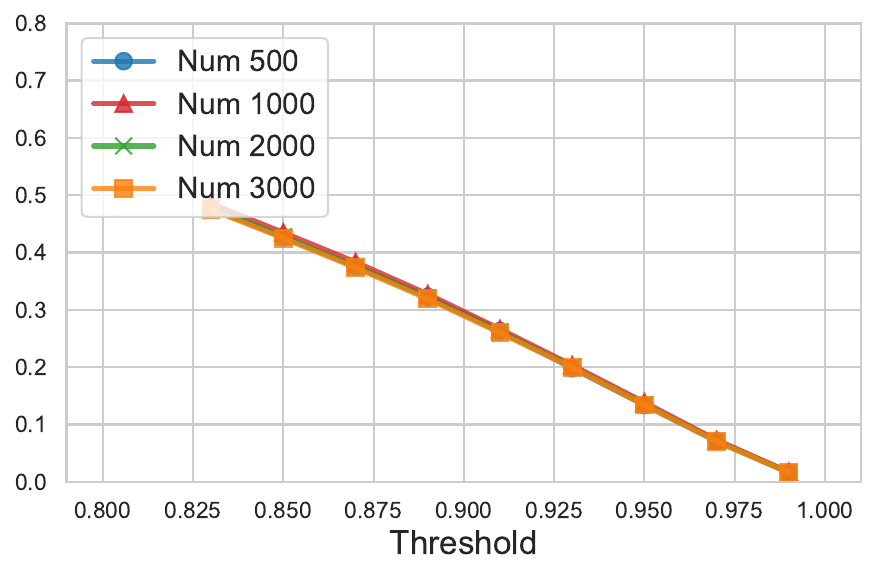}
    \caption{KN Filtering}
  \end{subfigure}
  \hfill
  \begin{subfigure}{0.32\linewidth}
    \includegraphics[width=\linewidth]{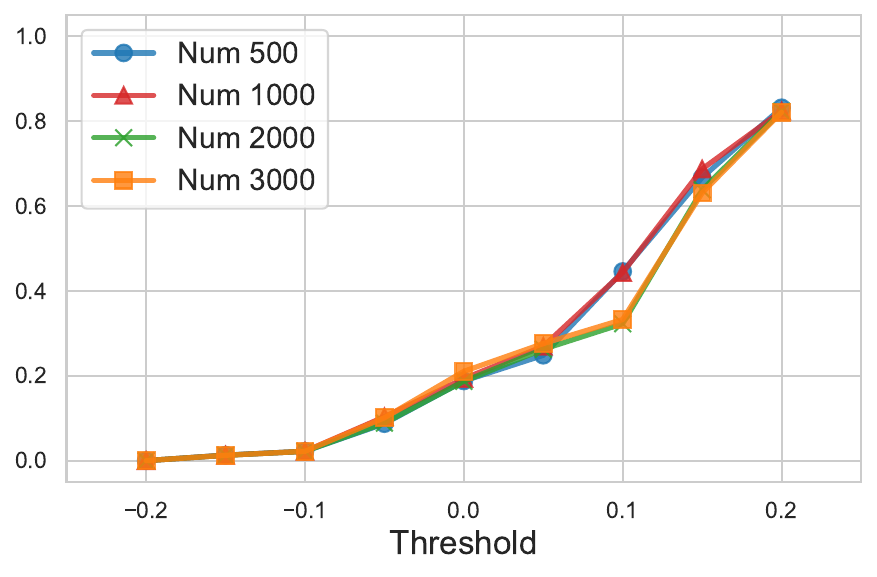}
    \caption{TR Filtering}
  \end{subfigure}

  \caption{Filtered token percentage under different threshold disturbance. In the ablation study of RI and TR, the value 0 on the horizontal axis represents the original threshold without any perturbation. In the ablation study of KN, the horizontal axis corresponds to different PCP values.}
  \label{fig:threshold ablation study}
  \vspace{-1em}
\end{figure}

\partitle{Filtered Tokens}
First, we investigate whether adjusting the XTF thresholds produces a noticeable effect on the total number of filtered tokens. Specifically, for RI and TR, we first apply normalization to the scores and then modify the thresholds obtained by XTF to compute the number of filtered tokens. For KN, we use different PCP values to perform token filtering. To evaluate the stability of the filtering effect, we compute the proportion of filtered tokens relative to all label tokens under four different training-set sizes.
%首先，我们探索调整XTF阈值是否会对被过滤的token总数产生明显影响。具体来说，对于RI和TR，我们先对分数进行Normalization，再对XTF得到的阈值进行修改，计算被过滤的token数量。对于KN，我们使用了不同的PCP进行token过滤。为了验证过滤效果的稳定性，我们分别在4种不同的训练集规模下，计算被过滤token相对于全部label token的占比。

The results are shown in Figure 6. Directly adjusting the threshold value produces a significant impact on the number of filtered tokens. For RI, a slight increase in the threshold causes a large number of tokens (>70\%) to be filtered. For KN, each 1\% change in PCP corresponds to approximately 2\% of all tokens, which exceeds 15\% of the number of tokens filtered under the 0.95 PCP threshold. For TR, both increasing and decreasing the threshold lead to substantial changes. It is noteworthy that the trends of threshold variation are highly similar across different training-set sizes, which indicates that the threshold-selection scheme based on score distributions is resistant to dataset scale.
%结果如图6所示，直接调整阈值value会对过滤token数量造成显著的影响。对%于RI，轻微地提高阈值就会导致大量token（>70%）被过滤；对于KN，每1%的PCP约对应总量2%的token，这超过了0.95PCP阈值下被过滤token数量的15%；对于TR，提高和降低阈值都会带来显著的变化。值得一提的是，阈值的变化趋势在不同规模训练数据下的是非常相似的，这说明XTF基于评分分布的阈值选取方案对于数据集规模是稳定的。

\begin{table}[t]
\centering
\caption{The ablation study on the rationality of the thresholds is conducted as follows: +3\%, -3\%, and -0\% correspond to adjustments in the number of tokens filtered under the respective attributes, based on the scores to decide whether a token should be filtered or not. The top five results are marked in the table.}
\label{tab:ablation study about thoreshold}
\resizebox{0.8\textwidth}{!}{
\begin{tabular}{c|ccccccccc}
\hline
 \multicolumn{10}{c}{} \\
\hline
RI(↓) & -3\% & -3\% & -3\% & -3\% & -3\% & -3\% & -3\% & -3\% & -3\% \\
KN(↓) & -3\% & -3\% & -3\% & -0\% & -0\% & -0\% & +3\% & +3\% & +3\% \\
TR(↓) & -3\% & -0\% & +3\% & -3\% & -0\% & +3\% & -3\% & -0\% & +3\% \\
ACC (\%) & 51.2 & 48.3 & 53.2 & 50.7 & 51.7 & 47.8 & 53.7 & 47.3 & 45.3 \\
\hline
Order & -- & -- & 5 & -- & -- & -- & 4 & -- & -- \\
\hline
 \multicolumn{10}{c}{} \\
\hline
RI(↓) & -0\% & -0\% & -0\% & -0\% & -0\% & -0\% & 0\% & 0\% & 0\% \\
KN(↓) & -3\% & -3\% & -3\% & -0\% & -0\% & -0\% & +3\% & +3\% & +3\% \\
TR(↓) & -3\% & -0\% & +3\% & -3\% & -0\% & +3\% & -3\% & -0\% & +3\% \\
ACC (\%) & 44.3 & 50.5 & 47.8 & 47.3 & 56.2 & 54.2 & 49.8 & 48.0 & 52.8 \\
\hline
Order & -- & -- & -- & -- & 1 & 3 & -- & -- & -- \\
\hline
 \multicolumn{10}{c}{} \\
\hline
RI(↓) & +3\% & +3\% & +3\% & +3\% & +3\% & +3\% & +3\% & +3\% & +3\% \\
KN(↓) & -3\% & -3\% & -3\% & -0\% & -0\% & -0\% & +3\% & +3\% & +3\% \\
TR(↓) & -3\% & -0\% & +3\% & -3\% & -0\% & +3\% & -3\% & -0\% & +3\% \\
ACC (\%) & 48.8 & 47.8 & 47.3 & 48.3 & 50.3 & 50.3 & 48.8 & 47.3 & 55.7 \\
\hline
Order & -- & -- & -- & -- & -- & -- & -- & -- & 2 \\
\hline
\end{tabular}}
\end{table}

\partitle{Final Training Result}
%由于直接调整阈值会对过滤token数量产生较大影响，我们采用更精细的过滤方案，即，直接按照分数高低扰动每个sentence被过滤token的百分比。
Since directly adjusting the threshold produces a large impact on the number of filtered tokens, we adopt a more fine-grained filtering scheme, namely, perturbing the percentage of filtered tokens in each sentence directly according to the ranking of the scores. We varied the number of filtered tokens along different attribute dimensions by ±3\% \footnote{A perturbation range of 3\% is appropriate for our setting because the number of filtered tokens in a sentence is relatively small. If a perturbation of 1 or 2\% is used, the filtered tokens in many sentences remain completely unchanged. A value of 3\% is therefore a reasonable choice based on the experimental setting and granularity.}. We systematically combined all possible threshold settings, resulting in a total of 27 cases.

As shown in Table~\ref{tab:ablation study about thoreshold}, our thresholds, which are computed based on statistical principles, consistently lead to the best fine-tuning performance across all cases. While we observe that certain alternative threshold settings could also yield competitive results, these settings do not conform to consistent patterns, such as linear or normal distributions. Therefore, we believe that it remains challenging to derive a reliable and generalizable threshold computation method through post-hoc adjustment of our statistically derived thresholds.

\section{Computational Overhead}
\label{sec:computational overhead}

\begin{table}[ht]
\centering
\caption{Memory Usage (peak/normal) for different models. Only the scoring process is considered.}
\label{tab:memory usage}
\resizebox{\textwidth}{!}{
\begin{tabular}{c|cccc}
\hline
\textbf{Model} & \textbf{Metric} & \textbf{Math} & \textbf{Code} & \textbf{Medicine} \\
\hline
Llama 1B & Mb (peak/normal) & 5362.2 / 4746.3 & 20149.0 / 4746.3 & 7617.2 / 4746.3 \\
\hline
Llama 3B & Mb (peak/normal) & 13093.7 / 12288.7 & 32465.6 / 12288.7 & 16035.3 / 12288.7 \\
\hline
Llama 8B & Mb (peak/normal) & 31740.7 / 30665.0 & 60487.2 / 30665.0 & 36525.3 / 30665.0 \\
\hline
Deepspeek 1.5B & Mb (peak/normal) & 7509.1 / 6811.8 & 18715.8 / 6811.8 & 7617.2 / 6811.8 \\
\hline
Deepspeek 7B & Mb (peak/normal) & 30439.1 / 29225.3 & 54760.8 / 29225.3 & 34626.5 / 29225.3 \\
\hline
Deepspeek 14B & Mb (peak/normal) & 58941.7 / 56375.2 & 92438.5 / 56375.2 & 68705.6 / 56375.2 \\
\hline
\end{tabular}}
\end{table}

\begin{table}[ht]
\centering
\caption{Time Usage for different models. Only the scoring process is considered.}
\label{tab:time usage}
\resizebox{0.65\textwidth}{!}{
\begin{tabular}{c|cccc}
\hline
\textbf{Model} & \textbf{Metric} & \textbf{Math} & \textbf{Code} & \textbf{Medicine} \\
\hline
Llama 1B & Time (s) & 134.62 & 150.42 & 253.68 \\
\hline
Llama 3B & Time (s) & 324.95 & 355.31 & 608.85 \\
\hline
Llama 8B & Time (s) & 1167.23 & 1541.27 & 2637.92 \\
\hline
Deepspeek 1.5B & Time (s) & 191.44 & 191.32 & 377.20 \\
\hline
Deepspeek 7B & Time (s) & 1179.11 & 1445.44 & 2276.91 \\
\hline
Deepspeek 14B & Time (s) & 3910.83 & 4132.56 & 4235.85 \\
\hline
\end{tabular}}
\end{table}

We evaluate the GPU memory usage and time consumption of \Name across 6 models from two mainstream LLM (Large Language Model) series: Llama and Deepseek. The experimental setting of this experiment is aligned with that of the main experiment in Table~\ref{tab:main}. Additionally, we provide basic information about the selected datasets.

The GPU memory consumption of \Name is at the inference-level, and remains close to the memory required for model loading (normal state), without causing a significant increase in peak memory usage. Specifically, the maximum sequence length of the Code dataset reaches 2806, which is much longer than that of the Math (529) and Medicine (1231) datasets; thus, the Code dataset exhibits a relatively higher memory peak. However, this peak is still acceptable—far below the memory usage of fine-tuning, which is generally four times that of the model loading memory usage.

Time consumption is mainly determined by the average sequence length of the dataset. The Medicine dataset has an average sequence length of 401, which is longer than that of the Math (181) and Code (236) datasets, resulting in the highest time cost. Notably, although increasing the model size leads to higher computational costs, the theoretical overhead of \Name remains significantly lower than that of the most competitive token-level baseline, i.e., SLM (Token-Level Supervised Language Model)—the latter requires training an additional reference model. Details of this experiment are provided in the appendix of the paper.

\partitle{Cost Analysis}
In data processing techniques that involve LLMs, the time and computational costs mainly arise from model inference and training, and the cost of training far exceeds that of inference. Therefore, the cost levels can be divided into two major categories: (1) training-level cost: the token-level baselines, including SLM and TC, require training a reference model to assist with filtering, and they need to perform one inference on the original model and one inference on the reference model to obtain token-level loss values before computing scores. (2) inference-level cost: two types of sample-level methods fall into this category. The DF baseline needs to compute the average perplexity for each sample, and DA requires multiple LLM inferences to perform data augmentation. Although our \Name method is a token-level method, it completes token filtering with only two inferences. The first inference produces the output logits and attention, from which the RI score, KN score, and the domain vector for TR are obtained. The second inference computes the TR score for each token. Arguably, existing token-level methods that rely on multiple rounds of training on a reference model cannot compete with \Name in terms of cost, and the cost of sample-level methods is also not substantially lower than that of \Name.
%LLM参与的数据处理技术，其时间开销/计算成本主要源于模型的推理和训练，而训练的开销又远超推理。与baselines相比，XTF在计算成本上具有优势。具体来说，开销等级可以分为两个大类：（1）训练级开销：Token-level的baselines，包含SLM和TC，均需要训练reference model来协助过滤，并且需要分别在原本模型和reference model上进行一次推理，得到token级的loss值，再计算分数。（2）推理级开销：两类sample-level的方法属于这一类。DF baseline需要对每个样本计算平均困惑度，DA则需要调用LLM多次推理进行数据增广。我们的XTF方法，虽然是token-level方法，但是只需要两次推理就可以完成token过滤。第一次推理得到输出的logits和attention，得到RI分数，KN分数以及TR的domain vector。第二次推理计算出每个token的TR分数。可以说，基于reference model多次训练的现有token-level方法，无法在成本上和XTF竞争，sample-level方法的成本也不会比XTF低多少。

\subsection{Distribution Figures}
\label{sec:distribution figures}

\begin{figure}[t]
   \begin{subfigure}{0.32\linewidth}
    \includegraphics[width=\linewidth]{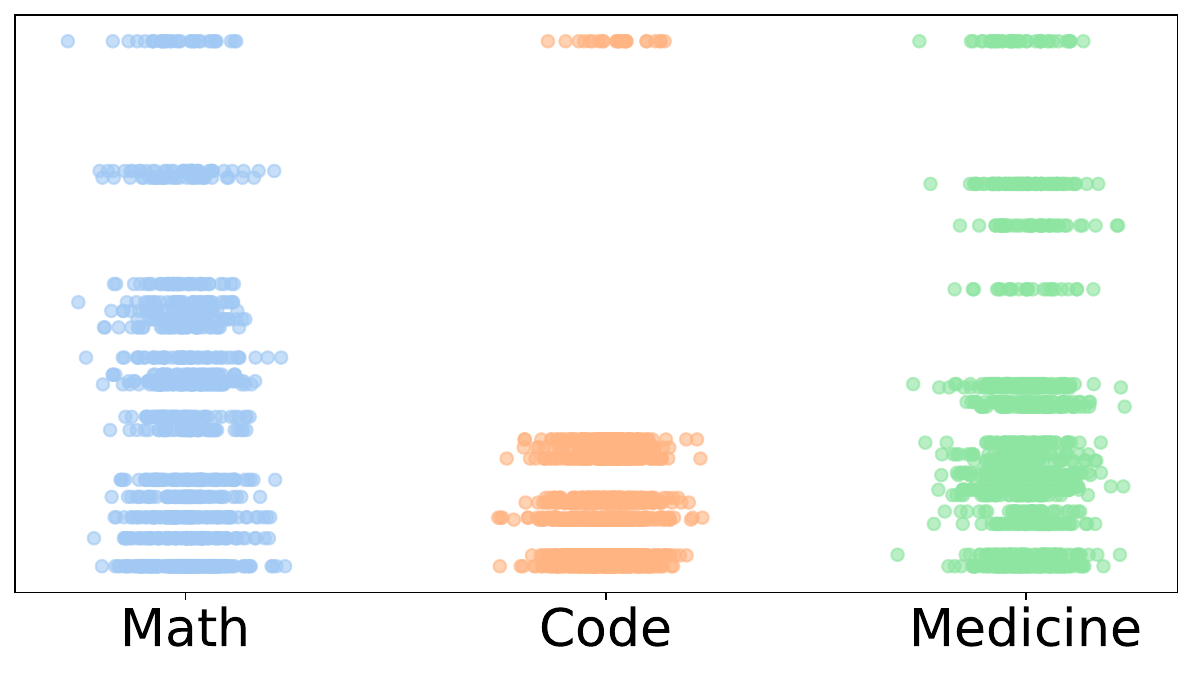}
    \caption{$\mathcal{S}_{\mathrm{RI}}$ on Deepseek7B}
  \end{subfigure}
  \hfill
  \begin{subfigure}{0.32\linewidth}
    \includegraphics[width=\linewidth]{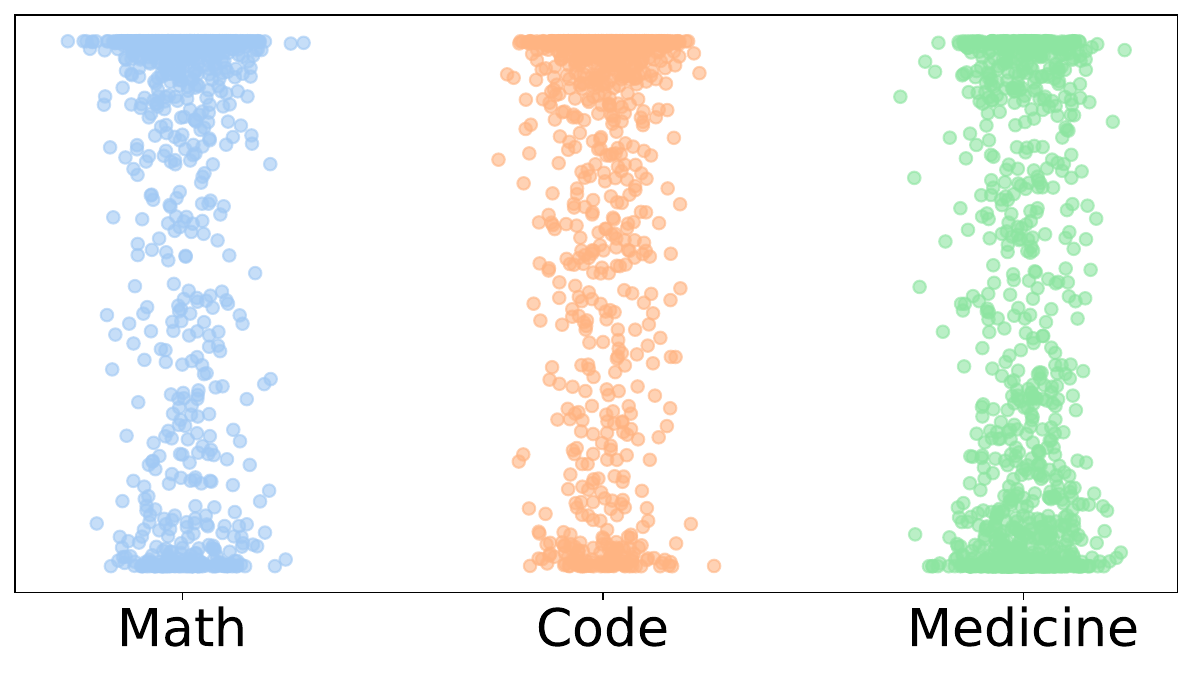}
    \caption{$\mathcal{S}_{\mathrm{KN}}$ on Deepseek7B}
  \end{subfigure}
  \hfill
  \begin{subfigure}{0.32\linewidth}
    \includegraphics[width=\linewidth]{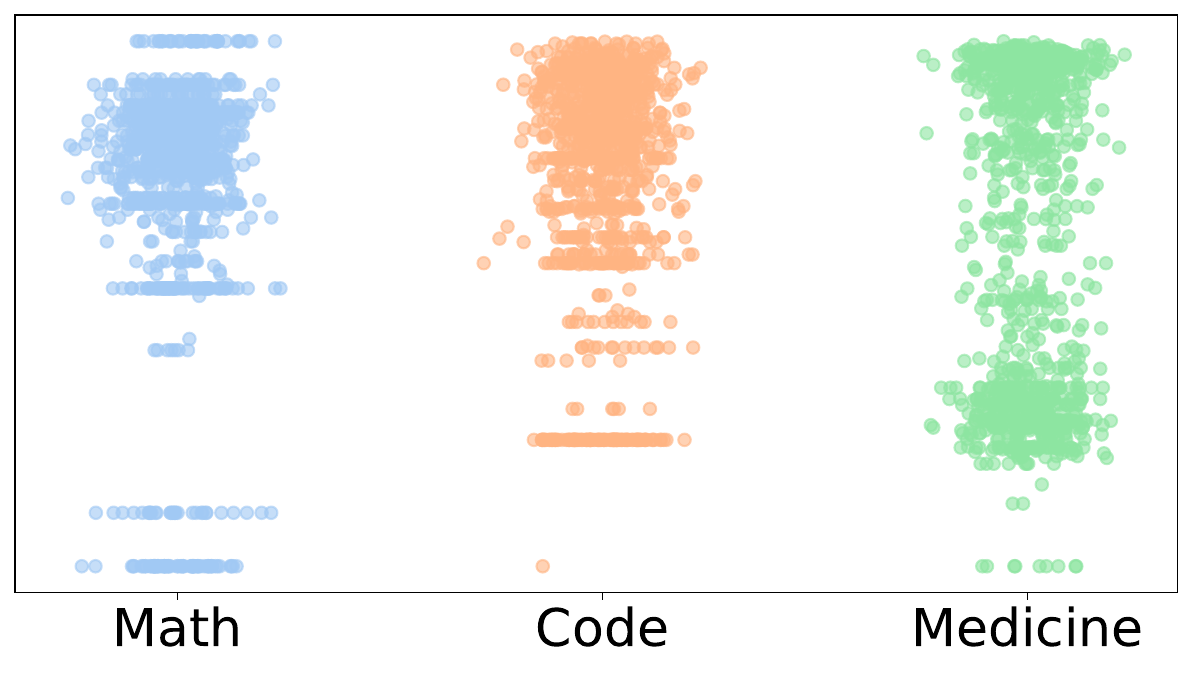}
    \caption{$\mathcal{S}_{\mathrm{TR}}$ on Deepseek7B}
  \end{subfigure}

    \begin{subfigure}{0.32\linewidth}
    \includegraphics[width=\linewidth]{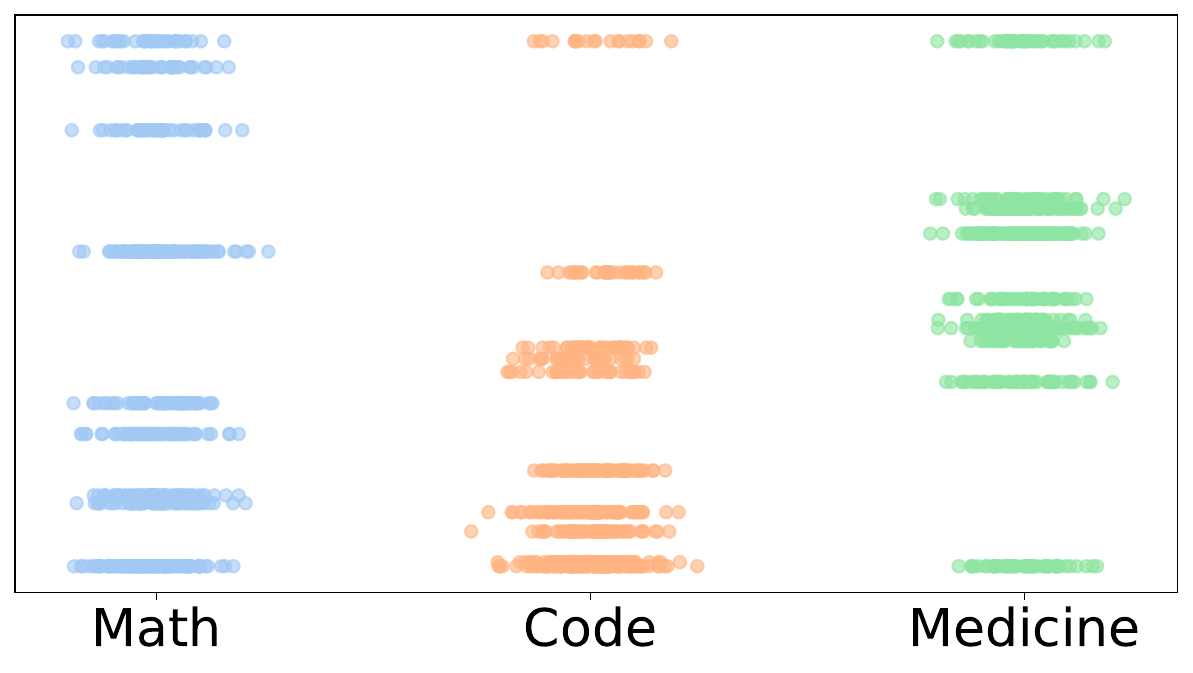}
    \caption{$\mathcal{S}_{\mathrm{RI}}$ on Llama1B}
  \end{subfigure}
  \hfill
  \begin{subfigure}{0.32\linewidth}
    \includegraphics[width=\linewidth]{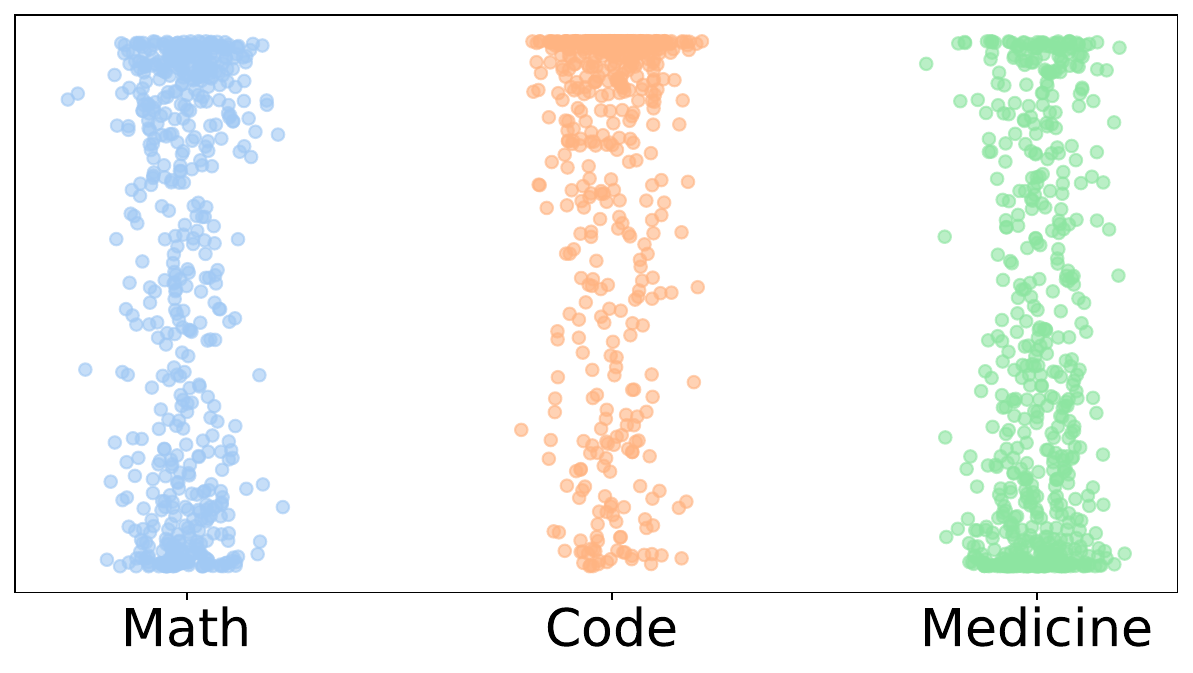}
    \caption{$\mathcal{S}_{\mathrm{KN}}$ on Llama1B}
  \end{subfigure}
  \hfill
  \begin{subfigure}{0.32\linewidth}
    \includegraphics[width=\linewidth]{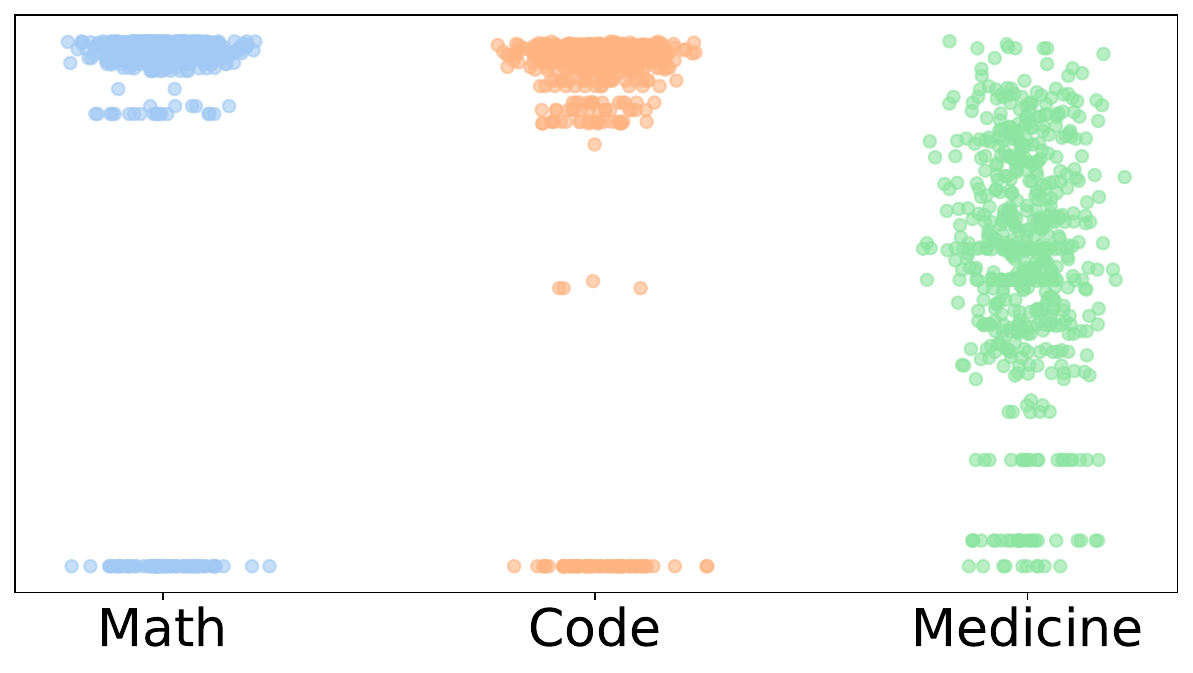}
    \caption{$\mathcal{S}_{\mathrm{TR}}$ on Llama1B}
  \end{subfigure}

  \begin{subfigure}{0.32\linewidth}
    \includegraphics[width=\linewidth]{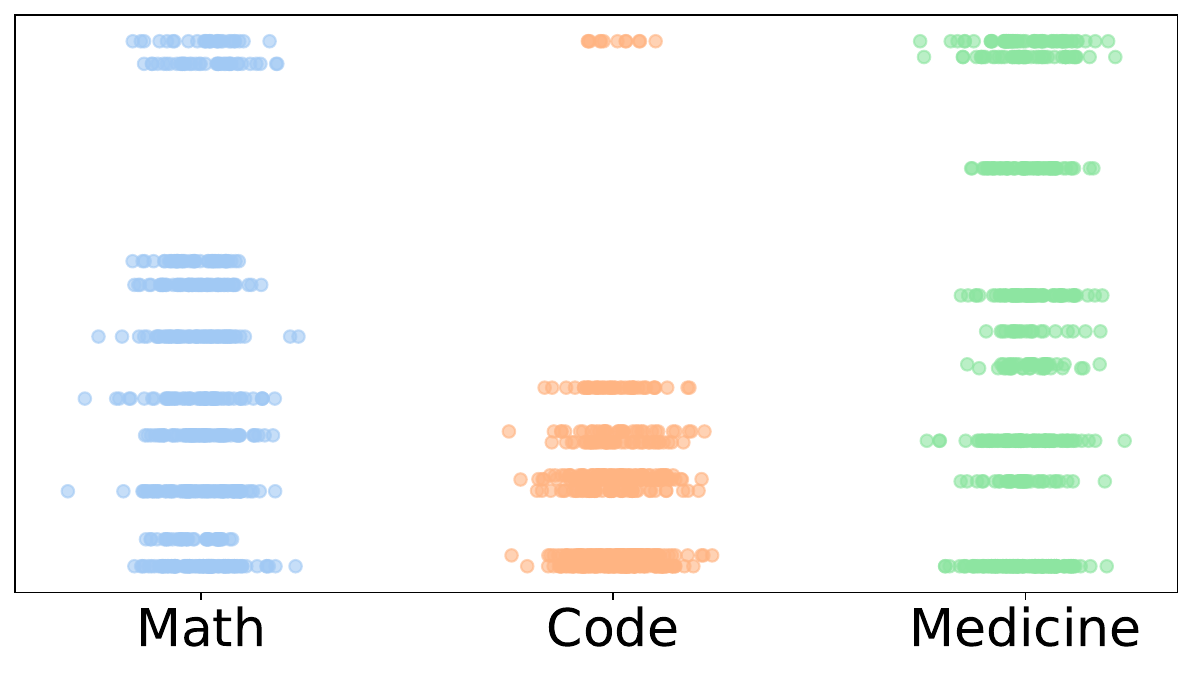}
    \caption{$\mathcal{S}_{\mathrm{RI}}$ on Llama3B}
  \end{subfigure}
  \hfill
  \begin{subfigure}{0.32\linewidth}
    \includegraphics[width=\linewidth]{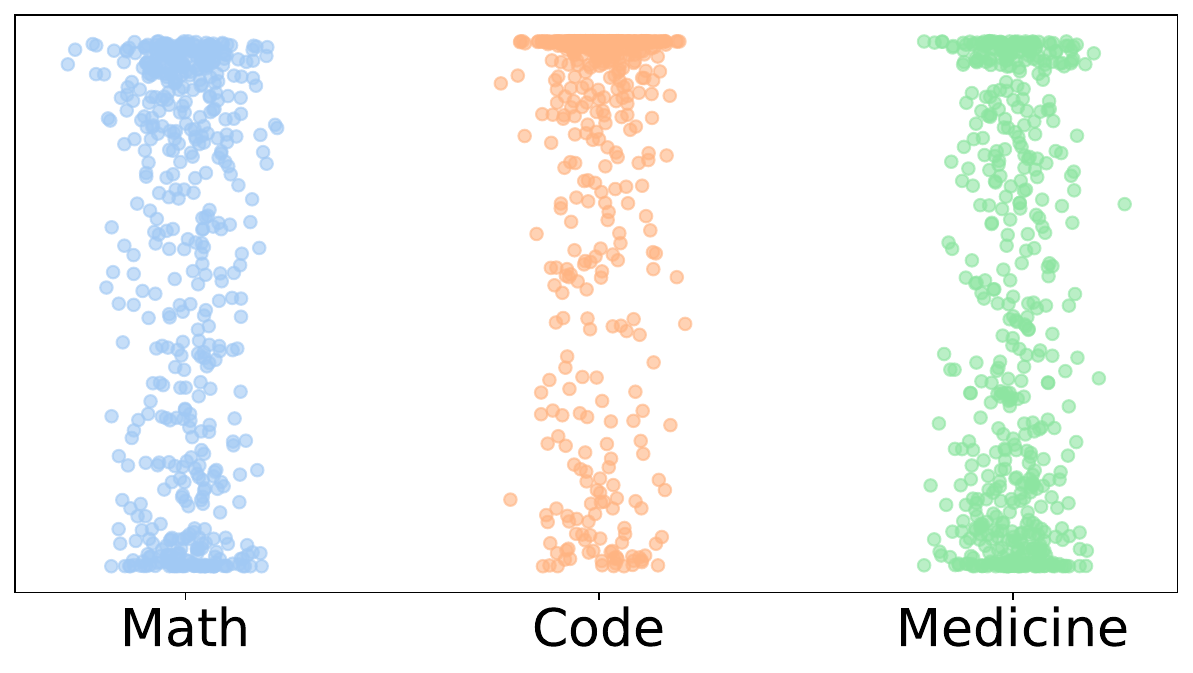}
    \caption{$\mathcal{S}_{\mathrm{KN}}$ on Llama3B}
  \end{subfigure}
  \hfill
  \begin{subfigure}{0.32\linewidth}
    \includegraphics[width=\linewidth]{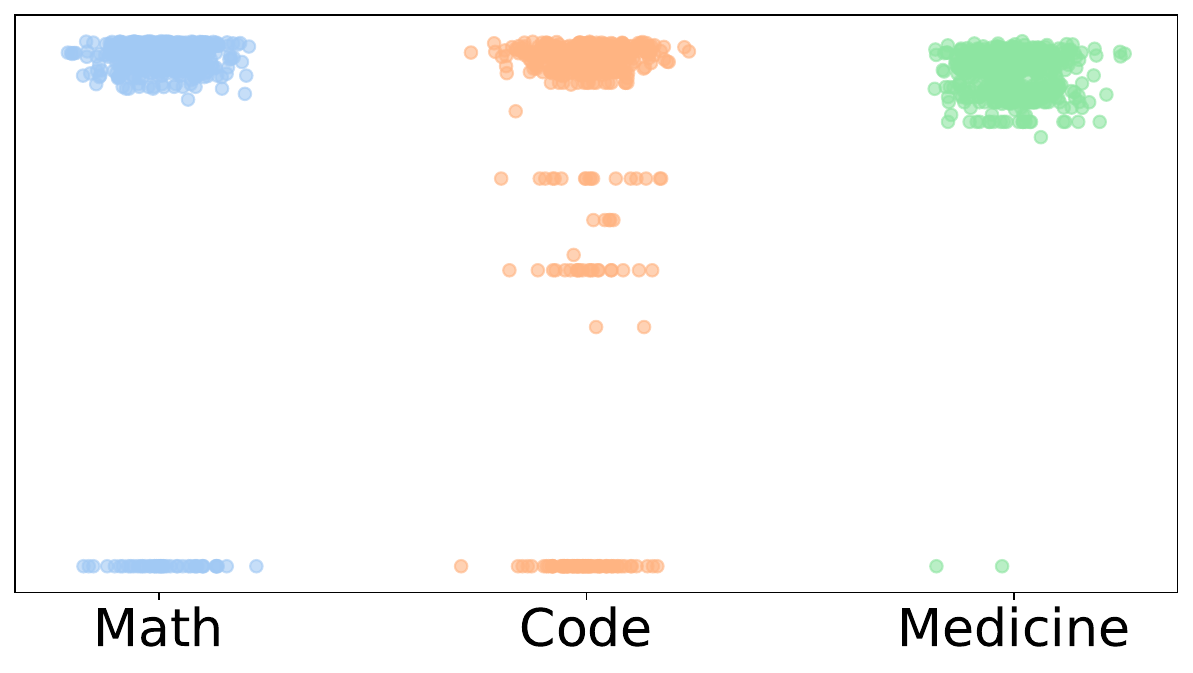}
    \caption{$\mathcal{S}_{\mathrm{TR}}$ on Llama3B}
  \end{subfigure}

  \begin{subfigure}{0.32\linewidth}
    \includegraphics[width=\linewidth]{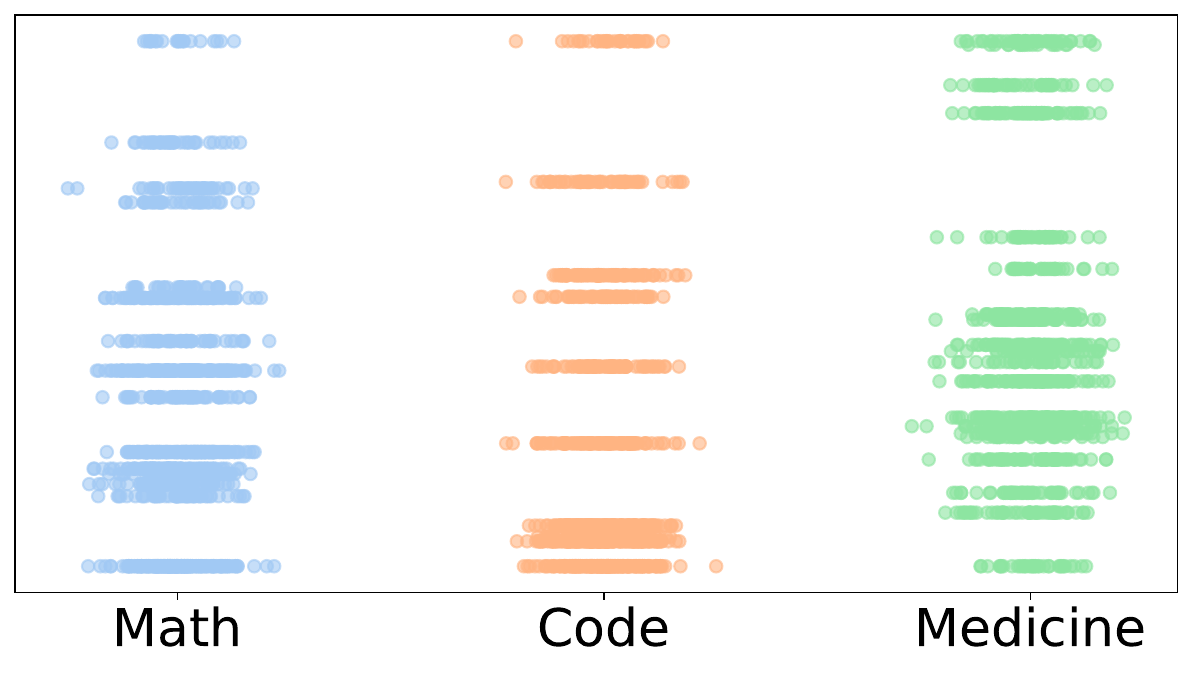}
    \caption{$\mathcal{S}_{\mathrm{RI}}$ on Llama8B}
  \end{subfigure}
  \hfill
  \begin{subfigure}{0.32\linewidth}
    \includegraphics[width=\linewidth]{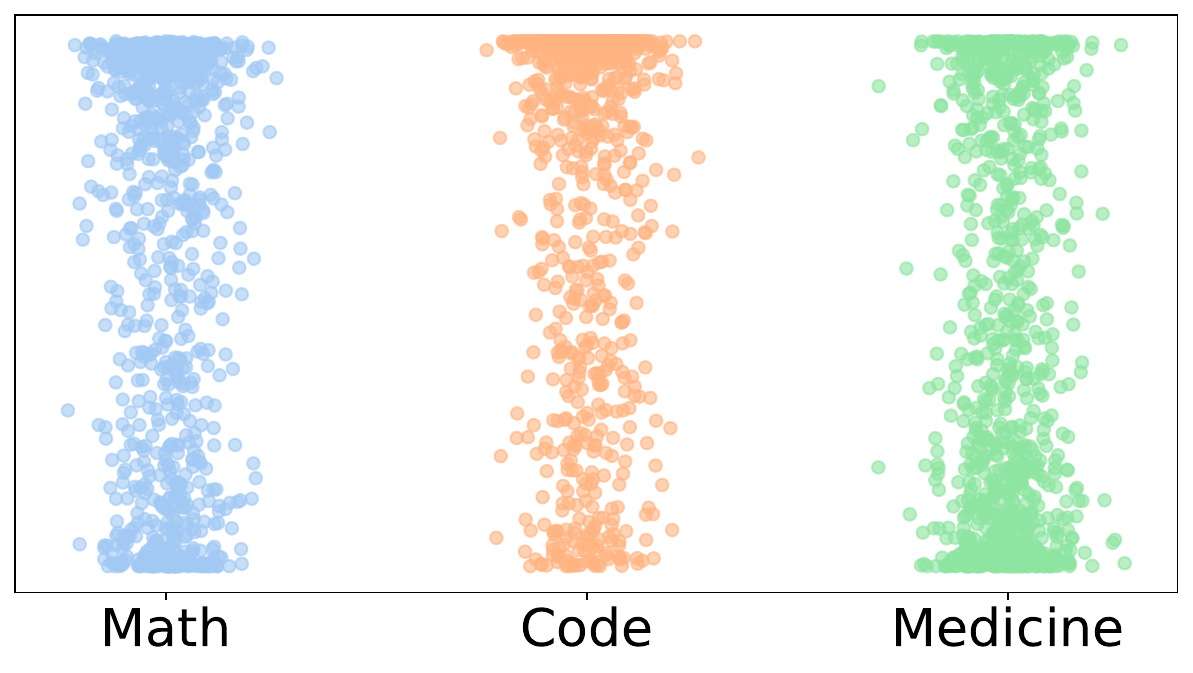}
    \caption{$\mathcal{S}_{\mathrm{KN}}$ on Llama8B}
  \end{subfigure}
  \hfill
  \begin{subfigure}{0.32\linewidth}
    \includegraphics[width=\linewidth]{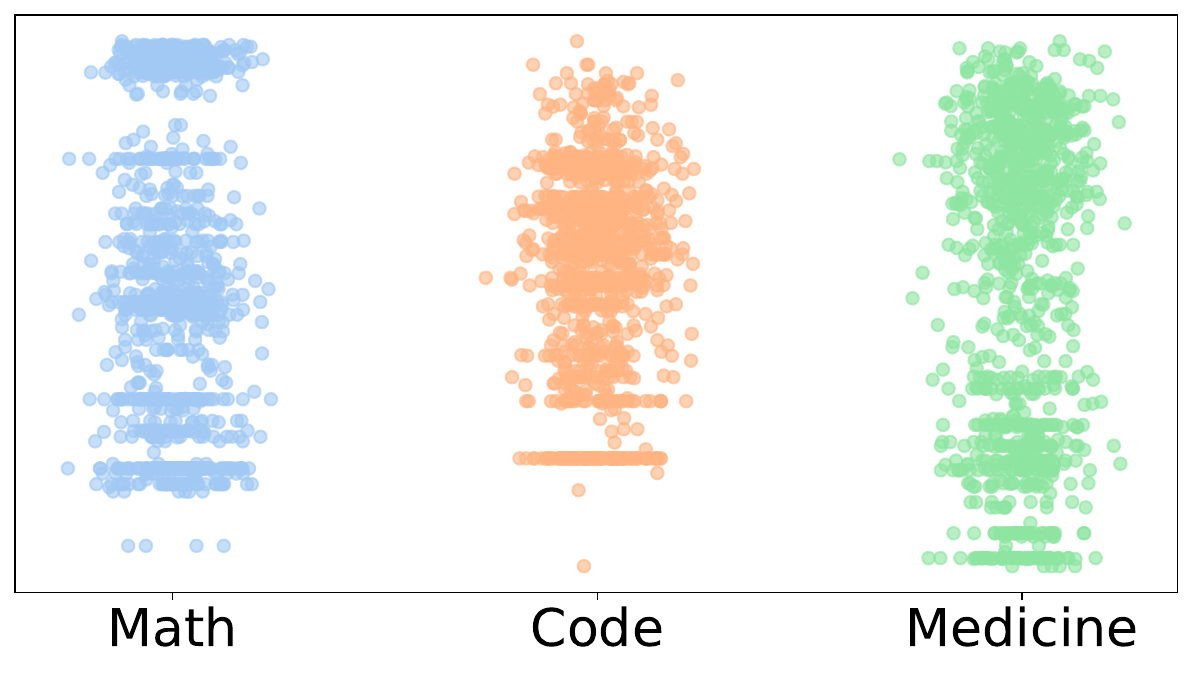}
    \caption{$\mathcal{S}_{\mathrm{TR}}$ on Llama8B}
  \end{subfigure}

  \begin{subfigure}{0.32\linewidth}
    \includegraphics[width=\linewidth]{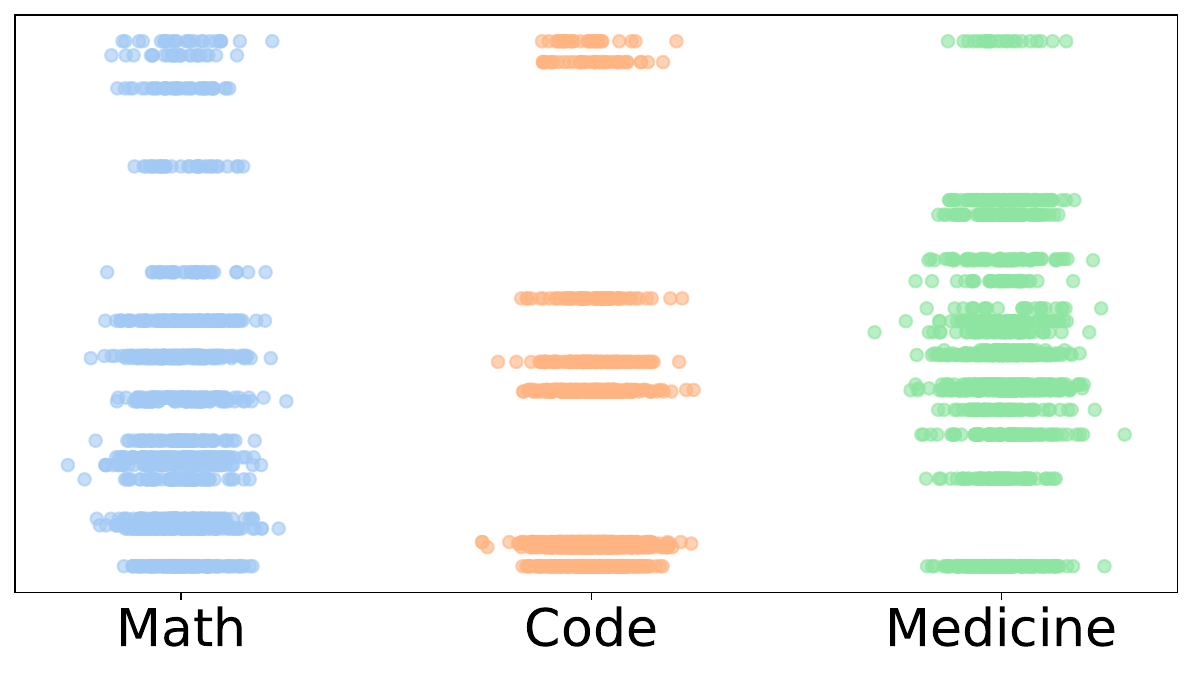}
    \caption{$\mathcal{S}_{\mathrm{RI}}$ on Mistral7B}
  \end{subfigure}
  \hfill
  \begin{subfigure}{0.32\linewidth}
    \includegraphics[width=\linewidth]{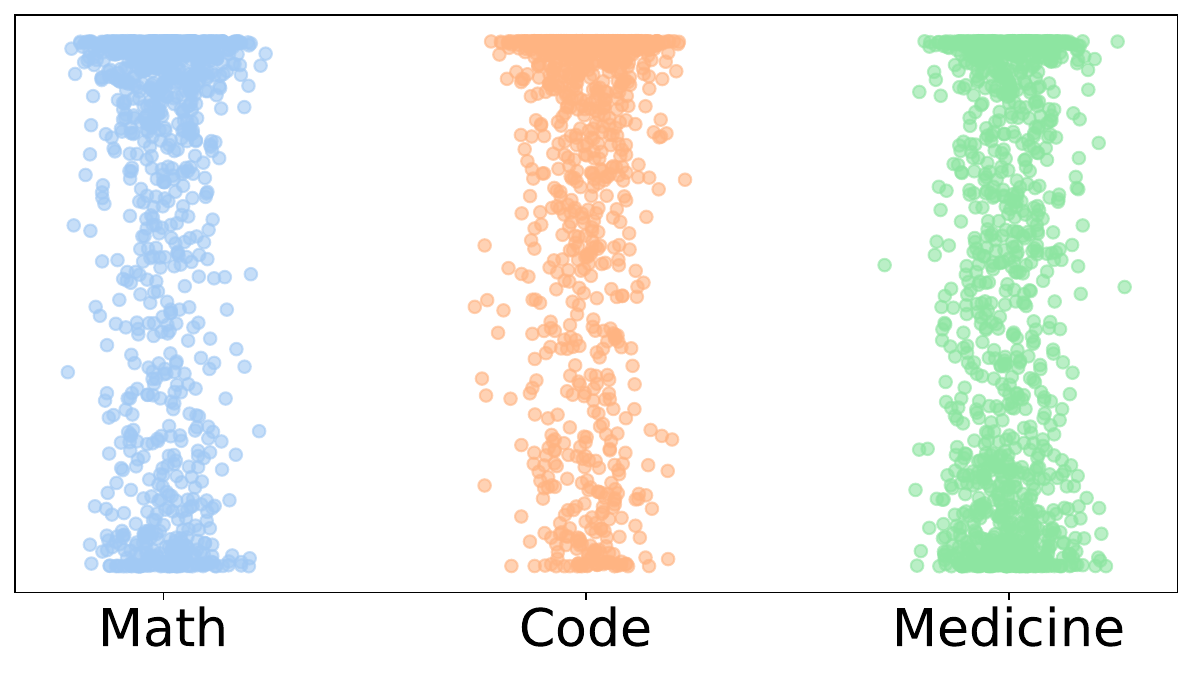}
    \caption{$\mathcal{S}_{\mathrm{KN}}$ on Mistral7B}
  \end{subfigure}
  \hfill
  \begin{subfigure}{0.32\linewidth}
    \includegraphics[width=\linewidth]{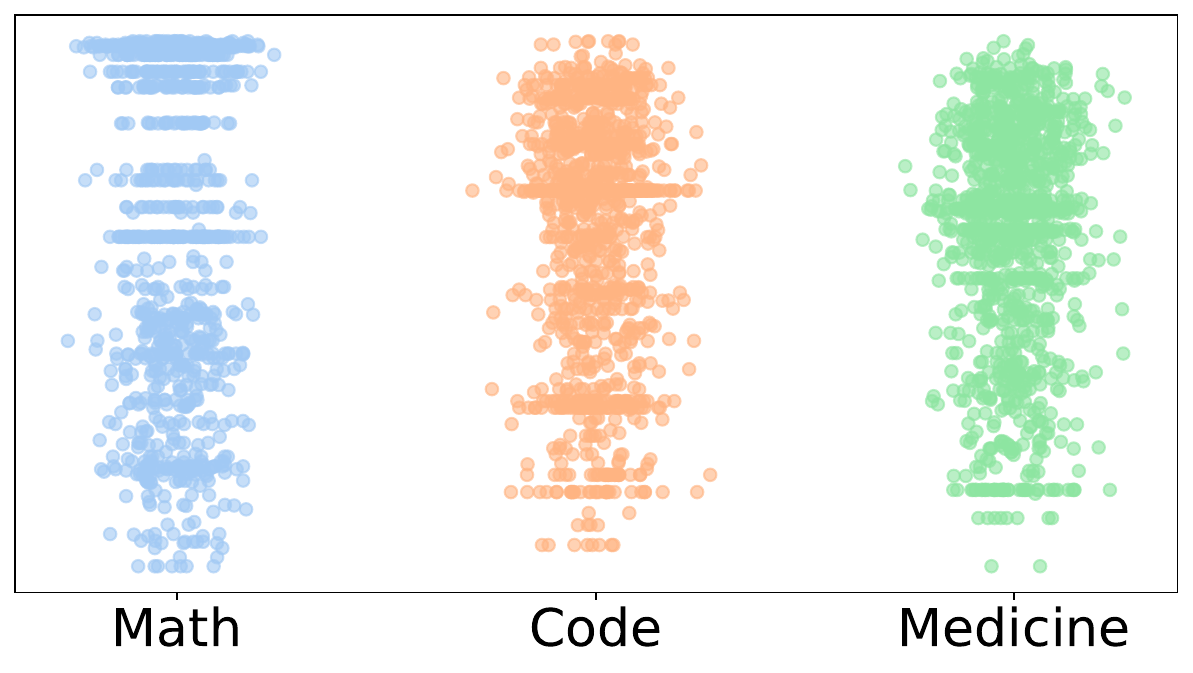}
    \caption{$\mathcal{S}_{\mathrm{TR}}$ on Mistral7B}
  \end{subfigure}
  \caption{Distribution of three scores across different LLMs.}
  \label{fig:more distribution figures}
\end{figure}

%在这里我们展示了更多三种分数的分布图。As shown in Figure~\ref{fig:more distribution figures},尽管在不同的模型或数据集上，分数分布有所差别，但是总体仍旧符合Section~\ref{sec:methodology}中所讨论的规律。具体来说，reasoning importance socre的分布经常在一些固定的数值上，并且仅有少量极端值。Knowledge novelty score的分布较为均匀，很难轻易划分。Task relevance score的分布有着集群特征，可以使用聚类方法进行划分。

Here we present more detailed distribution plots of the three scores. As shown in Figure~\ref{fig:more distribution figures}, although the score distributions differ across various models or datasets, they still generally align with the patterns discussed in Section~\ref{sec:methodology}. Specifically, the reasoning importance scores are often concentrated at certain fixed values, with only a few extreme values observed. The knowledge novelty score distribution is relatively uniform, making it difficult to partition. The task relevance score distribution exhibits distinct clustering characteristics, thus enabling effective partitioning via clustering methods.

\section{LLM usage}
An OpenAI LLM (GPT-4o) was utilized as an assistant for writing and formatting, specifically to refine and suggest edits to figure and table captions, including improvements in grammar, phrasing, clarity, and layout (e.g., column alignment, caption length, and placement). The model's role was strictly limited to surface-level text and visual edits, without contributing to research ideation, experimental design, data analysis, or technical content. All outputs were thoroughly reviewed and revised by the authors, who retain full responsibility for the final text and visuals.

\section{Examples for Token-level Noise}
\label{sec:examples}
Here we show some specific results of token-level noise filtering. As shown in Figure~\ref{fig:example_math}, Figure~\ref{fig:example_math2}, Figure~\ref{fig:example_code}, Figure~\ref{fig:example_medicine} and Figure~\ref{fig:example_fiqa}, the part marked in yellow is the noisy tokens filtered out according to the corresponding scores.

We would like to point out that the filtering decisions made by \Name may, in some cases, be more reliable than human intuition. Take GSM8K as an example: tokens related to mathematical operations (e.g., “+”, “=”) or even specific computed results are not always necessary for effective fine-tuning. From the perspective of Knowledge Novelty (KN)—a dimension easier to interpret—we show in Figure 7 that the filtering prediction probability of the third “+” in the expression “2P + P + 4 = 3P + 4 = 220” exceeds 95\%, and the token “216” in “220 - 4 = 216” is also filtered out. This outcome deviates from typical human intuition, indicating that learning reasoning logic is more critical for the model than memorizing symbols or specific results.

We do not claim that \Name’s scoring mechanism is absolutely flawless. To mitigate the risk of mistakenly removing important tokens, we intentionally adopt conservative thresholds for all three scoring dimensions. This strategy reduces the likelihood of filtering out valuable tokens, enabling \Name to maintain consistently strong performance.

\begin{figure}[t]
    \includegraphics[width=\linewidth]{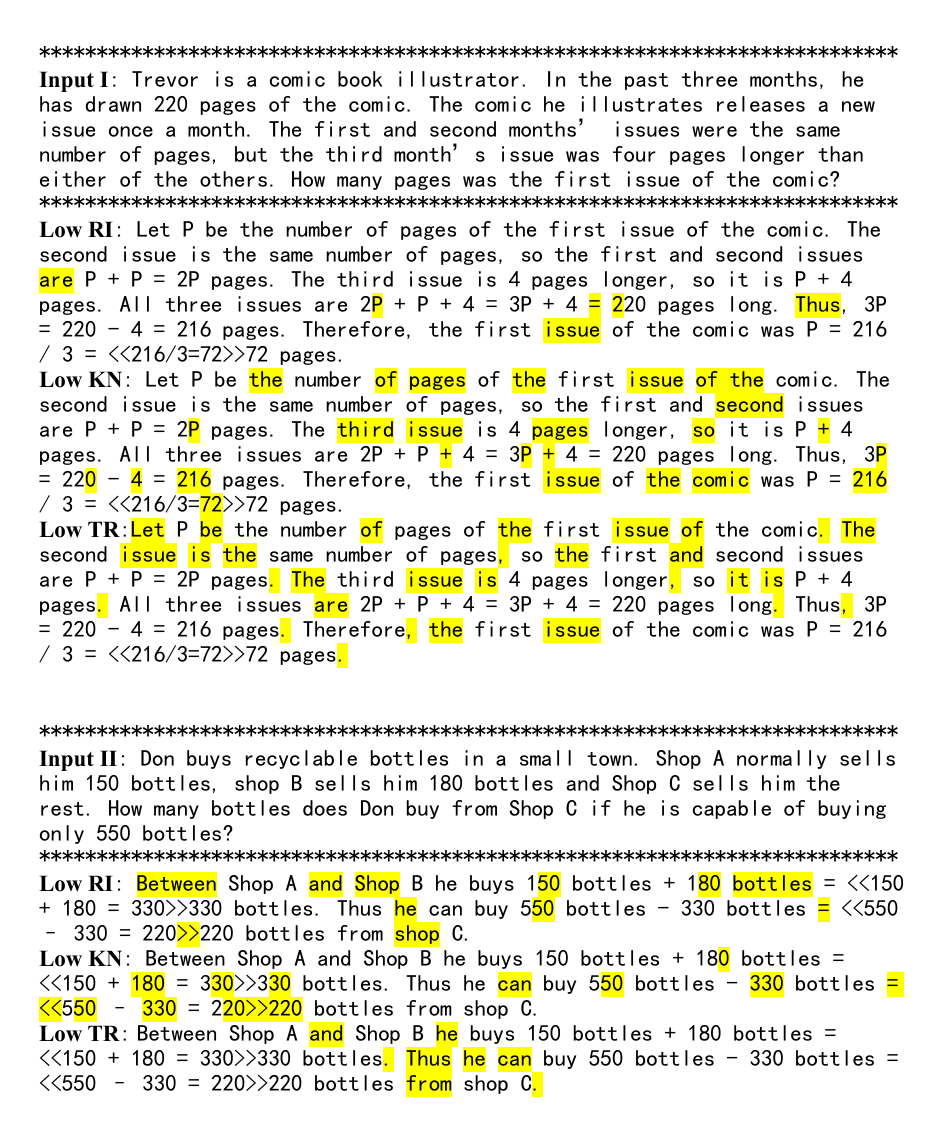}
    \caption{Low score tokens in gsm8k dataset.}
    \label{fig:example_math}
  \end{figure}
\begin{figure}[t]
    \includegraphics[width=\linewidth]{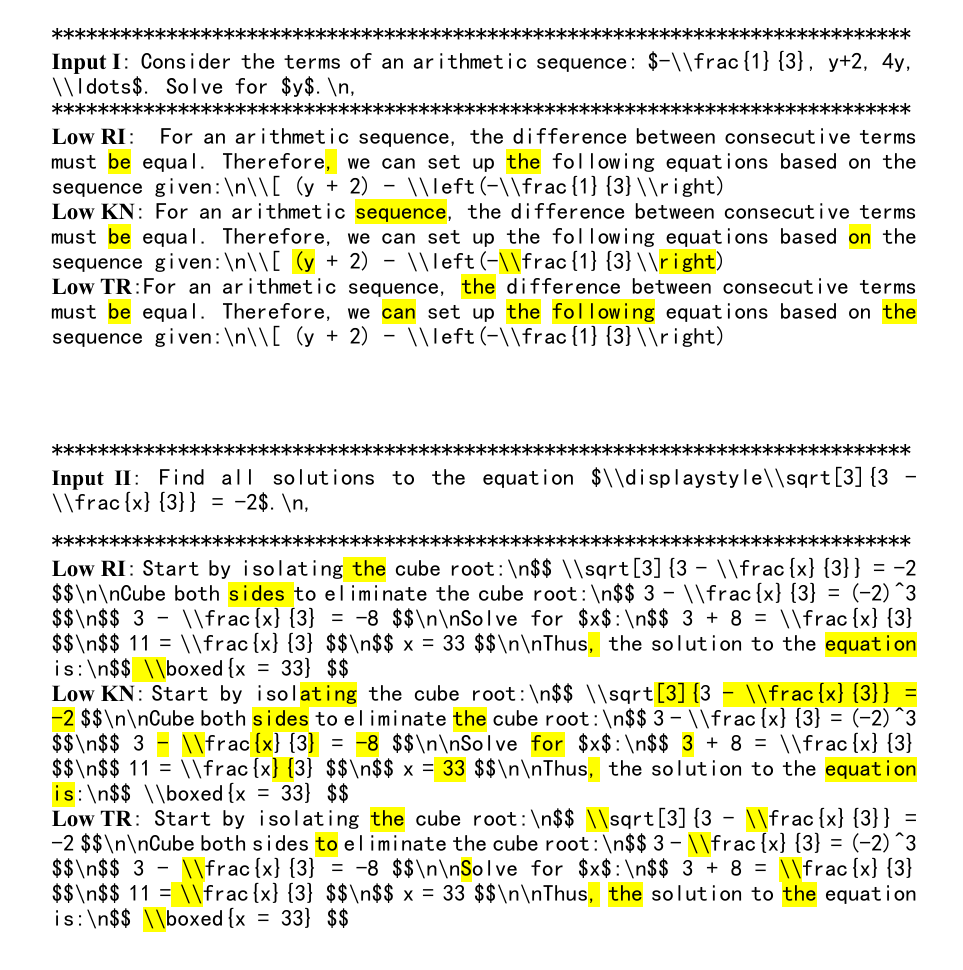}
    \caption{Low score tokens in NuminaMath-CoT dataset.}
    \label{fig:example_math2}
  \end{figure}  
\begin{figure}[t]
    \includegraphics[width=\linewidth]{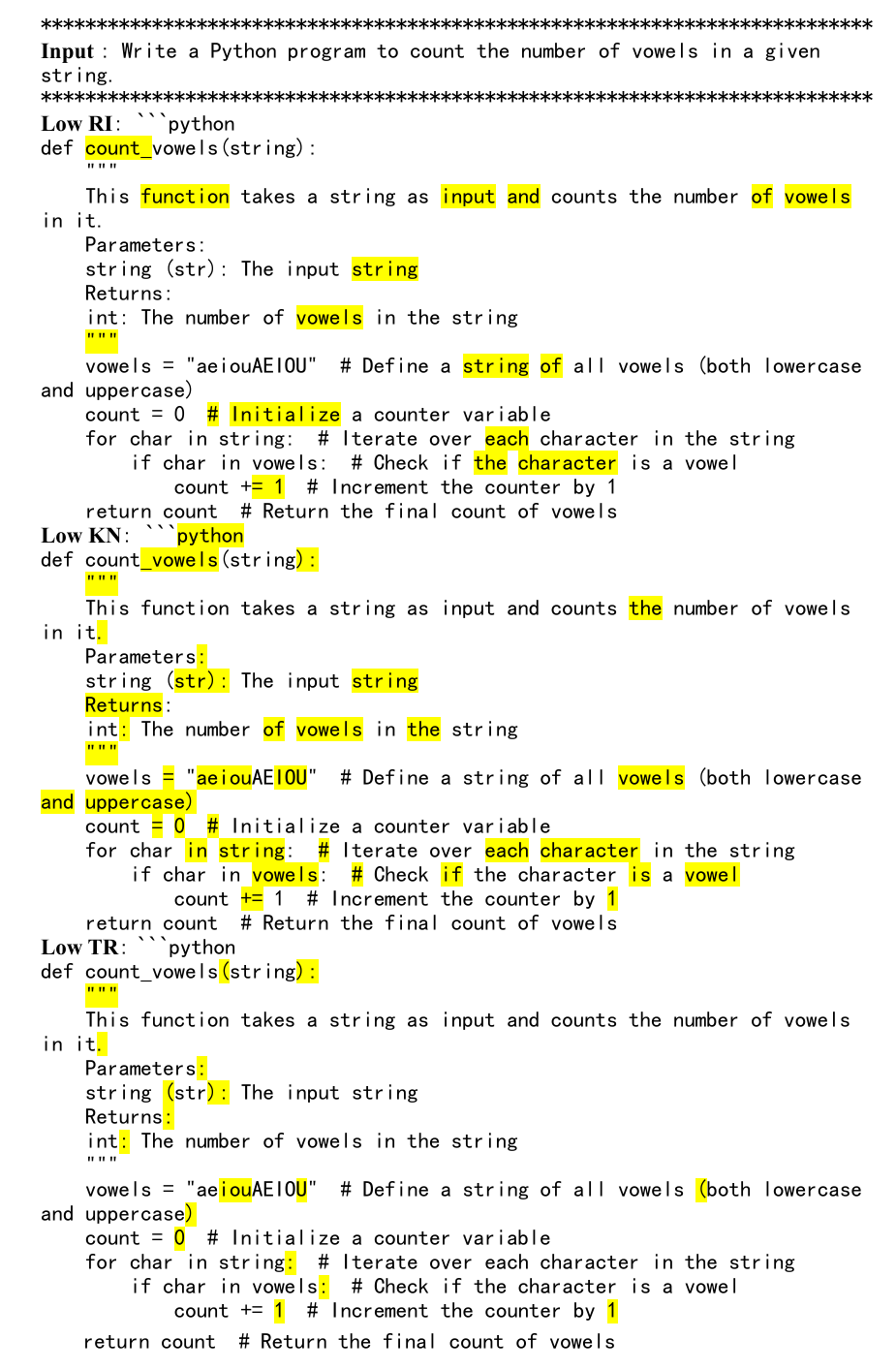}
    \caption{Low score tokens in CodeExercise-Python-27k dataset.}
    \label{fig:example_code}
  \end{figure}
\begin{figure}[t]
    \includegraphics[width=\linewidth]{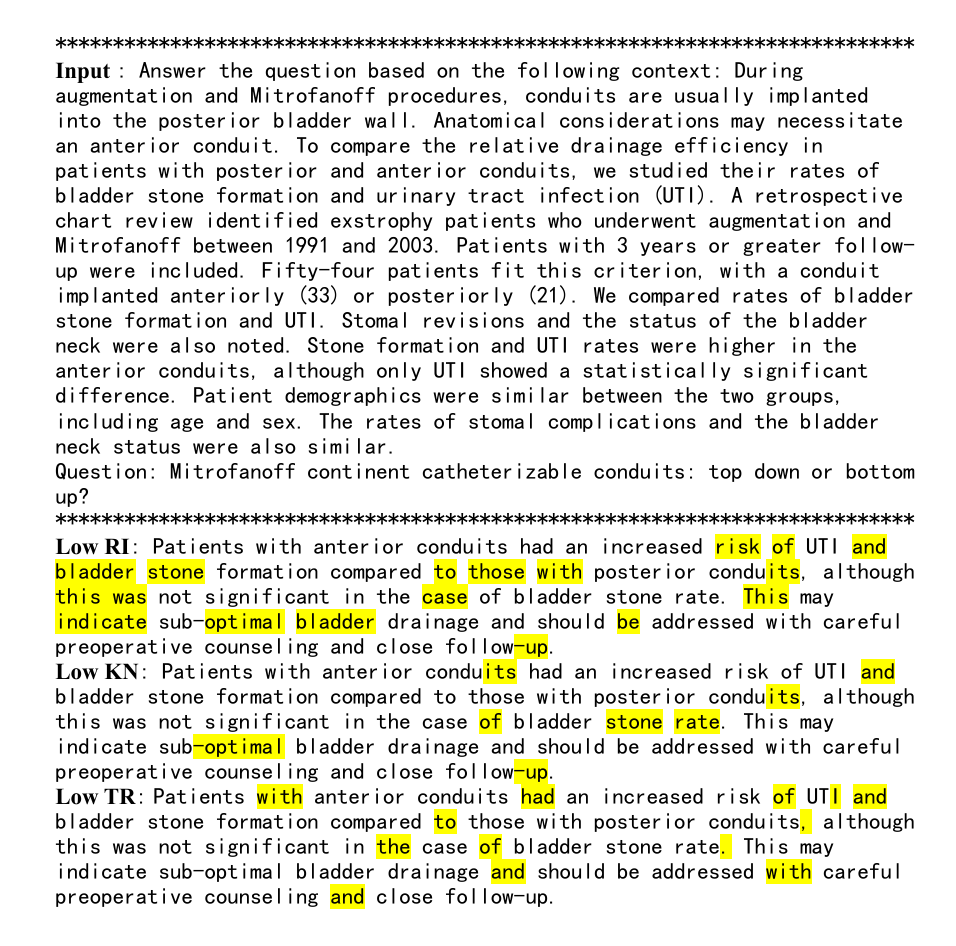}
    \caption{Low score tokens in PubMedQA dataset.}
    \label{fig:example_medicine}
  \end{figure}
\begin{figure}[t]
    \includegraphics[width=\linewidth]{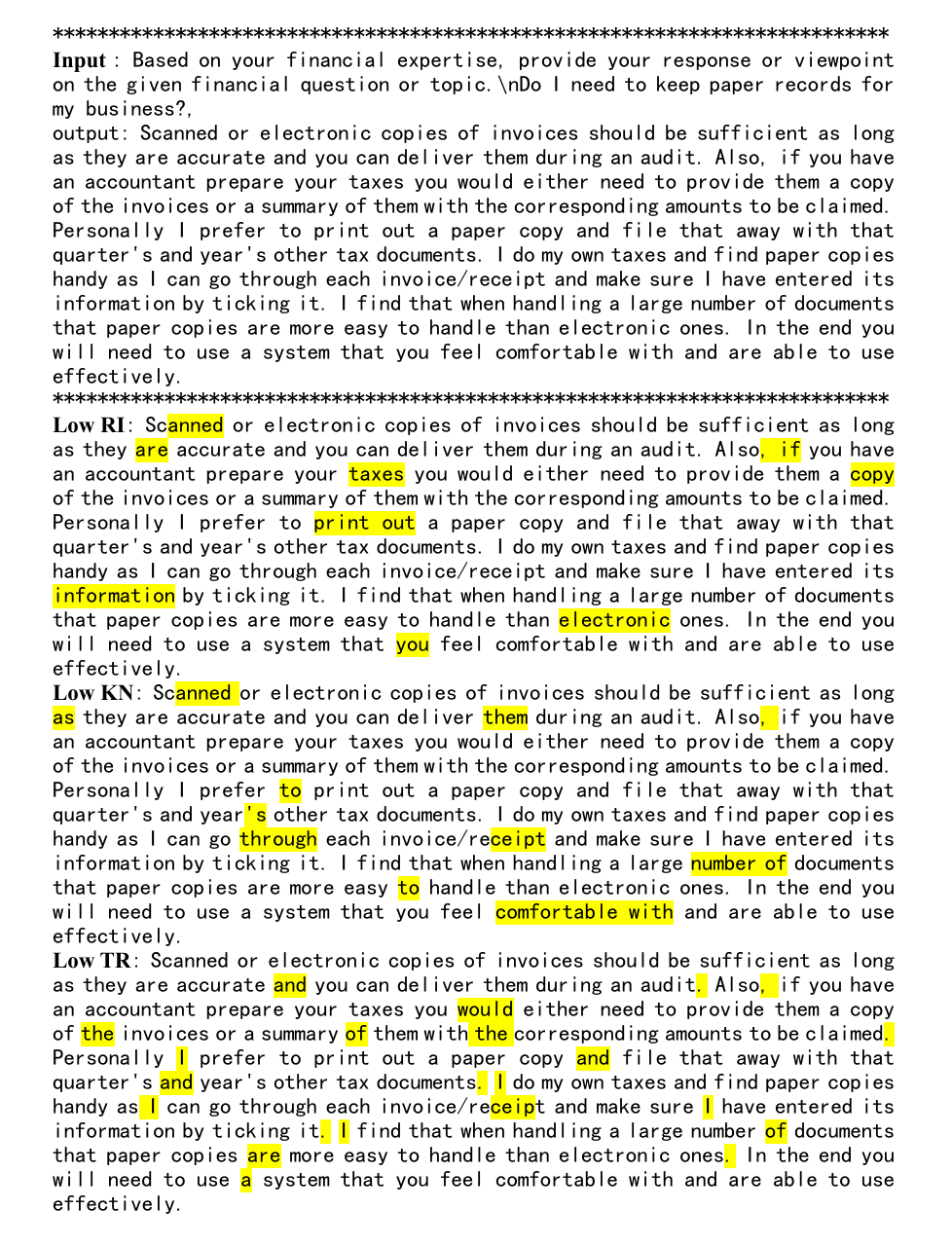}
    \caption{Low score tokens in fiqa dataset.}
    \label{fig:example_fiqa}
  \end{figure}

\end{document}